\documentclass[12pt]{article}
\pdfoutput=1
\usepackage{graphicx,psfrag,epsf}
\usepackage{enumerate}
\usepackage{natbib}
\usepackage{hyperref}
\usepackage{url} 

\newcommand{\blind}{1}

\usepackage[margin={1in,1in},includefoot]{geometry}

\usepackage{longtable}
\usepackage{graphicx}
\usepackage{dcolumn}
\usepackage{amsmath,amsfonts}
\usepackage{amsthm,amssymb}
\usepackage{dsfont,bbold}
\usepackage{bm}
\usepackage{courier}
\usepackage{color}
\usepackage{array}
\usepackage{comment}
\usepackage{enumerate}
\usepackage{mathrsfs}
\usepackage{multirow}
\usepackage{calc}
\usepackage{amscd}
\usepackage[all]{xy}
\usepackage{verbatim}
\usepackage{chngpage}
\usepackage[ruled]{algorithm2e}
\usepackage{lscape}
\usepackage{siunitx}
\usepackage{float}
\usepackage{booktabs}
\usepackage{multirow}
\usepackage{subcaption}
\usepackage{numprint}
\usepackage{siunitx}

\newtheorem{theorem}{Theorem}
\newtheorem{assumption}{Assumption}
\newtheorem{lemma}{Lemma}
\newtheorem{corollary}{Corollary}
\newtheorem{example}{Example}

\DeclareMathOperator*{\argmax}{argmax}
\DeclareMathOperator*{\argmin}{argmin}

\definecolor{brown}{rgb}{0.59,0.29,0}

\begin{document}

\def\spacingset#1{\renewcommand{\baselinestretch}%
{#1}\small\normalsize} \spacingset{1}


\if1\blind
{
  \title{\bf An Actor-Critic Contextual Bandit Algorithm for Personalized Mobile Health Interventions}
  \author{Huitian Lei\\
    Amazon \\
    \\
    Yangyi Lu\\
    Department of Statistics, \\ University of Michigan \\
    \\
    Ambuj Tewari \\
    Department of Statistics, \\
    Department of Electrical Engineering and Computer Science,\\
    University of Michigan\\
    \\
    Susan A. Murphy \thanks{
    The authors gratefully acknowledge funding from the National Institutes of Health grants,  R01HL125440, R01AA023187, P50DA039838, U54EB020404, NSF CAREER grant IIS-1452099 and Sloan Research Fellowship. Part of this work was done when HL was a graduate student at the University of Michigan.}\hspace{.2cm}\\
    Department of Statistics, \\
    Department of Computer Science, \\
    Radcliffe Institute, Harvard University}
  \maketitle
} \fi

\if0\blind
{
  \bigskip
  \bigskip
  \bigskip
  \begin{center}
    {\Large\bf An Actor-Critic Contextual Bandit Algorithm \\for Personalized Mobile Health Interventions}
\end{center}
  \medskip
} \fi

\newpage

\begin{abstract}
Increasing technological sophistication and widespread use of smartphones and wearable devices provide opportunities for innovative and highly personalized health interventions. A Just-In-Time Adaptive Intervention (JITAI)  uses real-time data collection and communication capabilities of modern mobile devices to deliver interventions in real-time that are adapted to the in-the-moment needs of the user. The lack of methodological guidance in constructing data-based JITAIs remains a hurdle in advancing JITAI research despite the increasing popularity of JITAIs among clinical scientists. In this article, we attempt to bridge this methodological gap by formulating the task of tailoring interventions in real-time as a contextual bandit problem. Interpretability requirements in the domain of mobile health lead us to formulate the problem differently from existing formulations intended for web applications such as ad or news article placement. Under the assumption of linear reward function, we choose the reward function (the ``critic") parameterization separately from a lower dimensional parameterization of stochastic policies (the ``actor"). We provide an online actor-critic algorithm that guides the construction and refinement of a JITAI. Asymptotic properties of the actor-critic algorithm are developed and backed up by numerical experiments. Additional numerical experiments are conducted to test the robustness of the algorithm when idealized assumptions used in the analysis of contextual bandit algorithm are breached. 
\end{abstract}

\noindent%
\vfill

\newpage
\spacingset{1.45} 

\section{Introduction}
\label{intro}


Equipped with sophisticated sensing, communication and computation capabilities, smartphones and mobile devices are being increasingly used to deliver Just-In-Time Adaptive Interventions (JITAIs). JITAIs are mobile health interventions where treatment is delivered in real time to individuals as they go about their daily lives.  A key ingredient of a JITAI is a \emph{policy}, that is, a decision rule that inputs sensor and self-report information at any given decision point and outputs a decision. The decision can be whether or not to provide treatment or the type of treatment to be provided.  The use of decision rules to adapt the type and timing of treatment delivery to the individual makes JITAIs particularly promising in facilitating {\em long-term} health behavior change, a pressing but notoriously hard problem (\cite{nahum2018just}). Indeed JITAIs have received increasing popularity and have been used to support health behavior change in a variety of domains including physical activity (\cite{consolvo2008activity, king2013harnessing,muller2017conceptualization}), eating disorders (\cite{bauer2010enhancement}), drug abuse (\cite{scott2009results,carpenter2020developments}), alcohol use (\cite{gustafson2011explicating, suffoletto2012text,witkiewitz2014development}), smoking cessation (\cite{riley2011health}), obesity and weight management (\cite{patrick2009text,thomas2015behavioral}), and other chronic disorders~\citep{richardson2020mhealth}. 

Despite the growing popularity of JITAIs, there is a lack of guidance concerning how to best learn a high-quality evidence-based JITAIs in an \lq\lq online\rq\rq\  setting. That is, learning occurs in a sequential manner as a given user experiences the treatments and sensor/self-report data, including health outcomes of interest, are collected. Ideally, the policy we learn for a given user should take into account the specific way he or she responds to the delivered treatments and is thus \emph{personalized} to the user. However, most of the JITAIs used in existing clinical trials are specified a priori and are based primarily on domain expertise.  The main contribution of this article is to take a step towards bridging the gap between the enthusiasm for JITAIs in the mobile health field and the current lack of statistical methodology to guide the online construction of a personalized policy for a user. We model the learning of a user-specific optimal policy as a contextual bandit problem (\cite{woodroofe1979one, langford2008epoch, li2010contextual,tewari2017ads}). A contextual bandit problem, also called a bandit problem with  side-information, is a sequential decision making problem where a learning algorithm, (i) chooses an action (e.g., treatment) at each time point based on the \emph{context} or side information, and (ii) receives an reward that reflects the quality of the action under the current context. In mobile health settings, the context can include summaries of the sensor and self-report data available at each time point. The goal of the algorithm is to learn the optimal policy, that is, the  policy that maximizes a regularized average reward for a user. We propose an online ``actor-critic'' algorithm for learning the optimal policy. Compared to offline learning, in online learning the contexts and rewards arrive in a sequential fashion and the estimate of the optimal policy is updated as data accumulates. The updated policy is used to choose the treatment action at the subsequent time point. In our actor-critic algorithm, the critic estimates parameters in a model for the conditional mean of the reward given context and action. The actor then updates the estimated optimal policy based on the estimated reward model. Under idealized assumptions, we derive consistency and asymptotic normality of the estimates produced by our algorithm.

Our work is motivated by our collaboration on HeartSteps (\cite{klasnja_microrandomized_2015, dempsey2015randomised}). In the HeartSteps project, the second and third of three studies will involve the use of a online learning algorithm for constructing personalized policies; the algorithm presented here represents our first step in developing the learning algorithm.  The goal of the HeartSteps project is to reduce sedentary behavior and increase physical activity in individuals who have experienced a cardiac event and been in cardiac rehab. The current version of HeartSteps involves data collection both via a smartphone as well as wristband sensor. A variety of sensor and self-report data is available at each time point, including step count, GPS location, weather, time of the day, day of the week and user calendar busyness. The current version of HeartSteps can deliver a treatment (an activity suggestion) at any of 5 time points per day via an audible ping and a notification on the smartphone lock screen. 

This article is organized as follows. In Section \ref{form}, we formulate online learning of a policy for a given user as a contextual bandit problem and define what we mean by an optimal policy. Due to the concern that deterministic policies may habituate users to treatments, thereby causing them to ignore treatment, our definition of optimality is different from the ones found in most existing contextual bandit papers. In Section \ref{alg}, we present an actor-critic contextual bandit algorithm for learning the optimal policy. In Section \ref{theory}, we derive consistency and asymptotic normality of the estimates produced by our algorithm. We also use these results to derive regret bounds for our algorithm. In Section \ref{numerical}, we present a comprehensive simulation study to investigate the performance of the actor-critic algorithm under various simulation settings including settings which violate the usual assumptions underpinning contextual bandit algorithms.

\section{Learning JITAIs as a Contextual Bandit Problem}\label{form}

We formulate the online learning of optimal policy for a given user as a \emph{stochastic} contextual bandit problem. A contextual bandit problem is specified by a quadruple  $(\mathcal{S}, d, \mathcal{A}, r)$, where $\mathcal{S}$ is the context space, $d$ is a probability distribution on the context space, $\mathcal{A}$ is the action space and $r$ is the  reward space. 
At a decision point $t$, the online learning algorithm collects the context $S_t \in \mathcal{S}$, take an action $A_t \in \mathcal{A}$ after which a reward $R_t\in r$ is revealed before the next decision point. The algorithm only gets to observe the reward corresponding to the action taken; it does not have access to the rewards that would have been generated given all other actions. The sequence of tuples $\{(S_{\tau}, A_{\tau}, R_{\tau})\}_{\tau=1}^{t}$ summarizes all information available to the algorithm prior to decision point $t+1$.

For most JITAI applications, interventions are expected to have an impact on the reward but little or no impact on the context distribution at the next decision point. In our HeartSteps example, an encouraging message shown on an user's lock screen will likely increase his/her steps in the following hours. We do not expect, however, the message to drastically change the context, such as weather and time of the day, at the next time an intervention is generated. In fact, interventions in a JITAI are sometimes referred to as ``Ecological Momentary Interventions" (EMIs) or ``micro-interventions". The naming emphasizes that the effects of many interventions in this domain are short-lived in nature. Based on the momentary nature of JITAI intervention effects, we make the following assumption.
\begin{assumption}
[i.i.d.\ contexts] Action $A_t$ has a in-the-moment effect on the reward $R_t$ with expected reward function:
$$
\mathbb{E}\left( R_t | S_t = s, A_t = a \right) =r(s,a) .
$$
However $A_t$ does not affect the distribution of $S_{\tau}$ for $\tau \geq t+1$. We further assume that contexts $S_{t}$ are i.i.d. with probability density function $d(s)$.
\label{assumption_iid}
\end{assumption}


A \emph{(stochastic) policy} is a mapping from the context space to (a probability distribution over) the action space. Policies in JITAI are used to specify (the probability of) an action given a context. In this article, we focus on a binary action space $\mathcal{A}=\{0,1\}$ and a class of parametrized stochastic policies for which the probability of taking action 1 given context $S=s$ is parameterized as $\pi_{\theta}(A=1|S=s) =\frac{e^{g(s)^T\theta}}{1+ e^{g(s)^T\theta}}$. 
Here $g(s)$ is a $p$-dimensional policy feature vector that contains candidate variables (and their transformations) useful for decision making. A big advantage of using a class of parametrized policies is the transparency on how each variable in $g(s)$ influences the choice of action: the influences are reflected by the sign and magnitude of the corresponding components in $\theta$. Confidence intervals for and hypothesis testing on the optimal $\theta$ help answer scientific questions on the usefulness of a particular contextual variable for decision making. 
For example, suppose the scientist includes a GPS location based variable as a candidate variable in the policy, yet the confidence interval for the $\theta$ coefficient of this variable turns out to contain $0$. Then we might omit the sensing of this variable in future because continuously sensing GPS location on smartphones drains the battery.   Similarly,  self-reported measures on user's emotional states induce user burden. Therefore, if the confidence interval for the $\theta$ coefficients of these variables contains $0$ we may reduce user burden by omitting their collection.

\subsection{The Regularized Average Reward}\label{sec_reg_reward}

It is well-known that exploration (see, e.g., \cite{audibert2009exploration}) is essential to learning optimal treatment policy: by assigning non-zero probability to each action in action space, exploration prevents the algorithm from being trapped to a suboptimal policy. However it turns out that standard definitions of optimality often lead to deterministic policies. For example, a natural and intuitive definition of an optimal policy is a policy that maximizes the average reward:
$$V^*(\theta)=\int_{s\in \mathcal{S}}d(s)\sum_{a\in\mathcal{A}}r(s,a) \pi_{\theta}(s,a)ds,$$
where $d(s)$ is the probability density function of context. The following lemma shows that, in a simple setting where the context space is one-dimensional and finite, there always exists a deterministic optimal policy. The proof of this lemma is provided in the supplementary material section \ref{sup_lemma1}.

\begin{lemma}
Suppose that the context space is discrete and finite, $\mathcal{S}= \{s_1, s_2, ..., s_K\}$. Among the policies parameterized as $\pi_{\theta}(s,1) =  \frac{e^{\theta_0+\theta_1s}}{1+e^{\theta_0+\theta_1s}}$, there exists a policy that maximizes $V^*(\theta)$ for which $P(\pi_{\theta}(S,1)=0\mbox{ or }1)=1$. 
\label{lemma_dtm}
\end{lemma}

One way to ensure treatment variety is to introduce a {\em chance constraint} (also called a ``probabilistic constraint"; see, e.g., \cite{prekopa2013stochastic})  that ensures, with high probability over the  context distribution, that the probability of taking each treatment action under any policy we consider is bounded sufficiently away from $0$.  For binary actions the constraint has the form:
\begin{align}
P(p_0 \leq \pi_{\theta}(S,1) \leq 1-p_0) \geq 1-\alpha
\label{const_s}
\end{align}
where $0<p_0<0.5$, $0<\alpha<1$ are constants controlling the amount of stochasticity. 
The stochasticity constraint requires that, for at least $(1-\alpha)100\%$ of the contexts, there is at least $p_0$ probability to take either of the two available actions. 

Maximizing the average reward $V^*(\theta)$ subject to the stochasticity constraint (\ref{const_s}) is a chance constrained optimization problem, an active research area in recent years (\cite{nemirovski2006convex, campi2011sampling}). Solving this chance constraint problem, however, involves a major difficulty: constraint (\ref{const_s}) is, in general, a non-convex constraint on $\theta$. Moreover, the left hand side of the chance constraint is an expectation of a non-smooth indicator function. Both the non-convexity and the non-smoothness make the optimization problem computationally intractable. We circumvent this difficulty by relaxing constraint (\ref{const_s}) to a convex alternative: 
\begin{align}
\theta^T \mathbb{E}[g(S)g(S)^T] \theta=\theta^T [\int_{s\in\mathcal{S}}g(s)g(s)^Td(s)ds]\theta \leq \left(\log(\frac{p_0}{1-p_0})\right)^2 \alpha ,
\label{const_q}
\end{align}
which is obtained by bounding the probability in~\eqref{const_s} using Markov's inequality and some algebra. Since the quadratic constraint is derived using an upper bound on the original probability, it is more stringent than the chance constraint and always guarantees \emph{at least} the desired amount of treatment variety. 
Instead of solving the quadratic optimization problem that maximizes the average reward $V^*(\theta)$ subject to the quadratic constraint (\ref{const_q}), we choose to maximize the corresponding Lagrangian function. Incorporating inequality constraints by using Lagrangian multipliers has been widely used in reinforcement learning literature to solve constrained Markov decision problem (\cite{borkar2005actor, bhatnagar2012online}). Given a Lagrangian multiplier $\lambda$, the following Lagrangian function: 
\begin{align}
J^*_{\lambda}(\theta)=\int_{s\in \mathcal{S}}d(s)\sum_{a\in\mathcal{A}}r(s,a) \pi_{\theta}(s,a)ds-\lambda \, \theta^T \mathbb{E}[g(S)g(S)^T] \theta
\label{reg_reward}
\end{align}
is referred to as the \emph{regularized average reward} in this article. For a fixed value of $\lambda$, we define the {\emph{optimal policy} to be the policy that maximizes the regularized average reward, namely $\theta_{\lambda}^*=\argmax J^*_{\lambda}(\theta)$. Under mild regularity conditions, we show that there is a one-to-one correspondence between quadratic constrained optimization of the average reward (using the constraint~\eqref{const_q}) and the unconstrained optimization of the regularized average reward~\eqref{reg_reward}. Details can be found in the supplementary material section \ref{sup_1to1}.
There are two computational advantages of maximizing the regularized average reward as opposed to solving a constrained optimization. First, optimizing the regularized average reward function results in a unique solution even when there is no treatment effect. When the expected reward does not depend on the treatment action, i.e., $\mathbb{E}(R|S=s,A=a)=\mathbb{E}(R|S=s)$,  all policies in the feasible set given by the constraint have the same average reward. The regularized average reward function, in contrast, has a unique maximizer at $\theta=\mathbf{0}_{p \times 1}$, a purely random policy that assigns $50\%$ probability to both actions. Therefore, maximizing the regularized average reward gives rise to a $0$ estimand when there is no treatment effect. Second, even when the uniqueness of optimal policy is not an issue, maximization of $J^*_{\lambda}(\theta)$ has computational advantages over maximization of $V^*(\theta)$ under the constraint (\ref{const_q}) because the subtraction of the quadratic term $\lambda \theta^T \mathbb{E}[g(S)g(S)^T] \theta$ introduces a degree of concavity to the surface of $J^*_{\lambda}(\theta)$, thus stabilizing the optimization.

\section{Learning the Optimal Policy: the Online Actor-Critic Algorithm}\label{alg}

In this section, we propose an online actor-critic algorithm for learning the optimal policy parameter $\theta_{\lambda}^*$. The main algorithm presented in this section uses a fixed penalty coefficient $\lambda$. For notational simplicity we will drop the dependency of optimal policy parameter on $\lambda$ and simply replace $\theta_{\lambda}^*$ with $\theta^*$. A more general algorithm that simultaneously learns the penalty coefficient and the optimal policy parameter will be presented in the numerical experiment section. The following table \ref{notations} summarizes the notations that will be used in section \ref{alg} and section \ref{theory}.

\begin{table}[H]
	\centering
	\begin{tabular}{c c}
		\hline
		Notation & Description\\
		\hline
		$S_t$ & context at time $t$ \\
		$A_t$ & action at time $t$ \\
		$R_t$ & momentary reward given $(S_t, A_t)$ \\
		$d(s)$ & probability density function for the contexts \\
		$f(s,a)$ & a k-dimensional reward feature \\
		$\mu^*$ & true reward parameter \\
		$r(s,a)$ & expected reward, i.e. $r(s,a)=f(s,a)^T \mu^*$ \\
		$\sigma^2$ & SubGaussian parameter for the error term in the linear reward model\\
		$\hat{\mu}_t$ & estimated reward parameter at time $t$ \\
		$g(s)$ & a p-dimensional policy feature \\	
		$\theta^*$ & optimal policy parameter \\
		$\hat{\theta}_t$ & estimated optimal policy parameter at time $t$ \\		
		$\zeta$ & $\mathcal{L}_2$ penalty coefficient in estimating $\hat{\mu}_t$\\
		$r_i$, $i=0,1$ & rank of matrix $\mathbb{E}\left[f(S,a_i)f(S,a_i)^T\right]$ \\
		$\lambda_{r_i}$, $i=0,1$ & smallest positive eigenvalue of $\mathbb{E}\left[f(S,a_i)f(S,a_i)^T\right]$\\
		$\lambda_p$ & smallest eigenvalue of $\mathbb{E}\left[g(S)g(S)^T\right]$\\
		$\lambda_{1,a_i}$, $i=0,1$ & largest eigenvalue of $\mathbb{E}\left[f(S,a_i)f(S,a_i)^T\right]$\\
		$J\left(\theta,\mu\right) $ & $\int_{s\in \mathcal{S}}d(s)\sum_{a\in \mathcal{A}} f(s,a)^T\mu \pi_\theta(s,a)ds - \lambda \theta^T \mathbb{E}\left[g(S)g(S)^T\right]\theta$\\
		$\tilde{J}_t(\theta,\mu)$ & $\frac{1}{t}\sum_{\tau=1}^t\sum_{a\in \mathcal{A}} f(S_\tau,a)^T \mu\pi_\theta(S_\tau,a) - \lambda \theta^T \left(\frac{1}{t}\sum_{\tau=1}^tg(S_\tau)g(S_\tau)^T\right)\theta$\\
		$\hat{J}_t(\theta,\mu)$ & $\frac{1}{t}\sum_{\tau=1}^t\sum_{a\in \mathcal{A}} r_\mu(S_\tau,a)\pi_\theta(S_\tau,a) - \lambda \theta^T \left(\frac{1}{t}\sum_{\tau=1}^tg(S_\tau)g(S_\tau)^T\right)\theta$\\
		$\theta^*$ & $\argmax_{\theta} J\left(\theta,\mu^*\right)$\\
		$\tilde{\theta}_t$ & $\argmax_{\theta}\tilde{J}_t\left(\theta,\hat{\mu}_t\right)$\\
		$\hat{\theta}_t$& $\argmax_{\theta}\hat{J}_t\left(\theta,\hat{\mu}_t\right)$\\
		\hline
	\end{tabular}
\caption{Notations}
\label{notations}
\end{table}

Actor critic algorithms~\citep{konda1999actor, bhatnagar2009natural, vamvoudakis2010online} have received a lot of attention in the reinforcement learning literature as an online approximation to policy iteration. These algorithms keep two separate parametrization: one for the expected reward function given a state-action pair (known as the Q-function in the reinforcement learning literature) and the other for the policy. Since the gradient of policy parameter depends on reward function, actor critic algorithms regularize the reward estimation and therefore reduce the variance of the policy gradient estimate. In our problem setting, we model the reward function $r(s,a)$ as a linear function (the critic step). To estimate the best policy, we maximize the estimated regularized average reward by plugging in the reward function estimate from the critic step. Note that in our simpler contextual bandit setting, we can perform a full maximization over policy parameters instead of just taking a step along the gradient.
\begin{assumption}
[Linear expected reward assumption] Given state-action pair $(s, a)$, the expected reward is a linear function $r(s,a) = f(s,a)^T\mu^*$. Here $f(s,a)$ is a k-dimensional feature vector and $\mu^*$ is an unknown reward parameter. The actual reward is generated by adding an error term on top of expected reward: $R_t = r(s,a) + \epsilon_t$. The error terms $\epsilon_t$ are i.i.d.\ with mean $0$ and finite variance $\sigma^2$. 
\label{assumption_linear}
\end{assumption}
\noindent \textbf{The critic}: the algorithm observes a stream of triples $\{(S_{\tau}, A_{\tau}, R_{\tau} )\}_{\tau=1}^t$ after decision point $t$.  We use penalized least squares to learn the reward parameter:
\begin{align}
\hat{\mu}_t=\left(\zeta I+\sum_{\tau=1}^t f(S_{\tau}, A_{\tau})f(S_{\tau}, A_{\tau})^T\right)^{-1} \sum_{\tau=1}^t f(S_{\tau}, A_{\tau}) R_{\tau}
\end{align}
where the $\zeta$ is the importance of the $L_2$ penalty. This penalty ensures that the matrix inverse in the above formula is well defined since $\sum_{\tau=1}^t f(S_{\tau}, A_{\tau})f(S_{\tau}, A_{\tau})^T$ does not have full rank when $t$ is small. 

\noindent \textbf{The actor}: the algorithm optimizes a plug-in estimation of objection function, where the state space probability density function $d(s)$ is estimated by the empirical distribution of $\{(S_{\tau}\}_{\tau=1}^t$ and expected reward given state-action pair is estimated by
\begin{eqnarray}
r_{\hat{\mu}_t}(s,a) = \left\{
\begin{array}{ll}
-2  & \mbox{ if } f(s,a)^T\hat{\mu}_t < -2 \\
f(s,a)^T\hat{\mu}_t & \mbox{ if } |f(s,a)^T\hat{\mu}_t| \leq 2 \\
2  & \mbox{ if } f(s,a)^T\hat{\mu}_t >2.
\label{reward_projection}
\end{array}
\right.  
\label{equ:est_reward}
\end{eqnarray}
Note that the clipping of the reward estimates to make them stay in the interval $[-2,2]$ is needed in our theoretical arguments and is entirely compatible with our theoretical assumptions made below. However, as we note in our numerical experiments, it does not seem to be required. The algorithm empirically appears to work fine even without the clipping.

An estimate to the aforementioned regularized average reward at time point $t$ is 
\begin{align}
\hat{J}_t(\theta,\hat{\mu}_t)=\frac{1}{t} \sum_{\tau=1}^t \sum_a r_{\hat{\mu}_t}(s,a)  \pi_{\theta}(S_{\tau},a) - \lambda \theta^T \left( \frac{1}{t}\sum_{\tau=1}^t g(S_{\tau}) g(S_{\tau})^T \right)\theta.
\label{emp_reward}
\end{align}

\noindent The actor critic algorithm, which alternates between a critic step and an actor step is depicted in Algorithm \ref{ac}. At each time point, an action is drawn using the estimated optimal policy at the previous time point. Upon receiving the reward, the critic iteratively updates matrices $A(t)$ and $B(t)$ to produce an updated estimate for the reward parameter $\mu^*$. The actor then updates the estimated regularized reward, based on which an update for the estimated optimal policy parameter is produced. 

\begin{algorithm}[H]
{\bf Inputs:} $T$, the total number of decision points; a $k$ dimensional reward feature $f(s,a)$; a $p$ dimensional policy feature $g(s)$. \\
{\bf Critic initialization:} $B(0)= \zeta I_{k\times k}$; $A(0)=\mathbf{0}_{k \times 1}$. \\
{\bf Actor initialization:} $\theta_0$ is initial policy parameter based on domain theory or historical data. \\
Start from $t=0$. \\
\While{$t \leq T$}{
At decision point $t$, observe context $S_t$. \\
Draw an action $A_t$ according to probability distribution $\pi_{\hat{\theta}_{t-1}}(S_t,A)$. \\
Observe an immediate reward $R_t$. \\
{\bf Critic update:}  \\
$B(t) = B(t-1)+ f(S_t,A_t) f(S_t, A_t)^T$, $A(t) = A(t-1) + f(S_t,A_t) R_t$, $\hat{\mu}_t= B(t)^{-1} A(t)$. The estimated reward function is $\hat{r}_{\hat{\mu}_t}(s,a)$ from (\ref{reward_projection}).\\
{\bf Actor update:} 
\begin{align*}
\hat{\theta}_t = \argmax_{\theta} \frac{1}{t}\sum_{\tau=1}^t \sum_{a} r_{\hat{\mu}_t}(S_{\tau},a) \pi_{\theta}(S_{\tau},a) - \lambda \, \theta^T  \left( \frac{1}{t}\sum_{\tau=1}^t g(S_{\tau}) g(S_{\tau})^T \right) \theta .
\end{align*}
Go to decision point $t+1$.
}
\caption{An online actor-critic algorithm with linear expected reward and stochastic policies}
\label{ac}
\end{algorithm}

\section{Convergence and Regret Analysis for the Actor-Critic Algorithm}\label{theory}

In this section, we present consistency and the asymptotic normality results for the proposed actor-critic algorithm. Proofs are provided in the supplementary material. We begin by making some mild assumptions on parameters and features and show that convergence in the regularized average reward can be achieved even when the optimal policy is not uniquely identifiable. However, under an additional uniqueness assumption, we are able to prove convergence in the parameter space and asymptotic normality of our parameter estimates.

\subsection{Convergence in Regularized Average Reward}

Assumption~\ref{assumption_boundedness} is a standard boundedness assumption that can be found in many contextual bandit literature (e.g. \cite{agrawal2012thompson}). We also make mild requirement for the policy features and reward features in Assumption~\ref{assumption_policy_feature} and Assumption~\ref{assumption_reward_feature}.


\begin{assumption} [Bounded rewards and features] The reward feature, reward coefficient and policy feature have bounded norm 1, i.e. $||f(S,A)||_2, ||\mu^*||_2 \leq 1, ||g(S)||_2\leq 1$.
\label{assumption_boundedness}
\end{assumption}

\begin{assumption}[Positive definiteness of policy features] The $p \times p$ matrix $\mathbb{E}[g(S)g(S)^T]$ is positive definite, whose smallest eigenvalue is lower bounded by $\lambda_p$. 
\label{assumption_policy_feature}
\end{assumption}

\begin{assumption}[Eigenvalues for reward features]
The ranks of $k\times k$ matrices:
$$\mathbb{E}\left[f(S,a_0)f(S,a_0)^T\right]\text{ and }\mathbb{E}\left[f(S,a_1)f(S,a_1)^T\right]$$ are both greater than zero, denoted by $r_0$ and $r_1$, the corresponding smallest positive eigenvalues are $\lambda_{r_0}$ and $\lambda_{r_1}$, largest eigenvalues are $\lambda_{1,a_0}$ and $\lambda_{1,a_1}$.
\label{assumption_reward_feature}
\end{assumption}

Theorem~\ref{thm:J_convergence} below shows that, under Assumption~\ref{assumption_iid}$-$\ref{assumption_reward_feature}, the regularized average reward using the policy in Algorithm~\ref{ac} converges to that using the optimal policy.
\begin{theorem}[Convergence of $J\left(\hat{\theta}_t,\mu^*\right)$ to $J\left(\theta^*,\mu^*\right)$] 
	\label{thm:J_convergence}
Under Assumptions~\ref{assumption_iid}$-$\ref{assumption_reward_feature}, for sufficiently small $\epsilon$, if
\begin{align*}
t = \tilde{O}\left(\max\left\{\frac{p}{\epsilon^2\lambda_p^2},\frac{\sigma^2}{\epsilon^2\tilde{p}_0^2\gamma^2},\frac{1}{\gamma^2\tilde{p}_0^2},\frac{\zeta}{\epsilon\tilde{p}_0\gamma}\right\} \log^2\left(\frac{1}{\delta}\right)\right),
\end{align*}
where $\gamma^2:=\min\{\lambda_{r_0}^2,\lambda_{r_1}^2\}$, $\tilde{p}_0 = \frac{1}{1+\exp\left(\sqrt{\frac{2}{\lambda\lambda_p}}\right)}$ and $\delta > 0$, then with probability $1-\delta$ we have,
\begin{align*}
	P\left(\left|J\left(\theta^*,\mu^*\right)-J\left(\hat{\theta}_t,\mu^*\right)\right|\leq 5\epsilon\right)\geq 1-\delta.
\end{align*}
In this theorem, $\widetilde{O}(\cdot)$ ignores poly-log terms non-regarding to $\delta$.

\end{theorem}

However, using this convergence result in Theorem~\ref{thm:J_convergence}, we prove that the regularized cumulative regret of the actor-critic algorithm is $\tilde{O}(T^{2/3})$ in Corollary~\ref{cor:J_convergence_regret}. 
The regularized cumulative regret $\text{Reg}_J$ up to time $T$ is the difference between the regularized cumulative reward under the optimal policy $\theta^*$ and that under the algorithm.
\begin{align*}
	\text{Reg}_J(T) := TJ(\theta^*,\mu^*) - \sum_{t=1}^{T}J(\hat{\theta}_t,\mu^*).
\end{align*}

\begin{corollary}\label{cor:J_convergence_regret}
	Under Assumption~\ref{assumption_iid}$-$\ref{assumption_reward_feature}, for $\delta > 0$, the regularized cumulative regret can be bounded by:
	$\mathrm{Reg}_J(T)  = \tilde{O}\left(T^{2/3}\right)$ with probability at least $1-\delta$.
\end{corollary}
As we will show in the next section, usual parametric rate of convergence $\tilde{O}(\sqrt{T})$ can be achieved with stronger assumptions.

\subsection{Asymptotic Convergence for Parameters in Actor Critic}
In addition to the aforementioned assumptions, we make the following assumption that ensures the identifiability of each component of the policy parameter. We will show the consistency of both the reward parameter and policy parameter, as well as their asymptotic normality. As a byproduct, the asymptotic properties of the policy parameter guarantees that the cumulative regret of Algorithm~\ref{ac} up to time $T$ is $\tilde{O}\left(\sqrt{T}\right)$.
\begin{assumption}\label{assumption:uniq}
(Uniqueness of global maximum and invertibility) The regularized average reward function $J(\theta, \mu^*)$, as a function of $\theta$, achieves the unique global maximum at $\theta=\theta^*$. In addition, for sufficiently small $\delta>0$, there exists $\epsilon>0$ and neighborhood of $\theta^*$, denoted by $B(\theta^*, \delta)$, such that 
\begin{align}
J(\theta^*,\mu^*)-\max_{\theta \notin B(\theta^*,\delta)}J(\theta,\mu^*) \geq \epsilon
\end{align}
We also assume that there exists a neighborhood of $\theta^*$ where $J(\theta, \mu^*)$ is invertible. 
\label{assumption_uniqueness}
\end{assumption}

\begin{assumption}\label{assumption:feature_pd}
(Positive definiteness of reward feature) The matrix $$\mathbb{E}_{\theta^*}(f(S,A)f(S,A)^T)=\int_s d(s) \sum_a f(s,a)f(s,a)^T \pi_{\theta^*}(s,a)ds,$$ which is the expected value of $f(S,A)f(S,A)^T$ under the optimal policy parameter $\theta^*$, is positive definite. 
\end{assumption}

Assumption~\ref{assumption_uniqueness} is a standard assumption in proving consistency and asymptotic normality of M-estimators (\cite{van2000asymptotic}). Under Assumption~\ref{assumption_iid}$-$\ref{assumption:feature_pd}, the following theorems establish the consistency and asymptotic normality of the critic and the actor.
\begin{theorem}
[Asymptotic properties of the critic]  Under Assumption~\ref{assumption_iid}$-$\ref{assumption:feature_pd}, the k-dimensional estimated reward parameter $\hat{\mu}_t$ converges to the true reward parameter $\mu^*$ in probability. In addition, $\sqrt{t}(\hat{\mu}_t-\mu^*)$ converges in distribution to a multivariate normal with mean $\mathbf{0}_{k \times 1}$ and covariance matrix 
$[\mathbb{E}_{\theta^*}(f(S,A)f(S,A)^T)]^{-1}\sigma^2$ and $\sigma$ is the standard deviation of the error term in Assumption~\ref{assumption_linear}. The plug-in estimator of the asymptotic covariance is consistent.
\label{theory_critic}
\end{theorem}
\begin{theorem}
[Asymptotic properties of the actor] Under Assumption~\ref{assumption_iid}$-$\ref{assumption:feature_pd}, the $p$-dimensional estimated optimal policy parameter $\hat{\theta}_t$ converges to $\theta^*$ in probability. In addition, $\sqrt{t}(\hat{\theta}_t-\theta^*)$ converges in distribution to multivariate normal with mean $\mathbf{0}_{p \times 1}$ and covariance matrix $[J_{\theta\theta}(\mu^*,\theta^*)]^{-1} V^* [J_{\theta\theta}(\mu^*,\theta^*)]^{-1}$, where 
\begin{align*}
V^*= \sigma^2 J_{\theta\mu}(\mu^*,\theta^*) \mathbb{E}_{\theta^*}[f(S,A)f(S,A)^T] J_{\mu\theta}(\mu^*,\theta^*)+\mathbb{E} [j_{\theta}(\mu^*,\theta^*,S)j_{\theta}(\mu^*,\theta^*,S)^T].
\end{align*}
In the expression of asymptotic covariance matrix,
\begin{align*}
 j_\theta(\mu,\theta,S)= \frac{\partial}{\partial\theta}\left(\sum_a f(S,a)^T\mu\ \pi_{\theta}(S,a)-\lambda \theta^T [g(S)g(S)^T] \theta\right),
\end{align*}
and   both $J_{\theta\theta}$ and $J_{\theta\mu}$ are the second order partial derivatives with respect to $\theta$ twice and with respect $\theta$ and $\mu$, respectively of $J$:
\begin{align}
J(\mu,\theta)=\int_{s\in \mathcal{S}}d(s)\sum_{a\in\mathcal{A}}f(s,a)^T \mu\ \pi_{\theta}(s,a)ds-\lambda \theta^T \mathbb{E}[g(S)g(S)^T]  \theta.
\label{230}
\end{align}
\noindent Positive definiteness of $J_{\theta\theta}(\mu^*,\theta^*)$ is guaranteed by Assumption~\ref{assumption:uniq}.
\label{theory_actor}
\end{theorem}

A bound on the cumulative regret can be derived as a by-product of the square-root convergence rate of $\hat{\theta}_t$, presented in below corollary. 
The cumulative regret $\text{Reg}_V$ up to time $T$ is the difference between the cumulative reward under the optimal policy $\theta^*$ and that under the algorithm.
\begin{align*}
	\text{Reg}_V(T) = TV(\theta^*) - \sum_{t=1}^{T}V(\hat{\theta}_t).
\end{align*}
\begin{corollary}\label{cor:V_convergence_regret}
	Under Assumption~\ref{assumption_iid}$-$\ref{assumption:feature_pd}, for $\delta > 0$, the cumulative regret $\mathrm{Reg}_V(T)$ of Algorithm \ref{ac} can be bounded by: $\mathrm{Reg}_V(T) = \tilde{O}(\sqrt{T})$ with probability at least $1-\delta$.
\end{corollary}

Readers familiar with contextual bandit literature may wish to  compare the above regret bound with the regret bounds for LinUCB (\cite{chu2011contextual}) and for Thompson sampling (\cite{agrawal2013thompson}). There are at least three differences between our framework and those considered in LinUCB and Thompson sampling papers. First,  we parameterize explicitly both the policy class as well as the expected reward.   These two papers parameterize the expected reward which then implicitly implies a parameterized deterministic policy class. As a result, our optimal policy is the policy that maximizes the regularized average reward over our explicitly defined policy class; however their optimal policy is the policy that maximizes the unregularized average reward.  Second, we restrict our policy class to be stochastic whereas their implicitly defined policy class is composed of deterministic policies. Lastly, the setting considered here is more restrictive in that we assume contexts are i.i.d. whereas these papers allows for arbitrary contexts as long as the conditional mean of the reward in any context is linear in the context features.


\section{Numerical Experiments}\label{numerical}

To assess the performance of the proposed actor-critic algorithm we conducted extensive simulations across  a variety of realistic scenarios.
Firstly, we conduct simulations to evaluate  the relevance of our asymptotic theory in finite $T$ settings in which the contexts are indeed i.i.d.  As will be seen, the bias and mean squared error (MSE) in estimating optimal policy decreases to 0 as sample size increases, and the bootstrap confidence interval for the optimal policy parameter achieves nominal confidence level. Secondly, we note that the i.i.d.\ assumption on the contexts is likely violated in real world applications in at least two ways: the current context may be influenced by the context at previous decision points and the current context may be influenced by past actions.  Simulations in which the  contexts follow an auto-regressive process show that the actor-critic algorithm is quite robust to auto-correlation among contexts. We also create simulation settings where context is influenced by previous actions through a burden effect of the treatments on the users.  We observe reasonable robustness of the bandit actor-critic algorithm when the burden effects are small or moderate. Last but not least, we investigate how performance of the algorithm may deteriorate when part of Assumption~\ref{assumption_iid} is violated, that is, the conditional mean of the reward is non-linear.

Throughout we base the simulations on a generative model that is motivated by the Heartsteps application for improving daily physical activity (\cite{klasnja_microrandomized_2015, dempsey2015randomised}). A simplified description of HeartSteps follows.  HeartSteps is a mobile health smartphone application seeking to reduce users' sedentary behavior and increase physical activity such as walking. A commercial wristband sensor is usted to collect  minute level steps counts. 
Each evening self-report on the usefulness of the application as well as problems in daily life are collected.  
At each of 3 decision points per day, sensor data is collected including the user's current location (home/work/other) and weather. At each decision point, the algorithm on the smartphone application must decide  whether to ``push" a tailored physical activity suggestion, i.e., $A_t=1$, or remain silent, i.e., $A_t=0$.
Our generative model uses a three dimensional context at decision point $t$: $S_{t}=[S_{t,1}, S_{t,2}, S_{t,3}]$. $S_{t,1}$ represents weather, with $S_{t,1}=-\infty$ being extremely severe and unfriendly weather for any outdoor activities and $S_{t,1}=\infty$ being the opposite.  $S_{t,2}$ reflects the user's recent habits in engaging in physical activity. $S_{t,2}=\infty$ represents that the user has been maintaining  positive daily physical habits while $S_{t,2}=-\infty$ represents the opposite. $S_{t,3}$ is a composite measure of disengagement with HeartSteps. $S_{t,3}=-\infty$ reflects an extreme state that the user is fully engaged, is adherent and is reporting that  the application is useful. On the other hand, $S_{t,3}=\infty$ denotes the opposite state of disengagement. Note that although we assumed bounded features in proving theoretical properties of the algorithm, these feature vectors have unbounded range. Results shown in later sections demonstrate  robustness to the boundness assumption.

The goal of HeartSteps is to reduce users' sedentary behavior. Here we reverse code the reward and define the {\emph{cost}} to be the sedentary time per hour between two decision points. So the goal of the actor-critic algorithm is to minimize an average penalized cost as opposed to maximizing an average penalized reward.    The generative model for the cost is a linear model: $C_{t} =10 - .4 S_{t,1}-.4 S_{t,2} -A_{t} \times (0.2 + 0.2 S_{t,1} + 0.2 S_{t,2})   + 0.4 S_{t,3} + \xi_{t,0}$, where $\xi_{t,0}$ are i.i.d. N(0,1) errors. In this linear model, higher values of $S_1$ and $S_2$, good weather and positive physical activity habits, are associated with less sedentary time while a higher value of $S_3$, disengagement, leads to increased sedentary time. The negative main effect of $A_t$ indicates that physical activity suggestion ($A_t=1$) reduces sedentary behavior compared to no suggestion $A_t=0$. The negative interaction between $A_t$ and $S_{t,1}$ and between $A_t$ and $S_{t,2}$ reflects that physical activity suggestions are more effective when the weather condition is activity friendly or when the user has acquired good physical activity habits.

The class of parametrized policies is $\pi_{\theta}(S,1)= \frac{ e^{\theta0+\sum_{i=1}^3\theta_iS_i}}{1+e^{\theta0+\sum_{i=1}^3\theta_iS_i}}$. 
The average cost under policy $\pi_{\theta}$ is:
\begin{align*}
C(\theta)= \int_{s\in\mathcal{S}} d_{\theta}(s) \sum_a \mathbb{E}(C|S=a, A=a) \pi_{\theta}(s,a)ds
\end{align*}
where $d_{\theta}(s)$ is the stationary distribution of context under policy $\pi_{\theta}$. When actions have no impact on context distributions, the stationary distribution $d(s)$ does not depend on the policy parameter $\theta$. In this case, the average cost reduces to: $C(\theta)=  \int_{s\in\mathcal{S}} d(s) \sum_a \mathbb{E}(C|S=a, A=a) \pi_{\theta}(s,a)ds$. This is true for the  generative models we  investigate in Section \ref{iid} and Section \ref{ar}. The  generative model we investigate in Section \ref{burden} allows actions to impact the context distribution at future decision points. In such a case, the stationary distribution of context depends on the policy parameter $\theta$. 
A quadratic constraint is enforced so that the optimal policy is stochastic. In the quadratic inequality ~\eqref{const_q}, we use $\alpha=0.1$ and $p_0=0.1$ throughout the numerical experiment unless otherwise specified. We then minimize the corresponding Lagrangian function. 


\textbf{The optimal policy $\theta^*$ and the oracle $\lambda^*$.} According to results  in Section \ref{sec_reg_reward}, we have that for every pair of $(p_0,\alpha)$ there exists a Lagrangian multiplier $\lambda^*$ such that the optimal solution to the regularized average cost function:
\begin{align}
\theta^*=\argmin_{\theta} C(\theta) + \lambda \theta^T \sum_s d_{\theta} ([1,s_1,s_2,s_3] [1,s_1,s_2,s_3]^T)\theta
\label{regularized_cost}
\end{align}
satisfies the quadratic constraint with equality. Furthermore, as $\lambda$ increases the stringency of the quadratic constraint increases: an increased value of $\lambda$ penalizes  the quadratic term $\theta^{*T} \sum_s d_{\theta^*} ([1,s_1,s_2,s_3] [1,s_1,s_2,s_3]^T)\theta^*$ more heavily. For a fixed pair of $(p_0,\alpha)$, we perform a line search to find the smallest $\lambda$, denoted as $\lambda^*$, such that the minimizer to the regularized average cost, denoted as $\theta^*$ satisfies the quadratic constraint. We recognize the difficulty in solving the optimization problem due to the non-convexity of the regularized average cost function. In our search for a global minimizer, we therefore use grid search, for a given $\lambda$, to find a crude solution to the optimization problem. We then improve the accuracy of the optimal solution using a more refined grid search provided by the pattern search function in Matlab.
The regularized average cost function is approximated by Monte Carlo samples. We used 5000 Monte Carlo samples to approximate the regularized average cost for simulation in Section \ref{iid} and Section \ref{ar} where the stationary distribution of contexts does not depend on the policy. For the simulations in Section \ref{burden}, where context distribution does depend on the policy, we generate a trajectory of 100000 Monte Carlo samples and ignore the first $10\%$ of the samples to approximate the stationary distribution. 

\textbf{Estimating $\lambda$ online}. In practice, the decision maker has no access to the oracle Lagrangian multiplier $\lambda^*$. A natural remedy is to integrate the estimation of $\lambda^*$ with the online actor-critic algorithm that estimates the policy parameters. An actor-critic algorithm with a fixed Lagrangian multiplier solves the ``primal" problem while the ``dual" problem is solved by searching for $\lambda^*$. Our integrated algorithm performs a line search to find the smallest $\lambda$ such that the estimated optimal policy satisfies the quadratic constraint. The stationary distribution of the contexts is approximated by the empirical distribution. Estimating $\lambda$ can be very time consuming, therefore in our simulations, the algorithm performs the line search over $\lambda$ only every 10 decision points. Similar ideas with gradient based updates on $\lambda$ have appeared in reinforcement literature to find the optimal policies in constrained MDP problems, see \cite{borkar2005actor, bhatnagar2012online} for examples. 

\textbf{Bootstrap confidence intervals.} In a number of trial simulations, we found that the plug-in variance estimator derived from Theorem \ref{theory_actor} tends to underestimate in small to moderate sample size, a direct consequence of which is the anti-conservatism of the Wald confidence interval. Details of the anti-conservatism are discussed in the supplementary material section \ref{small}. Our solution to the anti-conservative Wald confidence interval is the percentile-t bootstrap confidence interval. Algorithm \ref{boot} shows how to generate a bootstrap sample. Algorithm \ref{boot} is repeated for a total of $B$ times to obtain a bootstrap sample of the estimated optimal policy parameters, $\{\hat{\theta}^b_T\}_{b=1}^B$ and plug-in variance estimates, $\{\hat{V}^b_T\}_{b=1}^B$. We create bootstrap percentile-t confidence intervals for $\theta^*_i$, the i-th component of the optimal policy parameter. For each $\theta_i^*$, we use the empirical percentile of $\left\{ \frac{\sqrt{t}(\hat{\theta}^b_{T,i}-\hat{\theta}_{T,i})}{\sqrt{\hat{V}^b_T}} \right\}_{b=1}^B$, denoted by $p_{\alpha}$ to replace the normal distribution percentile in Wald confidence intervals. A $(1-2\alpha)\%$ confidence interval is 
\begin{align}
\left[ \hat{\theta}_{T,i} - p_{\alpha}\frac{\hat{V}_i}{\sqrt{T}}, \hat{\theta}_{T,i} + p_{\alpha}\frac{\hat{V}_i}{\sqrt{T}} \right]
\end{align}
where $\hat{\theta}_{T,i}$ is the i-th component of $\hat{\theta}_{T}$ and $\hat{V_i}$ is the plug-in variance estimate based on the original sample.

\begin{algorithm}
Inputs: The observed context history $\{S_t\}_{t=1}^T$. A bootstrap sample of residuals $\{\epsilon^b_t\}_{t=1}^T$. The estimated reward parameter $\hat{\mu}_T$ \\
Critic initialization: $B(0)= \zeta I_{k\times k}$, a $k \times k$ identity matrix. $A(0)=0_k$ is a $k \times 1$ column vector. \\
Actor initialization: $\hat{\theta}^b_0=\hat{\theta}_0$ is the best treatment policy based on domain theory or historical data. \\
\While{$t < T $}{
Context is $S_t$ \;
Draw an action $A_t^b$ according to policy $\pi_{\hat{\theta}^b_{t-1}}$ \;
Generate a bootstrap reward $R_t^b = f(S_t, A_t^b)^T \hat{\mu}_T + \epsilon_t^b\ $ ; \\
Critic update:  \\
$B(t)=B(t-1)+ f(S_t,A_t)f(S_t,A_t)^T$, $A(t)=A(t-1)+f(S_t,A_t)R_t^b$ \;
$\hat{\mu}_t^b = A(t)^{-1}B(t) $. The bounded estimate to reward function is $\hat{r}_t^b(s,a)$. \;
Actor update: 
\begin{align*}
\hat{\theta}_t^b = \argmax_{\theta} \frac{1}{t}\sum_{\tau=1}^t \sum_{a} \hat{r}_t^b(S_{\tau},a ) \pi_{\theta}(a|S_t)-  
 \lambda \theta^T  [\frac{1}{t}\sum_{\tau=1}^t g(S_{\tau},1)^Tg(S_{\tau},1)] \theta
\end{align*}
Go to decision point $t+1$ \;}
Plugin $\hat{\mu}^b_T$ and $\hat{\theta}^b_T$ to the asymptotic variance formula to get a bootstrapped variance estimate $\hat{V}^b_T$. 
\caption{Generating a bootstrap sample estimate $\hat{\theta}^b_T, \hat{V}^b_T$ }
\label{boot}
\end{algorithm}

\textbf{Simulation details.} The simulation results presented in the following sections are based on 1000 independent simulated users. For each simulated user, we allow a burn-in period of 20 decision points. During the burn-in period, actions are chosen by fair coin flips. After the burn-in period, the online actor-critic algorithm is implemented to learn the optimal policy and obtain an end-of-study estimated optimal policy at the last decision point.  In these simulations we did not force  $\hat{r}_t(s,a)$ to be in the interval $[-2,2]$ as in Algorithm~\ref{ac} or  Algorithm~\ref{boot}. We do not encounter any issues in convergence of the algorithm. 

Both bias and MSE shown in all of the following tables are averaged over 1000 end-of-study estimated optimal policies. For each simulated user the $95\%$ bootstrapped confidence intervals for $\theta^*$ is based on 500 bootstrapped samples generated by Algorithm \ref{boot}. With $95\%$ confidence, we expect that the empirical coverage rate of a confidence interval should be within $0.936$ and $0.964$, if the true confidence level is $0.95$.


\subsection{I.I.D. Contexts}\label{iid}

In this generative model, we choose the simplest setting where contexts at different decision points are i.i.d. We generate contexts $ \{ [S_{t,1}, S_{t,2}, S_{t,3}] \}_{t=1}^T$  from a multivariate normal distribution with mean 0 and identity covariance matrix. The population optimal policy is $\theta^*=[0.417778, 0.394811, 0.389474, 0.001068]$ at $\lambda^*=0.046875$. Table \ref{iid_be} lists the bias and mean squared error (MSE) of the estimated optimal policy parameters. Both measures shrink towards 0 as $T$, sample size per simulated user, increases from 200 to 500, which is consistent with the convergence in estimated optimal policy parameter as established in Theorem \ref{theory_actor}. Table \ref{iid_c} shows the empirical coverage rates of percentile-t bootstrap confidence interval at sample sizes 200 and 500. At sample size 200, the empirical coverage rates are between $0.936$ and $0.964$ for all $\theta_i$'s. At sample size 500, however, the bootstrap confidence interval for $\theta_2$ is a little conservative with an empirical coverage rate of $0.968$. 

\begin{table}
  \begin{tabular}{l | *{4}{S[table-auto-round,
                 table-format=-1.3]}| *{4}{S[table-auto-round,
                 table-format=-1.3]}}
    \toprule
    \multirow{2}{*}{T (sample size)} &
      \multicolumn{4}{c}{Bias} &
      \multicolumn{4}{c}{MSE} \\ \cline{2-9}
      & {$\theta_0$} & {$\theta_1$} & {$\theta_2$} & {$\theta_3$} & {$\theta_0$} & {$\theta_1$} & {$\theta_2$} & {$\theta_3$} \\
      \midrule
    200 &-0.081295&-0.090014&-0.089029&0.010305 &0.053756&0.052246&0.052209&0.055244\\\hline
    500 &-0.05266&-0.037185&-0.03383&-0.001537 &0.026866&0.023746&0.02144&0.029489\\\hline
    \bottomrule
  \end{tabular}
  \caption{I.I.D. contexts: bias and MSE in estimating the optimal policy parameter. Bias=$\mathbb{E}(\hat{\theta}_T)-\theta^*$}. 
  \label{iid_be}
\end{table}

\begin{table}
\centering
\begin{tabular} {c| cccc}
 \toprule
T(sample size) & $\theta_0$ & $\theta_1$ & $\theta_2$ & $\theta_3$ \\ \hline
200&0.962&0.942&0.938&0.945\\\hline
500&0.96&0.948&0.968&0.941\\\hline
\bottomrule
\end{tabular}
\caption{I.I.D. contexts: coverage rates of percentile-t bootstrap confidence intervals for the optimal policy parameter. }
\label{iid_c}
\end{table}


\subsection{AR(1) Context}\label{ar}

In this section, we study the performance of the actor-critic algorithm when the dynamics of the context is an auto-regressive stochastic process. We envision that in many health applications, contexts at adjacent decision points are likely to be correlated. Using HeartSteps as an example, weather ($S_1$) at two adjacent decisions points are likely to be similar. So are users' learning ability ($S_2$) and disengagement level $S_3$. One way to incorporate the correlation among contexts at near-by decision points is through a first order auto-regression process. We simulate the context according to
\begin{eqnarray*}
S_{t,1} &=& 0.4 S_{t-1,1} + \xi_{t,1 },\\
S_{t,2} &=& 0.4 S_{t-1,2} + \xi_{t,2 }, \\
S_{t,3} &=& \xi_{t,3}
\end{eqnarray*}
Here we choose $\xi_{t,1 } \sim N(0, 1-0.4^2)$, $\xi_{t,2 }\sim N(0,1-0.4^2)$ and $\xi_{t,3 }\sim N(0,1)$ so that the stationary distribution of $S_t$ is multivariate normal with zero mean and identity covariance matrix, same as the distribution of $S_t$ in the previous section. The initial distribution of $S_{t}, t=1$ is a multivariate standard normal.  

The oracle Lagrangian multiplier is $\lambda^*=0.05$ and the population optimal policy is $\theta^*=[0.417,  0.395,  0.394,  0]$, same as in the i.i.d.\ simulation. Bias and MSE of the estimated policy parameters are shown in Table \ref{ar_be}. Empirical coverage rate of the percentile t bootstrap confidence interval is reported in Table \ref{ar_c}. Both the bias and MSE diminish towards 0 as the sample size increases from 200 to 500, a clear indication that convergence of the algorithm is not affected by the auto-correlation in context. The bootstrap confidence interval for $\theta_3$ is anti-conservative at sample size 200, but recovers decent coverage at sample size 500. 

\begin{table}
  \begin{tabular}{l | *{4}{S[table-auto-round,
                 table-format=-1.3]} |*{4}{S[table-auto-round,
                 table-format=-1.3]}}
    \toprule
    \multirow{2}{*}{T (sample size)} &
      \multicolumn{4}{c}{Bias} &
      \multicolumn{4}{c}{MSE} \\ \cline{2-9}
      & {$\theta_0$} & {$\theta_1$} & {$\theta_2$} & {$\theta_3$} & {$\theta_0$} & {$\theta_1$} & {$\theta_2$} & {$\theta_3$} \\
      \midrule
    200&-0.092765&-0.088589&-0.075788&0.005926 &0.057522&0.053404&0.047061&0.056805 \\\hline
    500&-0.046294&-0.032283&-0.039641&-0.004681 &0.024747&0.021956&0.023911&0.028052\\\hline
    \bottomrule
  \end{tabular}
  \caption{AR(1) contexts: bias and MSE in estimating the optimal policy parameter. Bias=$\mathbb{E}(\hat{\theta}_T)-\theta^*$.}
  \label{ar_be}
\end{table}

%

\begin{table}
\centering
\begin{tabular} {c| cccc}
\toprule
T(sample size) & $\theta_0$ & $\theta_1$ & $\theta_2$ & $\theta_3$ \\\hline
200&0.963&0.952&0.957&0.927*\\\hline
500&0.969&0.962&0.96&0.949\\\hline
\bottomrule
\end{tabular}
\caption{AR(1) contexts: coverage rates of percentile-t bootstrap confidence intervals. Coverage rates significantly lower than $0.95$ are marked with asterisks (*).}
\label{ar_c}
\end{table}

\subsection{Actions Cause Increased Burden}\label{burden}

In this section, we study behavior of the actor-critic algorithm in the presence of an intervention burden effect. Our generative model with a burden effect represents a scenario where users disengage with the Heartsteps application, and hence the recommended intervention, if the application provides physical activity suggestions at too high a frequency. When users experience intervention burden effects, they become frustrated and have a tendency of falling back to their sedentary behavior. In our burden effect generative model, $S_{t,3}$ represents the disengagement level whose value increases if there is a physical activity suggestion at the previous decision point $A_{t-1}=1$. The positive main effect of $S_{t,3}$ in the cost model \eqref{burden_cost} below reflects that higher disengagement level is associated with higher cost (higher sedentary time). The initial distribution of $S_t$ is the standard multivariate normal distribution. After the first decision point, contexts are generated according to the following stochastic process:
\begin{eqnarray*}
S_{t,1} &=& 0.4 S_{t-1,1} + \xi_{t,1 },\\
	S_{t,2} &=& 0.4 S_{t-1,2}+ \xi_{t,2}, \\
	S_{t,3} &=& 0.4 S_{t-1,3} + 0.2 S_{t-1,3}A_{t-1} + 0.4 A_{t-1} + \xi_{t,3}
\end{eqnarray*}
We simulate the cost, sedentary time per hour between two decision points, according to the following linear model:
\begin{align}
C_{t} &=& 10 - .4 S_{t,1}-.4 S_{t,2} -A_{t} \times (0.2 + 0.2 S_{t,1} + 0.2 S_{t,2})   + \tau S_{t,3} + \xi_{t,0}. 
\label{burden_cost}
\end{align}
where parameter $\tau$ controls the ``size" of the burden effect: the larger $\tau$ is, the more severe the burden effect is. We study the performance of our algorithm in five different cases corresponding to $\tau=0,0.2,0.4,0.6,0.8$. Different values of $\tau$ represent users who experience different levels of burden effect. $\tau=0$ represents the type of users who experience no burden effect while $\tau=0.8$ represents the type of users who experience a large burden effect.   

Table \ref{burden_opt} in the supplementary material section \ref{sup_burden} lists the oracle $\lambda^*$ and the corresponding optimal policy $\theta^*$ at different levels of burden effect. Higher level of burden effects calls for increased value of oracle $\lambda^*$ to keep the desired intervention variety. The negative sign of $\theta_3^*$ at $\tau \geq 0.2$ indicates that the application should lower the probability of pushing an activity suggestion when the disengagement level is high. The magnitude of $\theta_3^*$ rises with the size of the burden effect, implying that as burden effect increases the application should further lower the probability of pushing activity suggestions at high disengagement level. $\theta_0^*$ decreases to be negative when $\tau$ increases, which indicates that as the size of burden effect grows, the application should lower the frequency of activity suggestions in general. 


Table \ref{burden_bm200} and \ref{burden_c200} list the bias, MSE and the empirical coverage rate of the percentile-t bootstrap confidence interval at sample size 200. Table \ref{burden_bm500} and \ref{burden_c500} list these three measures at sample size 500. When there is no burden effect ($\tau=0$), $S_{t,3}$ has no influence on the cost and is therefore considered as a ``noise" variable. The optimal policy parameters are estimated with low bias and MSE under the generative model with $\tau=0$ and the bootstrap confidence intervals have decent coverage, both of which are clear indications that the algorithm is robust to presence of noise variables that are affected by previous actions. As burden effects levels go up, we observe an increased bias and MSE in the estimated optimal policy parameters, $\theta_0$ and $\theta_3$ in particular. The empirical coverage rates of bootstrap confidence intervals for $\theta_0$ and $\theta_3$ are below the nominal $95\%$ level. There are two reasons to explain the increased bias and MSE. The most important one is the near-sightedness of bandit actor-critic algorithm. The bandit algorithm chooses the policy that maximizes the (immediate) average cost while ignoring the negative consequence of a physical activity suggestion $A_t=1$ on the disengagement level at the next decision point. The bandit algorithm therefore tends to ``over-treat" in general and in particular at high disengagement level, which is reflected in an over-estimated $\theta_0$ and $\theta_3$. The second reason comes from the bias in estimating $\lambda$, the Lagrangian multiplier. The oracle Lagrangian multiplier $\lambda^*$ is chosen so that the optimal policy parameter satisfies the quadratic constraint while the online bandit actor-critic algorithm estimates the Lagrangian multiplier so that the bandit-estimated optimal policy satisfies the quadratic constraint. To separate the consequence of underestimated $\lambda$ from the consequence of the myopia of the bandit algorithm, we implement the bandit algorithm with oracle $\lambda^*$. Results of these experiments are shown in the supplementary material section \ref{sup_burden}. We observe that, even with the use of oracle $\lambda^*$, the overestimation of $\theta_0$ and $\theta_3$ as well as the anti-conservatism of the confidence intervals are still present. 

Overall, the estimation of $\theta_1$ and $\theta_2$ shows robustness to the presence of burden effects. $\theta_1$ and $\theta_2$ are estimated with low bias and MSE under the presence of small to moderate burden effects ($\tau=0.2,0.4$). While we observe biases in estimating $\theta_1$ and $\theta_2$ under moderate to large burden effects ($\tau=0.6,0.8$), the magnitude of such bias increases slowly with the size of the burden effect. Empirical coverage rates of the bootstrap confidence intervals for $\theta_1$ and $\theta_2$ are decent for $\tau=0.2,0.4$ and only degrade slowly under $95\%$ when $\tau=0.6,0.8$. 


\begin{table}[H]
  \begin{tabular}{l | *{4}{S[table-auto-round,
                 table-format=-1.3]}| *{4}{S[table-auto-round,
                 table-format=-1.3]}}
    \toprule
    \multirow{2}{*}{{$\tau$}} &
      \multicolumn{4}{c}{Bias} &
      \multicolumn{4}{c}{MSE} \\ \cline{2-9}
      & {$\theta_0$} & {$\theta_1$} & {$\theta_2$} & {$\theta_3$} & {$\theta_0$} & {$\theta_1$} & {$\theta_2$} & {$\theta_3$} \\
      \midrule
    0&-0.027352&-0.035565&-0.030344&0.003449 &0.057811&0.03716&0.036343&0.035898\\\hline
    0.2&0.22947&-0.092877&-0.10406&0.16421 &0.10961&0.044463&0.046192&0.062836\\\hline
    0.4&0.50586&-0.063199&-0.035223&0.23473 &0.31295&0.039819&0.03665&0.090984\\\hline
    0.6&0.64507&0.042695&0.072542&0.27198 &0.47309&0.037714&0.040625&0.10984\\\hline
    0.8&0.70229&0.083867&0.09608&0.2718 &0.55024&0.042799&0.04454&0.1097\\\hline
    \bottomrule
  \end{tabular}
  \caption{Burden effect: bias and MSE in estimating the optimal policy parameter at sample size 200. Bias=$\mathbb{E}(\hat{\theta}_T)-\theta^*$.}
  \label{burden_bm200}
\end{table}

%

\begin{table}[H]
\centering
\begin{tabular} {c| cccc}
\toprule
$\tau$ & $\theta_0$ & $\theta_1$ & $\theta_2$ & $\theta_3$ \\ \hline
0&0.963&0.963&0.955&0.942\\\hline
0.2&0.853*&0.946&0.937&0.862*\\\hline
0.4&0.565*&0.96&0.954&0.776*\\\hline
0.6&0.39*&0.937&0.916*&0.739*\\\hline
0.8&0.329*&0.908*&0.899*&0.739*\\\hline
\bottomrule
\end{tabular}
\caption{Burden effect: coverage rates of percentile-t bootstrap confidence intervals for the optimal policy parameter at sample size 200. $\lambda$ is estimated online. Coverage rates significantly lower than $0.95$ are marked with asterisks (*).}
\label{burden_c200}
\end{table}

\begin{table}[H]
  \begin{tabular}{l | *{4}{S[table-auto-round,
                 table-format=-1.3]}|*{4}{S[table-auto-round,
                 table-format=-1.3]}}
    \toprule
    \multirow{2}{*}{{$\tau$}} &
      \multicolumn{4}{c}{Bias} &
      \multicolumn{4}{c}{MSE} \\ \cline{2-9}
      & {$\theta_0$} & {$\theta_1$} & {$\theta_2$} & {$\theta_3$} & {$\theta_0$} & {$\theta_1$} & {$\theta_2$} & {$\theta_3$} \\
      \midrule
    0&0.005989&0.009646&0.01695&-0.007669 &0.02729&0.018121&0.016362&0.018669\\\hline
    0.2&0.26259&-0.047761&-0.056703&0.15319 &0.096205&0.020303&0.019279&0.042104\\\hline
    0.4&0.5391&-0.017831&0.011807&0.2237 &0.3179&0.018404&0.016187&0.068678\\\hline
    0.6&0.67808&0.088146&0.1196&0.26111 &0.48706&0.025846&0.030356&0.086812\\\hline
    0.8&0.73514&0.12898&0.14313&0.26107 &0.56775&0.034717&0.036534&0.086779\\\hline
    \bottomrule
  \end{tabular}
  \caption{Burden effect: bias and MSE in estimating the optimal policy parameter at sample size 500. Bias=$\mathbb{E}(\hat{\theta}_t)-\theta^*$.}
  \label{burden_bm500}
\end{table}

%

\begin{table}[H]
\centering
\begin{tabular} {c| cccc}
\toprule
$\tau$ & $\theta_0$ & $\theta_1$ & $\theta_2$ & $\theta_3$ \\ \hline
0&0.973&0.949&0.955&0.942\\\hline
0.2&0.714*&0.95&0.962&0.788*\\\hline
0.4&0.217*&0.951&0.961&0.635*\\\hline
0.6&0.101*&0.886*&0.835*&0.545*\\\hline
0.8&0.07*&0.806*&0.788*&0.546*\\\hline
\bottomrule
\end{tabular}
\caption{Burden effect: coverage rates of percentile-t bootstrap confidence intervals for the optimal policy parameter at sample size 200. $\lambda$ is estimated online. Coverage rates significantly lower than $0.95$ are marked with asterisks (*).}
\label{burden_c500}
\end{table}

Figure \ref{burden_pv200} and \ref{burden_pv500} assess the quality of the estimated optimal policies by comparing the regularized average cost with the optimal regularized average cost. Figure \ref{burden_pv200} does the comparison at five levels of burden effect: $\tau=0,0.2,0.4,0.6,0.8$, at sample size 200. As the burden effects level up, the overall long-run average cost goes up, which is simply an artifact of the increasing main effect size of the disengagement level. Having a higher long-term average cost, the estimated optimal policy by the contextual bandit algorithm is always inferior then the optimal policy. The inferiority gap, as measure by the difference between the median long-run average cost and the long-run average cost of the optimal policy increases as $\tau$ increases. When sample size increases from 200 to 500, we observe less variation in the long-run average cost of the estimated optimal policies. Nevertheless, the gap remains stable. We also observe that the variance in the regularized average cost increases as the burden effect level goes up. 

\begin{figure}[H] 
  \begin{minipage}[b]{0.5\linewidth}
    \centering
    \includegraphics[scale=.4]{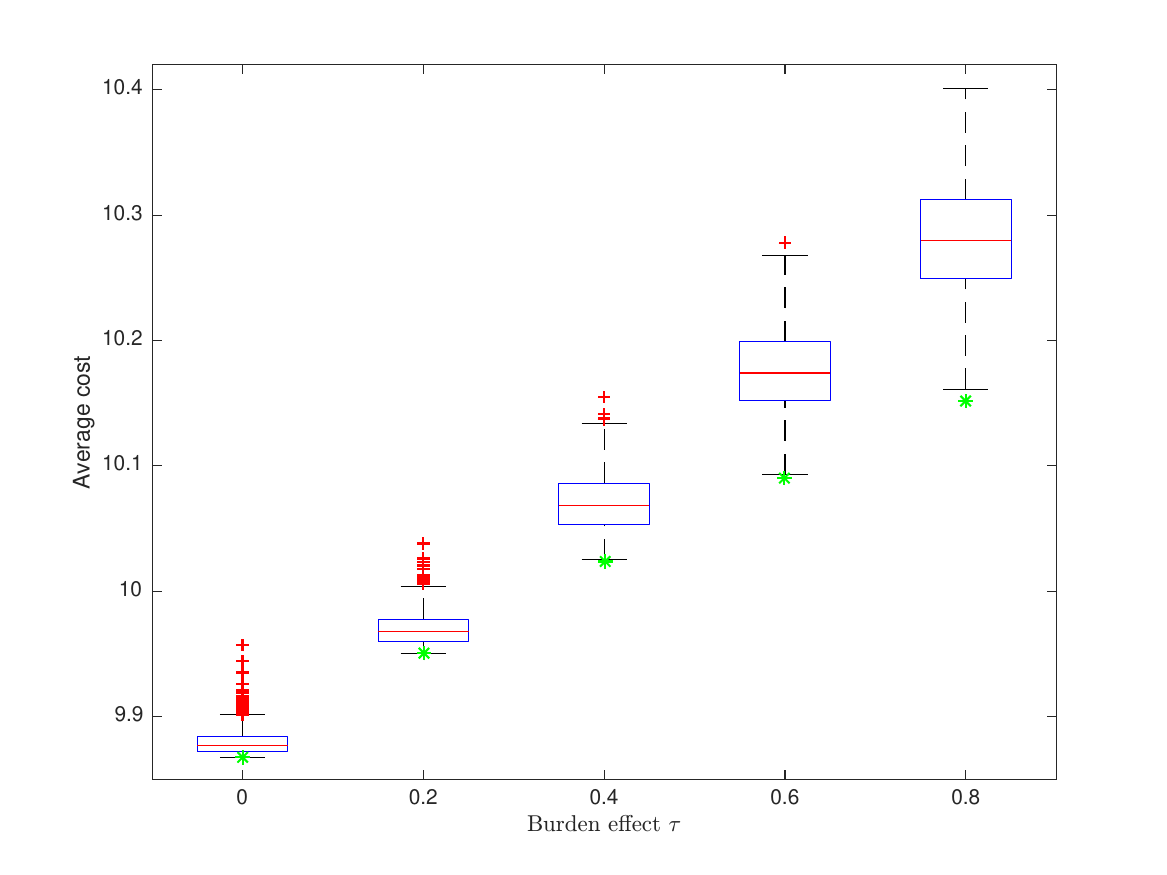} 
    \caption{Burden effect: box plots of \\regularized average cost at different levels \\of the burden effect at sample size 200.} 
    \label{burden_pv200}
    \vspace{4ex}
  \end{minipage}
  \begin{minipage}[b]{0.5\linewidth}
    \centering
    \includegraphics[scale=.4]{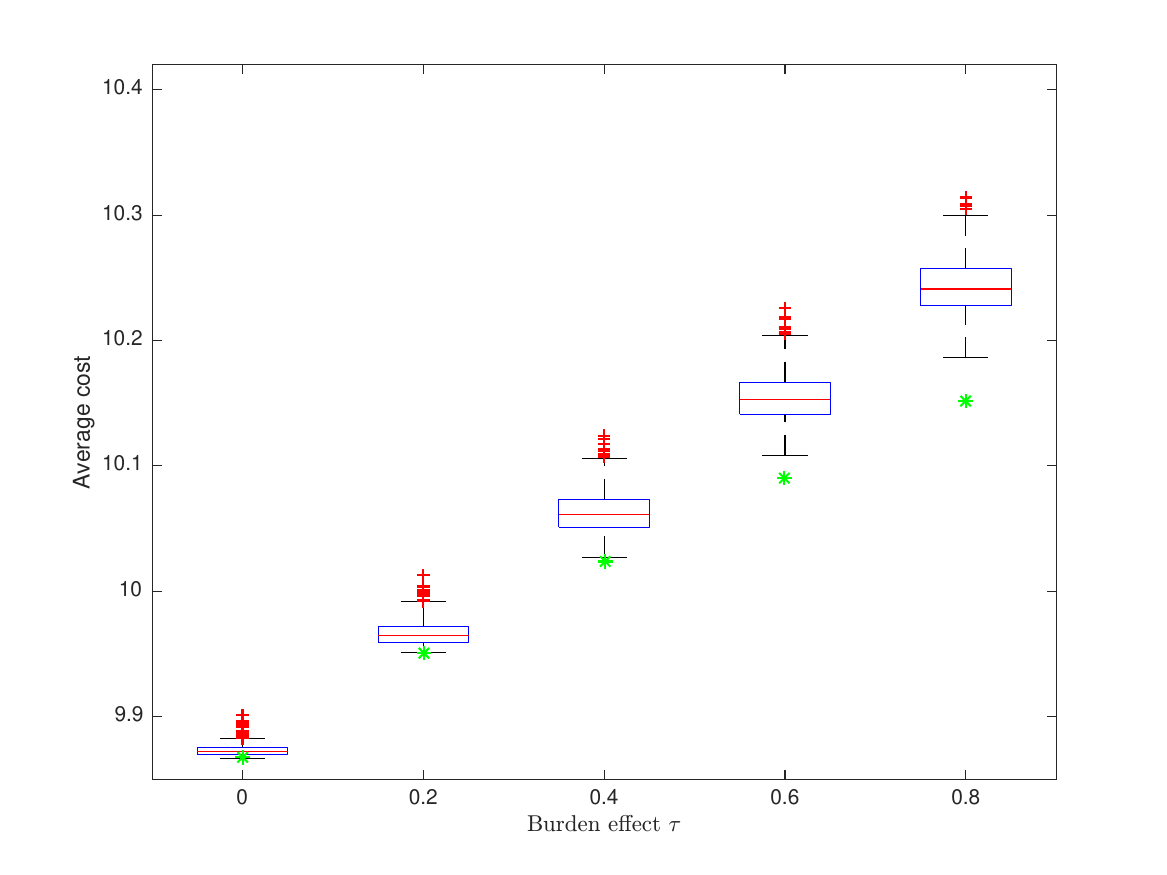} 
    \caption{Burden effect: box plots of regularized average cost at different levels of the burden effect at sample size 500.} 
    \label{burden_pv500}
    \vspace{4ex}
  \end{minipage} 
\end{figure}

%

Because our theoretical results assume i.i.d.\ contexts, we have no proof that the optimal policy estimated by the bandit actor-critic algorithm will converge to the optimal policy. Nevertheless, we observe convergence in the estimated policy as sample size $T$ grows. We conjecture that, when actions affect contexts distributions, the bandit algorithm converges to the policy $\pi_{\theta^{**}}$ that satisfies the following equilibrium equation:

\begin{align}
\theta^{**} = \argmin_{\theta} \sum_s d_{\theta^{**}}(s) \sum_a \pi_{\theta}(a|s) \mathbb{E}(C|A=a,S=s) - \lambda^{**} \theta^T \mathbb{E}_{\theta^{**}}[g(S)g(S)^T] \theta \\ 
\mbox{where } \lambda^{**} \mbox{ is the smallest } \lambda \mbox{ such that }\theta^{**} \sum_s d_{\theta^{**}}(s) g(s)^Tg(s) \theta^{**} \leq (\log(\frac{p_0}{1-p_0}))^2 \alpha
\label{myopic_eq}
\end{align}
When actions do not influence contexts distributions, the equilibrium equation is the same system of equations satisfied by the optimal policy. When previous actions have an impact on context distribution at later decision points, the stationary distribution of context is a function of policy. We call solution to equation \ref{myopic_eq}, the \emph{myopic equilibrium policy}. The myopic equilibrium policy minimizes the regularized average cost under the stationary distribution generated by itself. Such policy achieves an ``equilibrium state'' and there is no reason for the actor-critic to change the current policy if a myopic equilibrium has been reached. 
The conjecture is supported by our numerical results. Since myopic equilibrium policy only depends on the context dynamics and the treatment effect $\mathbb{E}(C|A=1,S=s)-\mathbb{E}(C|A=0,S=s)$, it remains the same at different levels of the burden effect. The myopic equilibrium policy is $\theta^{**} =[0.392, 0.372, 0.371, 0]$.  The bias and MSE in estimating the myopic equilibrium policy for $\tau=0.4$ is shown in table \ref{burden_bandit_bm}. The bias and MSE at other levels of the burden effect are the same. These results support with our conjecture that the estimated optimal policy by the bandit algorithm converges to the myopic equilibrium policy. 


\begin{table}[H]
  \begin{tabular}{l | *{4}{S[table-auto-round,
                 table-format=-1.3]}|*{4}{S[table-auto-round,
                 table-format=-1.3]}}
    \toprule
    \multirow{2}{*}{{Sample size (T)}} &
      \multicolumn{4}{c}{Bias} &
      \multicolumn{4}{c}{MSE} \\\cline{2-9}
      & {$\theta_0$} & {$\theta_1$} & {$\theta_2$} & {$\theta_3$} & {$\theta_0$} & {$\theta_1$} & {$\theta_2$} & {$\theta_3$} \\
      \midrule
    200 &-0.078345&-0.080799&-0.075323&0.004031&0.063196&0.042354&0.041083&0.035901\\\hline
    500&-0.0451&-0.035431&-0.028293&-0.0070048&0.029301&0.019342&0.016848&0.018687\\\hline
    \bottomrule
  \end{tabular}
  \caption{Burden effect: bias and MSE in estimating the myopic equilibrium policy for $\tau=0.4$. Bias=$\mathbb{E}(\hat{\theta}_t)-\theta^{**}$.}
  \label{burden_bandit_bm}
\end{table}

\subsection{Expected Cost is a Nonlinear function of the Cost Feature}\label{nonlin}

In this section, we investigate the performance of the online actor critic algorithm when the expected cost is a nonlinear function of the cost feature used in the critic step. In such scenarios, the linear actor critic algorithm finds the ``best" policy in two steps: first it projects the true cost function into the linear space spanned by the cost feature, then it finds the policy that minimizes the regularized cost function under the projection. In contrast, the true optimal policy is the policy that minimizes the regularized cost function without the projection. In this simulation, we are interested to see how the extra step of projection affects the estimation and inference of the optimal policy parameter. 

Recall that the cost feature is $f(S_t,A_t)=[1,S_{t,1},S_{t,2},S_{t,3},A_t,A_tS_{t,1},A_tS_{t,2},A_tS_{t,3}]$. In particular consider the case where the interaction term between $A_t$ and $S_{t,1}$ is a linear combination of a linear cost function and a nonlinear one:
\begin{align*}
C_{t} &= (1-\alpha)[10 - .4 S_{t,1}-.4 S_{t,2} -A_{t} \times (0.2 + 0.2 S_{t,1} + 0.2 S_{t,2})   + 0.4 S_{t,3} + \xi_{t,0} ] \\
 &   + \alpha[10 - .4 S_{t,1}^2-.4 S_{t,2} -A_{t} \times (0.2 + 0.2 S_{t,1}^2 + 0.2 S_{t,2})   + 0.4 S_{t,3} + \xi_{t,0}]\\
 &= 10 - .4 [(1-\alpha)S_{t,1}+\alpha S_{t,1}]-.4 S_{t,2} -A_{t} \times (0.2 + 0.2 [(1-\alpha)S_{t,1}+\alpha S_{t,1}] + 0.2 S_{t,2})  \\
 &+ 0.4 S_{t,3} + \xi_{t,0}
\end{align*}
The tuning parameter $\alpha\in [0,1]$ controls the amount of nonlinearity: when $\alpha=0$, the expected cost is the linear cost function used in the previous sections. Nonlinearity increasingly  dominates the interaction between $S_{t,1}$ and $A_t$ as $\alpha$ increases. The online actor critic algorithm, unaware of the possible nonlinearity in the cost function, uses the same cost feature and the same policy feature as in the previous sections. Recall the policy is parameterized as $\pi_{\theta}(S,1)= \frac{ e^{\theta0+\sum_{i=1}^3\theta_iS_i}}{1+e^{\theta0+\sum_{i=1}^3\theta_iS_i}}$.  Table \ref{nonlin_opt} in the supplementary material section \ref{sup_nonlin} provides the optimal $\theta$ values.  Table \ref{nonlin_bm200} and Table \ref{nonlin_bm500} show the bias and MSE of the linear actor critic algorithm at different levels of nonlinearity at sample size 200 and 500. The bias for estimating $\theta_i^*$, $i=1,2,3$ remains stable whereas the MSE inflates as the $\alpha$ increases. Both the bias and MSE for estimating $\theta_0^*$ increase as the cost function moves away from a linear structure. Table \ref{nonlin_c200} and table \ref{nonlin_c500} show the coverage rates of the confidence interval for $\theta^*$. The confidence interval coverages for $\theta_i^*$, $i=0.1,2$ deteriorate as the level of nonlinearity increases. However, the confidence level for $\theta_3^*=0$, the coefficient for $S_{t,3}$ which is not a useful tailoring variable, remains decent as the level of nonlinearity increases.


\begin{table}[H]
  \begin{tabular}{l | *{4}{S[table-auto-round,
                 table-format=-1.3]}| *{4}{S[table-auto-round,
                 table-format=-1.3]}}
    \toprule
    \multirow{2}{*}{{$\alpha$}} &
      \multicolumn{4}{c}{Bias} &
      \multicolumn{4}{c}{MSE} \\ \cline{2-9}
      & {$\theta_0$} & {$\theta_1$} & {$\theta_2$} & {$\theta_3$} & {$\theta_0$} & {$\theta_1$} & {$\theta_2$} & {$\theta_3$} \\
      \midrule
    0&-0.10027&-0.073682&-0.10293&-0.006647&0.062844&0.052352&0.051107&0.057183\\\hline
	0.2&-0.13737&-0.011584&-0.10898&-0.013299&0.064139&0.064654&0.049546&0.053185\\\hline
	0.4&-0.17875&0.012282&-0.10912&-0.013507&0.07568&0.098525&0.047819&0.047615\\\hline
	0.6&-0.21053&0.020088&-0.099335&-0.016966&0.083276&0.14226&0.043431&0.042186\\\hline
    \bottomrule
  \end{tabular}
  \caption{Nonlinear Cost: bias and MSE in estimating the optimal policy parameter at sample size 200. Bias=$\mathbb{E}(\hat{\theta}_T)-\theta^*$.}
  \label{nonlin_bm200}
\end{table}

\begin{table}[H]
  \begin{tabular}{l | *{4}{S[table-auto-round,
                 table-format=-1.3]}| *{4}{S[table-auto-round,
                 table-format=-1.3]}}
    \toprule
    \multirow{2}{*}{{$\alpha$}} &
      \multicolumn{4}{c}{Bias} &
      \multicolumn{4}{c}{MSE} \\ \cline{2-9}
      & {$\theta_0$} & {$\theta_1$} & {$\theta_2$} & {$\theta_3$} & {$\theta_0$} & {$\theta_1$} & {$\theta_2$} & {$\theta_3$} \\
      \midrule
0&-0.038169&-0.035733&-0.044717&-0.002232&0.024531&0.021786&0.02197&0.026732\\\hline
0.2&-0.080181&0.029508&-0.067326&-0.007644&0.026449&0.032378&0.023492&0.023355\\\hline
0.4&-0.10881&0.064703&-0.06932&-0.009877&0.030598&0.064228&0.023013&0.020424\\\hline
0.6&-0.13645&0.057955&-0.067828&-0.009448&0.037806&0.11235&0.021244&0.0184\\\hline
    \bottomrule
  \end{tabular}
  \caption{Nonlinear Cost: bias and MSE in estimating the optimal policy parameter at sample size 500. Bias=$\mathbb{E}(\hat{\theta}_T)-\theta^*$.}
  \label{nonlin_bm500}
\end{table}

\begin{table}[H]
\centering
\begin{tabular} {c| cccc}
\toprule
$\alpha$ & $\theta_0$ & $\theta_1$ & $\theta_2$ & $\theta_3$ \\ \hline
0&0.944&0.947&0.954&0.939\\\hline
0.2&0.926*&0.879*&0.942&0.935*\\\hline
0.4&0.892*&0.738*&0.922*&0.942\\\hline
0.6&0.835*&0.588*&0.914*&0.942\\\hline
\end{tabular}
\caption{Nonlinear Cost: coverage rates of percentile-t bootstrap confidence intervals for the optimal policy parameter at sample size 200. $\lambda$ is estimated online. Coverage rates significantly lower than $0.95$ are marked with asterisks (*).}
\label{nonlin_c200}
\end{table}

\begin{table}[H]
\centering
\begin{tabular} {c| cccc}
\toprule
$\alpha$ & $\theta_0$ & $\theta_1$ & $\theta_2$ & $\theta_3$ \\ \hline
0&0.971&0.961&0.966&0.958\\\hline
0.2&0.931*&0.875*&0.936&0.956\\\hline
0.4&0.885*&0.655*&0.924*&0.958\\\hline
0.6&0.837*&0.471*&0.915*&0.961\\\hline
\end{tabular}
\caption{Nonlinear Cost: coverage rates of percentile-t bootstrap confidence intervals for the optimal policy parameter at sample size 500. $\lambda$ is estimated online. Coverage rates significantly lower than $0.95$ are marked with asterisks (*).}
\label{nonlin_c500}
\end{table}

\section{Conclusion}\label{discuss}


In this article, we present a general framework to define optimal policies for use in JITAIs that encourages intervention variety. We also gave an online actor-critic algorithm to learn the optimal policy. Although the theoretical properties of the algorithm assume i.i.d.\ contexts, the numerical experiments show robustness of the algorithm to violations of this assumption. In particular, experiments show that performance of the algorithm, in term of bias, MSE and confidence interval coverage, is not affected by auto-correlation among contexts. Experiments also demonstrate some robustness of the algorithm when distribution of the context depends on previous actions. Furthermore, we conjecture that, when actions influence the distribution of context at later decision points, the contextual bandit algorithm converges to the myopic equilibrium policy. Our numerical experiments back up this conjecture. Theoretical proof of the conjecture, however, is an open question and requires future work. 

There are a few areas for which the actor-critic algorithm could be improved and extended. First, the linear expected reward assumption might be a bit strong in some scenarios, especially when a low dimension reward feature is used. When the assumption is deemed untenable, more sophisticated components should be added to the reward (cost) feature. To this end, both the actor-critic algorithm and the asymptotic theory should be extended to encompass the scenario where the dimension of the reward (cost) feature grows with the sample size. If one intends to use linear reward (cost) model with a fixed dimension of reward feature, we highly recommend frequent validation of the linear model using model diagnostic tools. Linear regression diagnostic tools can be used as the first line of defense. However, more sophisticated model checking methods for online learning need to be developed to make sure the reward (cost) model is adequate.  
Second, there is room for improvement in optimization in the actor step. Optimizing the estimated regularized average reward function is in general a non-convex optimization problem and could be time-consuming. In the proposed algorithm, optimization at decision point $t+1$ does not use the estimated policy parameters at previous decision points. In other words, the optimization is not incremental and may waste computing resources when the sample size gets large. Careful design of online optimization methods that leverages previous estimates will likely significantly improve the computational efficiency of the actor-critic algorithm and help in its practical adoption in mobile health applications. Third, the algorithm presented in this article learns a user's the optimal policy based \emph{solely} on his/her history. However, in order to speed up the learning it is attractive idea, especially in the beginning of the learning period, to pool data across multiple users. Methods and theories for learning based on multiple users need to be developed (e.g., see the work of~\cite{tomkins2021intelligentpooling}). 


\bibliographystyle{plainnat}

\bibliography{ref}

\begin{thebibliography}{39}
\providecommand{\natexlab}[1]{#1}
\providecommand{\url}[1]{\texttt{#1}}
\expandafter\ifx\csname urlstyle\endcsname\relax
  \providecommand{\doi}[1]{doi: #1}\else
  \providecommand{\doi}{doi: \begingroup \urlstyle{rm}\Url}\fi

\bibitem[Agrawal and Goyal(2012)]{agrawal2012thompson}
Shipra Agrawal and Navin Goyal.
\newblock Thompson sampling for contextual bandits with linear payoffs.
\newblock \emph{arXiv preprint arXiv:1209.3352}, 2012.

\bibitem[Agrawal and Goyal(2013)]{agrawal2013thompson}
Shipra Agrawal and Navin Goyal.
\newblock Thompson sampling for contextual bandits with linear payoffs.
\newblock In \emph{ICML (3)}, pages 127--135, 2013.

\bibitem[Audibert et~al.(2009)Audibert, Munos, and
  Szepesv{\'a}ri]{audibert2009exploration}
Jean-Yves Audibert, R{\'e}mi Munos, and Csaba Szepesv{\'a}ri.
\newblock Exploration--exploitation tradeoff using variance estimates in
  multi-armed bandits.
\newblock \emph{Theoretical Computer Science}, 410\penalty0 (19):\penalty0
  1876--1902, 2009.

\bibitem[Bauer et~al.(2010)Bauer, de~Niet, Timman, and
  Kordy]{bauer2010enhancement}
Stephanie Bauer, Judith de~Niet, Reinier Timman, and Hans Kordy.
\newblock Enhancement of care through self-monitoring and tailored feedback via
  text messaging and their use in the treatment of childhood overweight.
\newblock \emph{Patient education and counseling}, 79\penalty0 (3):\penalty0
  315--319, 2010.

\bibitem[Bertsekas(1999)]{bertsekas1999nonlinear}
Dimitri~P Bertsekas.
\newblock \emph{Nonlinear programming}.
\newblock Athena scientific, 1999.

\bibitem[Bhatnagar and Lakshmanan(2012)]{bhatnagar2012online}
Shalabh Bhatnagar and K~Lakshmanan.
\newblock An online actor--critic algorithm with function approximation for
  constrained markov decision processes.
\newblock \emph{Journal of Optimization Theory and Applications}, 153\penalty0
  (3):\penalty0 688--708, 2012.

\bibitem[Bhatnagar et~al.(2009)Bhatnagar, Sutton, Ghavamzadeh, and
  Lee]{bhatnagar2009natural}
Shalabh Bhatnagar, Richard~S Sutton, Mohammad Ghavamzadeh, and Mark Lee.
\newblock Natural actor--critic algorithms.
\newblock \emph{Automatica}, 45\penalty0 (11):\penalty0 2471--2482, 2009.

\bibitem[Billingsley(1961)]{billingsley1961lindeberg}
Patrick Billingsley.
\newblock The lindeberg-levy theorem for martingales.
\newblock \emph{Proceedings of the American Mathematical Society}, 12\penalty0
  (5):\penalty0 788--792, 1961.

\bibitem[Borkar(2005)]{borkar2005actor}
Vivek~S Borkar.
\newblock An actor-critic algorithm for constrained markov decision processes.
\newblock \emph{Systems \& control letters}, 54\penalty0 (3):\penalty0
  207--213, 2005.

\bibitem[Campi and Garatti(2011)]{campi2011sampling}
Marco~C Campi and Simone Garatti.
\newblock A sampling-and-discarding approach to chance-constrained
  optimization: feasibility and optimality.
\newblock \emph{Journal of Optimization Theory and Applications}, 148\penalty0
  (2):\penalty0 257--280, 2011.

\bibitem[Carpenter et~al.(2020)Carpenter, Menictas, Nahum-Shani, Wetter, and
  Murphy]{carpenter2020developments}
Stephanie~M Carpenter, Marianne Menictas, Inbal Nahum-Shani, David~W Wetter,
  and Susan~A Murphy.
\newblock Developments in mobile health just-in-time adaptive interventions for
  addiction science.
\newblock \emph{Current Addiction Reports}, pages 1--11, 2020.

\bibitem[Chu et~al.(2011)Chu, Li, Reyzin, and Schapire]{chu2011contextual}
Wei Chu, Lihong Li, Lev Reyzin, and Robert~E Schapire.
\newblock Contextual bandits with linear payoff functions.
\newblock In \emph{International Conference on Artificial Intelligence and
  Statistics}, pages 208--214, 2011.

\bibitem[Consolvo et~al.(2008)Consolvo, McDonald, Toscos, Chen, Froehlich,
  Harrison, Klasnja, LaMarca, LeGrand, Libby, et~al.]{consolvo2008activity}
Sunny Consolvo, David~W McDonald, Tammy Toscos, Mike~Y Chen, Jon Froehlich,
  Beverly Harrison, Predrag Klasnja, Anthony LaMarca, Louis LeGrand, Ryan
  Libby, et~al.
\newblock Activity sensing in the wild: a field trial of ubifit garden.
\newblock In \emph{Proceedings of the SIGCHI Conference on Human Factors in
  Computing Systems}, pages 1797--1806. ACM, 2008.

\bibitem[Dempsey et~al.(2015)Dempsey, Liao, Klasnja, Nahum-Shani, and
  Murphy]{dempsey2015randomised}
Walter Dempsey, Peng Liao, Pedja Klasnja, Inbal Nahum-Shani, and Susan~A
  Murphy.
\newblock Randomised trials for the fitbit generation.
\newblock \emph{Significance}, 12\penalty0 (6):\penalty0 20--23, 2015.

\bibitem[Fiacco and Ishizuka(1990)]{fiacco1990sensitivity}
Anthony~V Fiacco and Yo~Ishizuka.
\newblock Sensitivity and stability analysis for nonlinear programming.
\newblock \emph{Annals of Operations Research}, 27\penalty0 (1):\penalty0
  215--235, 1990.

\bibitem[Gustafson et~al.(2011)Gustafson, Shaw, Isham, Baker, Boyle, and
  Levy]{gustafson2011explicating}
David~H Gustafson, Bret~R Shaw, Andrew Isham, Timothy Baker, Michael~G Boyle,
  and Michael Levy.
\newblock Explicating an evidence-based, theoretically informed, mobile
  technology-based system to improve outcomes for people in recovery for
  alcohol dependence.
\newblock \emph{Substance use \& misuse}, 46\penalty0 (1):\penalty0 96--111,
  2011.

\bibitem[King et~al.(2013)King, Hekler, Grieco, Winter, Sheats, Buman,
  Banerjee, Robinson, and Cirimele]{king2013harnessing}
Abby~C King, Eric~B Hekler, Lauren~A Grieco, Sandra~J Winter, Jylana~L Sheats,
  Matthew~P Buman, Banny Banerjee, Thomas~N Robinson, and Jesse Cirimele.
\newblock Harnessing different motivational frames via mobile phones to promote
  daily physical activity and reduce sedentary behavior in aging adults.
\newblock \emph{PloS one}, 8\penalty0 (4):\penalty0 e62613, 2013.

\bibitem[Klasnja et~al.(2015)Klasnja, Hekler, Shiffman, Boruvka, Almirall,
  Tewari, and Murphy]{klasnja_microrandomized_2015}
Predrag Klasnja, Eric~B. Hekler, Saul Shiffman, Audrey Boruvka, Daniel
  Almirall, Ambuj Tewari, and Susan~A. Murphy.
\newblock Microrandomized trials: {An} experimental design for developing
  just-in-time adaptive interventions.
\newblock \emph{Health Psychology}, 34\penalty0 (Suppl):\penalty0 1220--1228,
  2015.
\newblock ISSN 1930-7810, 0278-6133.
\newblock \doi{10.1037/hea0000305}.
\newblock URL \url{http://doi.apa.org/getdoi.cfm?doi=10.1037/hea0000305}.

\bibitem[Konda and Tsitsiklis(1999)]{konda1999actor}
Vijay~R Konda and John~N Tsitsiklis.
\newblock Actor-critic algorithms.
\newblock In \emph{NIPS}, volume~13, pages 1008--1014, 1999.

\bibitem[Langford and Zhang(2008)]{langford2008epoch}
John Langford and Tong Zhang.
\newblock The epoch-greedy algorithm for multi-armed bandits with side
  information.
\newblock In \emph{Advances in neural information processing systems}, pages
  817--824, 2008.

\bibitem[Li et~al.(2010)Li, Chu, Langford, and Schapire]{li2010contextual}
Lihong Li, Wei Chu, John Langford, and Robert~E Schapire.
\newblock A contextual-bandit approach to personalized news article
  recommendation.
\newblock In \emph{Proceedings of the 19th international conference on World
  wide web}, pages 661--670. ACM, 2010.

\bibitem[M{\"u}ller et~al.(2017)M{\"u}ller, Blandford, and
  Yardley]{muller2017conceptualization}
Andre~Matthias M{\"u}ller, Ann Blandford, and Lucy Yardley.
\newblock The conceptualization of a just-in-time adaptive intervention (jitai)
  for the reduction of sedentary behavior in older adults.
\newblock \emph{Mhealth}, 3, 2017.

\bibitem[Nahum-Shani et~al.(2018)Nahum-Shani, Smith, Spring, Collins,
  Witkiewitz, Tewari, and Murphy]{nahum2018just}
Inbal Nahum-Shani, Shawna~N Smith, Bonnie~J Spring, Linda~M Collins, Katie
  Witkiewitz, Ambuj Tewari, and Susan~A Murphy.
\newblock Just-in-time adaptive interventions (jitais) in mobile health: key
  components and design principles for ongoing health behavior support.
\newblock \emph{Annals of Behavioral Medicine}, 52\penalty0 (6):\penalty0
  446--462, 2018.

\bibitem[Nemirovski and Shapiro(2006)]{nemirovski2006convex}
Arkadi Nemirovski and Alexander Shapiro.
\newblock Convex approximations of chance constrained programs.
\newblock \emph{SIAM Journal on Optimization}, 17\penalty0 (4):\penalty0
  969--996, 2006.

\bibitem[Patrick et~al.(2009)Patrick, Raab, Adams, Dillon, Zabinski, Rock,
  Griswold, and Norman]{patrick2009text}
Kevin Patrick, Fred Raab, Marc Adams, Lindsay Dillon, Marion Zabinski, Cheryl
  Rock, William Griswold, and Gregory Norman.
\newblock A text message-based intervention for weight loss: randomized
  controlled trial.
\newblock \emph{Journal of medical Internet research}, 11\penalty0
  (1):\penalty0 e1, 2009.

\bibitem[Pr{\'e}kopa(1995)]{prekopa2013stochastic}
Andr{\'a}s Pr{\'e}kopa.
\newblock \emph{Stochastic programming}.
\newblock Springer Science \& Business Media, 1995.

\bibitem[Richardson et~al.(2020)Richardson, Harrison, Heathcote, Rush, Shear,
  Lalloo, Hood, Wicksell, Stinson, and Simons]{richardson2020mhealth}
Patricia~A Richardson, Lauren~E Harrison, Lauren~C Heathcote, Gillian Rush,
  Deborah Shear, Chitra Lalloo, Korey Hood, Rikard~K Wicksell, Jennifer
  Stinson, and Laura~E Simons.
\newblock mhealth for pediatric chronic pain: state of the art and future
  directions.
\newblock \emph{Expert Review of Neurotherapeutics}, 20\penalty0 (11):\penalty0
  1177--1187, 2020.

\bibitem[Riley et~al.(2011)Riley, Rivera, Atienza, Nilsen, Allison, and
  Mermelstein]{riley2011health}
William~T Riley, Daniel~E Rivera, Audie~A Atienza, Wendy Nilsen, Susannah~M
  Allison, and Robin Mermelstein.
\newblock Health behavior models in the age of mobile interventions: are our
  theories up to the task?
\newblock \emph{Translational behavioral medicine}, 1\penalty0 (1):\penalty0
  53--71, 2011.

\bibitem[Scott and Dennis(2009)]{scott2009results}
Christy~K Scott and Michael~L Dennis.
\newblock Results from two randomized clinical trials evaluating the impact of
  quarterly recovery management checkups with adult chronic substance users.
\newblock \emph{Addiction}, 104\penalty0 (6):\penalty0 959--971, 2009.

\bibitem[Suffoletto et~al.(2012)Suffoletto, Callaway, Kristan, Kraemer, and
  Clark]{suffoletto2012text}
Brian Suffoletto, Clifton Callaway, Jeff Kristan, Kevin Kraemer, and Duncan~B
  Clark.
\newblock Text-message-based drinking assessments and brief interventions for
  young adults discharged from the emergency department.
\newblock \emph{Alcoholism: Clinical and Experimental Research}, 36\penalty0
  (3):\penalty0 552--560, 2012.

\bibitem[Tewari and Murphy(2017)]{tewari2017ads}
Ambuj Tewari and Susan~A Murphy.
\newblock From ads to interventions: Contextual bandits in mobile health.
\newblock In \emph{Mobile Health}, pages 495--517. Springer, 2017.

\bibitem[Thomas and Bond(2015)]{thomas2015behavioral}
J~Graham Thomas and Dale~S Bond.
\newblock Behavioral response to a just-in-time adaptive intervention (jitai)
  to reduce sedentary behavior in obese adults: Implications for jitai
  optimization.
\newblock \emph{Health Psychology}, 34\penalty0 (S):\penalty0 1261, 2015.

\bibitem[Tomkins et~al.(2021)Tomkins, Liao, Klasnja, and
  Murphy]{tomkins2021intelligentpooling}
Sabina Tomkins, Peng Liao, Predrag Klasnja, and Susan Murphy.
\newblock Intelligentpooling: Practical thompson sampling for mhealth.
\newblock \emph{Machine learning}, 110\penalty0 (9):\penalty0 2685--2727, 2021.

\bibitem[Tropp(2012)]{tropp2012user}
Joel~A Tropp.
\newblock User-friendly tail bounds for sums of random matrices.
\newblock \emph{Foundations of computational mathematics}, 12\penalty0
  (4):\penalty0 389--434, 2012.

\bibitem[Vamvoudakis and Lewis(2010)]{vamvoudakis2010online}
Kyriakos~G Vamvoudakis and Frank~L Lewis.
\newblock Online actor--critic algorithm to solve the continuous-time infinite
  horizon optimal control problem.
\newblock \emph{Automatica}, 46\penalty0 (5):\penalty0 878--888, 2010.

\bibitem[Van~der Vaart(2000)]{van2000asymptotic}
Aad~W Van~der Vaart.
\newblock \emph{Asymptotic statistics}, volume~3.
\newblock Cambridge university press, 2000.

\bibitem[Witkiewitz et~al.(2014)Witkiewitz, Desai, Bowen, Leigh, Kirouac, and
  Larimer]{witkiewitz2014development}
Katie Witkiewitz, Sruti~A Desai, Sarah Bowen, Barbara~C Leigh, Megan Kirouac,
  and Mary~E Larimer.
\newblock Development and evaluation of a mobile intervention for heavy
  drinking and smoking among college students.
\newblock \emph{Psychology of Addictive Behaviors}, 28\penalty0 (3):\penalty0
  639, 2014.

\bibitem[Woodroofe(1979)]{woodroofe1979one}
Michael Woodroofe.
\newblock A one-armed bandit problem with a concomitant variable.
\newblock \emph{Journal of the American Statistical Association}, 74\penalty0
  (368):\penalty0 799--806, 1979.

\bibitem[Zedek(1965)]{zedek1965continuity}
Mishael Zedek.
\newblock Continuity and location of zeros of linear combinations of
  polynomials.
\newblock \emph{Proceedings of the American Mathematical Society}, 16\penalty0
  (1):\penalty0 78--84, 1965.

\end{thebibliography}

\newpage
\appendix
\pagenumbering{arabic}
\renewcommand{\thepage}{A\arabic{page}}
\section*{Supplementary Material}

\section{Proof of Lemma 1}\label{sup_lemma1}


\begin{proof}
Without the loss of generality, we assume that $0< s_1 < s_2 < ... < s_K$. Otherwise, if some $s_i$'s are negative, we can transform all the contexts to be positive by adding to $s_i$'s a constant greater than $\min_{1\leq i\leq K}s_i$. Denote this constant by $M$ and the corresponding policy parameter by $\tilde{\theta}$. There is a one-to-one correspondence between the two policy classes:

\begin{align*}
\tilde{\theta}_0 & = \theta_0-M\theta_1 \\
\tilde{\theta}_1 & =  \theta_1
\end{align*}
Therefore if the lemma holds when all contexts are positive the same conclusion hold in the general setting. We use $p(\theta)$ to denote the probability the probability of choosing action $A=1$ for policy $\pi_{\theta}$ at the K different values of context:
 \begin{align*}
 ( \frac{e^{\theta_0+\theta_1s_1}}{1+e^{\theta_0+\theta_1s_1}}, \frac{e^{\theta_0+\theta_1s_2}}{1+e^{\theta_0+\theta_1s_2}}, ..., \frac{e^{\theta_0+\theta_1s_K}}{1+e^{\theta_0+\theta_1s_K}})
 \end{align*}
 Notice that each entry in $p(\theta)$ is number between 0 and 1 with equality if the policy is deterministic at certain context. A key step towards proving deterministic optimal policy is to show the following closed convex hull equivalency:
\begin{align*}
conv( \{ p(\theta): \theta \in \mathbb{R}^2 \} )= conv ( \{ (\nu_1, ...,\nu_K), \nu_i \in\{0,1\}, \nu_1\leq ... \leq \nu_K \mbox{ or } \nu_1 \geq ...\geq \nu_K \} )
\end{align*}
We examine the limiting points of $p(\theta)$ when $\theta_0$ and $\theta_1$ tends to infinity. We consider the case where $\theta_0 \neq0$ and let $\theta_1=p\theta_0$ where $p$ is a fixed value. It holds that

\begin{displaymath}
\frac{e^{\theta_0+\theta_1s}}{1+e^{\theta_0+\theta_1s}} = \frac{e^{\theta_0(1+ps)}}{1+e^{\theta_0(1+ps)}} \rightarrow \left \{
\begin{array}{lr}
0: if \theta_0 \rightarrow -\infty, p>-1/s \\
0: if \theta_0 \rightarrow \infty, p<-1/s \\
1: if \theta_0 \rightarrow -\infty, p<-1/s \\
1: if \theta_0 \rightarrow \infty, p>-1/s 
\end{array}
\right.
\end{displaymath}

\noindent It follows that when $\theta_0 \rightarrow -\infty$ and $p$ scans through the $K+1$ intervals on $\mathbb{R}$: $(-\infty, -1/s_1]$, $(-1/s_1, -1/s_2]$, . ... $( -1/s_K, \infty)$, $p(\theta)$ approaches the following $K+1$ limiting points:

\begin{eqnarray*}
(1,1,..., 1) \\
(0,1,..., 1) \\
...\\
(0,0,..., 1) \\
(0,0,..., 0) 
\end{eqnarray*}
when $\theta_0 \rightarrow \infty$ and $p$ scans through the $K+1$ intervals, $p(\theta)$ approaches the following $K+1$ limiting points 

\begin{eqnarray*}
(0,0,..., 0) \\
(1,0,..., 0) \\
...\\
(1,1,..., 0) \\
(1,1,..., 1)
\end{eqnarray*}
There are in total $2K$ limiting points:  $\{ (\nu_1, ...,\nu_K), \nu_i \in\{0,1\}, \nu_1\leq ... \leq \nu_K \mbox{ or } \nu_1 \geq ...\geq \nu_K \}$. Each limiting point is a $K$ dimensional vector with 0-1 entries in an either increasing or decreasing order. Now we show that any $p(\theta), \theta\in \mathbb{R}^2$ is a convex combination of the limiting points. Let $p(\theta)=[p_1(\theta), p_2(\theta), ..., p_K(\theta)]$. In fact, 

\begin{itemize}
\item If $\theta_1=0$, $p(\theta) = (1-p_1(\theta))(0,0,..., 0) + p_1(\theta)(1,1,..., 1)$
\item If $\theta_1>0$, we have $0<p_1(\theta)<p_2 (\theta)<... <p_K(\theta) <1$ and
\begin{align*}
p(\theta)=p_1(\theta)(1,1,..., 1)+(p_2(\theta)-p_1(\theta))(0,1,..., 1)+...\\+(p_K(\theta)-p_{K-1}(\theta))(0,0,..., 1)+(1-p_K(\theta))*(0,0,..., 0)
\end{align*}
\item If $\theta_1<0$, we have $1>p_1(\theta)>p_2(\theta) >... >p_K (\theta)>0$ and
\begin{align*}
p(\theta)=(1-p_1(\theta))*(0,0,..., 0)+(p_1(\theta)-p_2(\theta))(1,0,..., 0)+ ... \\+ (p_K(\theta)-p_{K-1}(\theta))(1,1,..., 0) + p_K(\theta)(1,1,..., 1)
\end{align*}
\end{itemize}


Returning to optimizing the average reward, we denote $\alpha_i=P(S=s_i) (\mathbb{E}(R|S=s_i,A=1)-\mathbb{E}(R|S=s_i,A=0))$. 

\begin{align}
\max_{\theta} V^*(\theta) & =  \max_{\theta}  \sum_{i=1}^K \alpha_i p_i(\theta) \\ 
& = \max_{ (p_1,...,p_K)\in \{ p(\theta): \theta \in \mathbb{R}^2 \} } \sum_{i=1}^K \alpha_i p_i  \label{23}\\
& = \max_{ (p_1,...,p_K)\in conv( \{ p(\theta): \theta \in \mathbb{R}^2 \} ) } \sum_{i=1}^K \alpha_i p_i \label{24} \\
& = \max_{ (p_1,...,p_K)\in conv ( \{ (\nu_1, ...,\nu_K), \nu_i \in\{0,1\}, \nu_1\leq ... \leq \nu_K \mbox{ or } \nu_1 \geq ...\geq \nu_K \} )} \sum_{i=1}^K \alpha_i p_i  \label{25}
\end{align}
. Equation from \eqref{23} to \eqref{24} is followed by the fact that the objective function is linear (and thus convex) in $p_i$'s. Equivalency from \eqref{24} to \eqref{25} is a direct product of the closed convex hull equivalency. Theories in linear programming theory suggests that one of the maximal points is attained at the vertices of the convex hull of the feasible set. Therefore we have proved that one of the policy that maximizes $V^*(\theta)$ is deterministic. 
\end{proof}

\section{One-to-one Correspondence between Constrained and Unconstrained Optimization}\label{sup_1to1}
The constrained optimization finds the policy that maximizes the average reward subject to the quadratic constraint, i.e., 
\begin{align}
\max_{\theta}  V^*(\theta), \mbox{  s. t. } \theta^T \mathbb{E}[g(S)^Tg(S)] \theta \leq (\log(\frac{p_0}{1-p_0}))^2 \alpha
\label{constr_optim}
\end{align}
The unconstrained optimization finds the policy that maximizes the regularized average reward:
\begin{align}
\theta^* =\argmax_{\theta} J^*_{\lambda}(\theta)
\label{unc_optim}
\end{align}
A natural question to ask, when transforming the constrained optimization problem \eqref{constr_optim} to an unconstrained one \eqref{unc_optim}, does a Lagrangian multiplier exist for each level of stringency of the quadratic constraint? While the correspondence between the constrained optimization and the unconstrained one may not seem so obvious due to the lack of convexity in $V^*(\theta)$, we established the following Lemma \ref{lemma_unique} given Assumption \ref{singleton} and Assumption \ref{assumption_policy_feature}. Assumption \ref{singleton} assumes the uniqueness of the global maximum for all positive $\lambda$. 
\begin{assumption}
For every $0<\lambda<\infty$, the global maximum of the regularized average reward is a singleton. 
\begin{align*}
J^*_{\lambda}(\theta)=\sum_{s\in \mathcal{S}}d(s)\sum_{a\in\mathcal{A}}E(R|S=s,A=a) \pi_{\theta}(s,a)-\lambda \theta^T \mathbb{E}[g(S)g(S)^T] \theta.
\end{align*}
\label{singleton}
\end{assumption}


\begin{lemma}
If the maximizer of the average reward function $V^*(\theta)$ is deterministic, i.e. $P(\pi_{\theta}(A=1|S)=1)>0$ or $P(\pi_{\theta}(A=0|S)=1)>0$, under Assumption \ref{assumption_policy_feature} and \ref{singleton}, for every $K= (\log(\frac{p_0}{1-p_0}))^2 \alpha> 0$ there  exist a $\lambda>0$ such that the solution of the constrained optimization problem \ref{constr_optim} is the solution of the unconstrained optimization problem \ref{unc_optim}. 
\label{lemma_unique}
\end{lemma}

\begin{proof}
Let $\theta^*_{\lambda}$ be one of the global maxima of the Lagrangian function: $\theta^*_{\lambda}= \argmax_{\theta} J^*_{\lambda}(\theta)$. Let $\beta_{\lambda}=\theta^{*T}_{\lambda} \mathbb{E}[g(S)^Tg(S)] \theta^*_{\lambda}$. By Proposition 3.3.4 in \cite{bertsekas1999nonlinear}, $\theta^*_{\lambda}$ is a global maximum of constrained problem:
\begin{align*}
&\max_{\theta} V^*(\theta) \\
& \mbox{ s.t.    }\theta^T \mathbb{E}[g(S)^Tg(S)] \theta \leq \beta_{\lambda}
\end{align*}
In addition, the stringency of the quadratic constraint increases monotonically with the value of the Lagrangian coefficient $\lambda$. Let $0<\lambda_1<\lambda_2$ and with some abuse of notation, let $\theta_1$ and $\theta_2$ be (one of) the global maximals of Lagrangian function $J^*_{\lambda_1}(\theta)$ and $J^*_{\lambda_2}(\theta)$. It follows that 

\begin{align*}
& -V^*(\theta_2)+\lambda_2 \theta_2^T \mathbb{E}[g(S)^Tg(S)] \theta_2 \\
\leq & -V^*(\theta_1)+\lambda_2 \theta_1^T \mathbb{E}[g(S)^Tg(S)] \theta_1 \\
= &  -V^*(\theta_1)+\lambda_1 \theta_1^T \mathbb{E}[g(S)^Tg(S)] \theta_1 + (\lambda_2-\lambda_1)  \theta_1^T \mathbb{E}[g(S)^Tg(S)] \theta_1 \\
\leq & -V^*(\theta_2)+\lambda_1 \theta_2^T \mathbb{E}[g(S)^Tg(S)] \theta_2 + (\lambda_2-\lambda_1)  \theta_1^T \mathbb{E}[g(S)^Tg(S)] \theta_1
\end{align*}
It follows that
\begin{align*}
\theta_1^T \mathbb{E}[g(S)^Tg(S)] \theta_1 \geq  \theta_2^T \mathbb{E}[g(S)^Tg(S)] \theta_2 .
\end{align*}
As $\lambda$ approaches 0, the maximal of the regularized average reward approaches the maximal of the average reward function, for which $\mathbb{E}(\theta^Tg(S))^2 \rightarrow \infty$. As $\lambda$ increases towards $\infty$, maximal of the regularized average reward approaches the random policy with $\theta=0$. It's only left to show that $\theta^{*T}_{\lambda}\mathbb{E}[g(S)^Tg(S)] \theta^{*}_{\lambda}$ is a continuous function of $\lambda$. Under Assumption \ref{singleton}, we can verify that conditions in Theorem 2.2 in \cite{fiacco1990sensitivity} holds. This theorem implies that the solution set of the unconstrained optimization \ref{unc_optim} is continuous in $\lambda$, sufficient to conclude the continuity of $\theta^{*T}_{\lambda}\mathbb{E}[g(S)^Tg(S)] \theta^{*}_{\lambda}$.
\end{proof}

\section{Proof of Theorem~\ref{thm:J_convergence}}
\begin{lemma}[Boundedness of $\theta^*$ and $\hat{\theta}_t$ for large $t$]
	For fixed regularization parameter $\lambda$, $\left\|\theta^*\right\|_2^2\leq\frac{1}{\lambda\lambda_p}$, $\left\|\hat{\theta}_t\right\|_2^2\leq \frac{4}{\lambda\lambda_p}$ with prob at least $1-p\left(2e\right)^{-\frac{t\lambda_p}{2}}$.
	\label{lemma:bdd}
\end{lemma}
\begin{proof}
	By definition, $J(\theta^*,\mu^*) \geq J(0,\mu^*)$, we have
	\begin{align*}
	0 &\leq \lambda \theta^{*T} \mathbb{E}\left[g(S)g(S)^T\right] \theta^* \leq \int_{s\in \mathcal{S}}d(s)\sum_{a\in \mathcal{A}} f(s,a)^T\mu^* \left(\pi_{\theta^*}\left(s,a\right)-1/2\right)ds \leq 1\\
	\Rightarrow & \lambda||\theta^*||_2^2 \lambda_p \leq \lambda \theta^{*T} \mathbb{E}\left[g(S)g(S)^T\right] \theta^* \leq 1\\
	\Rightarrow & \left\|\theta^*\right\|_2^2 \leq \frac{1}{\lambda\lambda_p}.
	\end{align*}
	For the boundedness of $\left\|\hat{\theta}_t\right\|_2^2$, we use the matrix Chernoff inequality. For $\forall 0<\delta<1$, we have
	\begin{align*}
	P\left(\lambda_{min}\left(\frac{1}{t}\sum_{\tau=1}^tg\left(S_\tau\right)g\left(S_\tau\right)^T\right) \geq \left(1-\delta\right)\lambda_p\right) \geq 1-p\left[\frac{e^{-\delta}}{\left(1-\delta\right)^{1-\delta}}\right]^{t\lambda_p}.
	\end{align*}
	Take $\delta=\frac{1}{2}$, it becomes
	\begin{align*}
	P\left(\lambda_{min}\left(\frac{1}{t}\sum_{\tau=1}^tg\left(S_\tau\right)g\left(S_\tau\right)^T\right) \geq \frac{1}{2}\lambda_p\right) \geq 1-p\left(2e\right)^{-\frac{t\lambda_p}{2}}.
	\end{align*}
	By definition, $\hat{J}_t\left(\hat{\theta}_t,\hat{\mu}_t\right)\geq \hat{J}_t\left(0,\hat{\mu}_t\right)$, so
	\begin{align*}
	0 &\leq \lambda \hat{\theta}_t^T \left(\frac{1}{t}\sum_{\tau=1}^tg\left(S_\tau\right)g\left(S_\tau\right)^T\right) \hat{\theta}_t \leq \frac{1}{t}\sum_{\tau=1}^t\sum_{a\in \mathcal{A}} r_{\hat{\mu}_t}(S_\tau,a)\left(\pi_\theta(S_\tau,a)-\frac{1}{2}\right)\leq 2
	\end{align*}
	Thus,
	\begin{align*}
	P\left(\left\|\hat{\theta}_t\right\|_2^2 \leq \frac{4}{\lambda\lambda_p}\right) \geq 1-p\left(2e\right)^{-\frac{t\lambda_p}{2}}
	\end{align*}
\end{proof}

\begin{proof}[Proof of Theorem~\ref{thm:J_convergence}]
	\begin{align}
	0\leq J(\theta^*,\mu^*)-J(\hat{\theta}_t,\mu^*) & \leq
	\left|J(\theta^*,\mu^*) - \tilde{J}_t(\theta^*,\mu^*)\right| \label{equ:hoeff}\\
	&+ \left|\tilde{J}_t(\theta^*,\mu^*)-\hat{J}_t(\theta^*,\mu^*)\right|\label{equ:J_approx_fix}\\
	&+ \left|\hat{J}_t(\theta^*,\mu^*)-\hat{J}_t(\theta^*,\hat{\mu}_t)\right| \label{equ:J_mu_fix}\\
	&+ \hat{J}_t(\theta^*,\hat{\mu}_t) - \hat{J}_t(\hat{\theta}_t,\hat{\mu}_t) \label{equ:easy}\\
	&+ \left|\hat{J}_t(\hat{\theta}_t,\hat{\mu}_t)-\hat{J}_t(\hat{\theta}_t,\mu^*)\right|\label{equ:J_mu_random}\\
	&+ \left|\hat{J}_t(\hat{\theta}_t,\mu^*) - \tilde{J}_t(\hat{\theta}_t,\mu^*)\right| \label{equ:J_approx_random}\\
	&+ \left|\tilde{J}_t(\hat{\theta}_t,\mu^*) - J(\hat{\theta}_t,\mu^*)\right| \label{equ:hoeff_random}
	\end{align}
	
	Next, we show each part of above decomposition converges to zero with high probability.\\
	
	\noindent
	\textbf{Equation~\ref{equ:hoeff}:} By Lemma~\ref{lemma:bdd}, we have
	\begin{align*}
	\left|X_\tau\right|\triangleq\left|\sum_{a\in \mathcal{A}}f(S_\tau,a)^T\mu^*\pi_{\theta^*}(S_\tau,a) - \lambda \theta^{*T}g(S_\tau)g(S_\tau)^T\theta^*\right| \leq 2+\lambda\left\|\theta^*\right\|_2^2 \leq 2+\frac{1}{\lambda_p}.
	\end{align*}
	Then by Hoeffding's inequality, using $-2-\frac{1}{\lambda_p}\leq X_\tau\leq 2+\frac{1}{\lambda_p}$, we have
	\begin{align*}
	P\left(\left|J(\theta^*,\mu^*) - \tilde{J}_t(\theta^*,\mu^*)\right| > \epsilon\right) \leq 2\exp\left\{-\frac{t\epsilon^2}{2\left(2+\frac{1}{\lambda_p}\right)^2}\right\}.
	\end{align*}
	
	\noindent
	\textbf{Equation~\ref{equ:J_approx_fix} and Equation~\ref{equ:J_approx_random}:} By definition of functions $\hat{J}_t$ and $\tilde{J}_t$, we have
	\begin{align*}
	\left|\tilde{J}_t(\theta^*,\mu^*)-\hat{J}_t(\theta^*,\mu^*)\right| = 0\\
	\left|\hat{J}_t(\hat{\theta}_t,\mu^*) - \tilde{J}_t(\hat{\theta}_t,\mu^*)\right| = 0.
	\end{align*}
	Because $|f(S,a)^T\mu^*|\leq 2$ always holds for any $(S,a)$.
	\\
	\\
	\noindent
	\textbf{Equation~\ref{equ:easy}:} By $\hat{\theta}_t\triangleq \argmax_{\theta}\hat{J}_t(\theta,\hat{\mu}_t)$, we have
	\begin{align*}
	\hat{J}_t(\theta^*,\hat{\mu}_t) - \hat{J}_t(\hat{\theta}_t,\hat{\mu}_t)\leq 0.
	\end{align*}
	
	\noindent
	\textbf{Equation~\ref{equ:hoeff_random}:}
	Define set $S_\lambda = \{\theta: ||\theta||_2 \leq \sqrt{\frac{4}{\lambda\lambda_p}}:=M_\lambda \}$. Then there exists an $\epsilon_0$-net $\mathcal{N_\lambda}$ of $S_\lambda$. It is well known that the covering number can be bounded by $\left|N_\lambda\right| \leq (\frac{3M_\lambda}{\epsilon_0})^p$. Thus, by union bound we have,
	\begin{align*}
	P\left(\sup_{\theta \in \mathcal{N_\lambda}} \left|\tilde{J}_t(\theta,\mu^*) - J(\theta,\mu^*)\right|>\epsilon\right) \leq 2\left(\frac{3M_\lambda}{\epsilon_0}\right)^p \exp\left\{-\frac{t\epsilon^2}{2\left(2+\frac{1}{\lambda_p}\right)^2}\right\}
	\end{align*}
	Consider $\forall \theta \in S_\lambda$, one can show $J\left(\theta,\mu\right)$ and $\tilde{J}_t\left(\theta,\mu\right)$ are $\left(2+2\lambda M_\lambda^2\right)$-Lipschitz functions in $\theta$, thus we have below decompositions:
	\begin{align*}
	\left|\tilde{J}_t(\theta,\mu^*) - J(\theta,\mu^*)\right| \leq & \left|\tilde{J}_t(\theta,\mu^*) - \tilde{J}_t(\hat{\theta},\mu^*)\right|+\left|\tilde{J}_t(\hat{\theta},\mu^*)-J(\hat{\theta},\mu^*)\right|+\left|J(\hat{\theta},\mu^*)-J(\theta,\mu^*)\right|\\
	\leq & \left(4+4\lambda M_\lambda^2\right) \epsilon_0+\sup_{\theta \in \mathcal{N_\lambda}} \left|\tilde{J}_t(\theta,\mu^*) - J(\theta,\mu^*)\right|.
	\end{align*}
	where $\hat{\theta}$ is the closest point to $\theta$ in the $\epsilon_0$-net $\mathcal{N_\lambda}$. Set $\epsilon_0 = \frac{\epsilon}{\left(8+8\lambda M_\lambda^2\right)}$, we can bound equation~\ref{equ:hoeff_random} by
	\begin{align*}
	& P\left(\left|\tilde{J}_t(\hat{\theta}_t,\mu^*) - J(\hat{\theta}_t,\mu^*)\right| \geq \epsilon\right)\\
	\leq & P\left(\left\{\left|\tilde{J}_t(\hat{\theta}_t,\mu^*) - J(\hat{\theta}_t,\mu^*)\right| \geq \epsilon \right\} \cap \left\{\hat{\theta}_t \in S_\lambda \right\} \right) + P\left(\hat{\theta}_t \notin S_\lambda\right)\\
	\leq & P\left(\sup_{\theta \in S_\lambda} \left|\tilde{J}_t(\theta,\mu^*) - J(\theta,\mu^*)\right| \geq \epsilon\right) + p\left(2e\right)^{-\frac{t\lambda_p}{2}} \\
	\leq & P\left(\left(4+4\lambda M_\lambda^2\right) \epsilon_0+\sup_{\theta \in \mathcal{N_\lambda}} \left|\tilde{J}_t(\theta,\mu^*) - J(\theta,\mu^*)\right|>\epsilon\right) + p\left(2e\right)^{-\frac{t\lambda_p}{2}}\\
	\leq & P\left(\sup_{\theta \in \mathcal{N_\lambda}} \left|\tilde{J}_t(\theta,\mu^*) - J(\theta,\mu^*)\right|>\epsilon/2\right) + p\left(2e\right)^{-\frac{t\lambda_p}{2}}\\
	\leq & 2\left(\frac{\left(24+24\lambda M_\lambda^2\right)M_\lambda}{\epsilon}\right)^p \exp\left\{-\frac{t\epsilon^2}{8\left(2+\frac{1}{\lambda_p}\right)^2}\right\}+ p\left(2e\right)^{-\frac{t\lambda_p}{2}}.
	\end{align*}
	
	\noindent
	\textbf{Equation~\ref{equ:J_mu_fix}:} By eigenvalue decomposition, we have $\mathbb{E}\left[f(S,a_i)f(S,a_i)^T\right] = U_i\Sigma_i U_i^T$, where the dimensions are $U_i$: $k$ by $r_i$, $\Sigma_i: r_i$ by $r_i$ and $rank(\Sigma_i)=r_i$.
	Thus $P\left(f(S_\tau,a_i) \in col(U_i)\right) = 1$. By this property, we thus have 
	\begin{align}
	P\left(U_iU_i^Tf(S_\tau,a_i) = f(S_\tau,a_i)\right)=1 \ \ (i=0,1).\label{equ:proj}
	\end{align}

	By Lemma~\ref{lemma:bdd}, we show
	\begin{align*}
	P\left(\tilde{p}_0\leq \pi_{\hat{\theta}_{t}}\left(S,a_1\right) = \frac{\exp\left(g(S)^T\hat{\theta}_t\right)}{1+\exp\left(g(S)^T\hat{\theta}_t\right)} \leq 1-\tilde{p}_0\right) \geq 1-p\left(2e\right)^{-\frac{t\lambda_p}{2}},
	\end{align*}
	where $\tilde{p}_0 := \frac{1}{1+\exp\left(\sqrt{\frac{4}{\lambda\lambda_p}})\right)}$.
	
	Equation~\ref{equ:J_mu_fix} can be written as:
	\begin{align*}
	\left|\hat{J}_t(\theta^*,\mu^*)-\hat{J}_t(\theta^*,\hat{\mu}_t)\right| &= \left|\frac{1}{t}\sum_{\tau=1}^t\sum_{a\in \mathcal{A}} \left(r_{\mu^*}(S_\tau,a)-r_{\hat{\mu}_t}(S_\tau,a)\right)\pi_\theta(S_\tau,a)\right|\\
	&= \left|\frac{1}{t}\sum_{\tau=1}^t\sum_{a\in \mathcal{A}} \left(f(S_\tau,a)^T\mu^*-r_{\hat{\mu}_t}(S_\tau,a)\right)\pi_\theta(S_\tau,a)\right|.
	\end{align*}
	
	Since inequality $\left|r_{\hat{\mu}_t}(S_\tau,a_0)-f(S_\tau,a)^T\mu^*\right| \leq \left|f(S_\tau,a)^T\hat{\mu}_t-f(S_\tau,a)^T\mu^*\right|$ always holds, we only need to bound the larger term $\left|f(S_\tau,a)^T\hat{\mu}_t-f(S_\tau,a)^T\mu\right|$.
	
	\begin{align*}
	&P\left(\left|f(S_\tau,a_0)^T(\mu^*-\hat{\mu}_t)\right|>\epsilon\right)\\
	=& P\left(\left|f(S_\tau,a_0)^TU_0U_0^T(\mu^*-\hat{\mu}_t)\right|>\epsilon\right)\\
	\leq  &P\left(\left\|U_0U_0^T(\mu^*-\hat{\mu}_t)\right\|_2>\epsilon\right) \\
	= &P\left(\left\|U_0U_0^T\left(\frac{1}{t}\zeta I_k+\frac{1}{t}\sum_{i=1}^t f(S_i,A_i)f(S_i,A_i)^T\right)^{-1}\left(\frac{1}{t}\sum_{i=1}^t f(S_i,A_i)\epsilon_i-\frac{\zeta}{t}\mu^*\right)\right\|>\epsilon\right) \\
	\leq &P\left(\lambda_{max}\left(U_0U_0^T\left(\frac{1}{t}\zeta I_k+\frac{1}{t}\sum_{i=1}^t f(S_i,A_i)f(S_i,A_i)^T\right)^{-1}\right)\left\|\frac{1}{t}\sum_{i=1}^t f(S_i,A_i)\epsilon_i-\frac{\zeta}{t}\mu^*\right\|>\epsilon\right)\\
	\leq &P(\lambda_{max}\left(U_0U_0^T\left(\frac{1}{t}\zeta I_k+\frac{1}{t}\sum_{i=1}^t f(S_i,A_i)f(S_i,A_i)^T\mathbb{1}_{\{A_i=a_0\}}\right)^{-1}\right)\\
	\times&\left\|\frac{1}{t}\sum_{i=1}^t f(S_i,A_i)\epsilon_i-\frac{\zeta}{t}\mu^*\right\|>\epsilon)
	\end{align*}
	
	By eigenvalue decomposition, we write $\frac{1}{t}\sum_{i=1}^t f(S_i,A_i)f(S_i,A_i)^T\mathbb{1}_{\{A_i=a_0\}}:=U_0\Sigma_{(t)}U_0^T$, let $V_0$ denote the orthonormal matrix where its column vectors span the complementary subspace of $U_0$'s column space. Thus we have $U_0^TV_0 = 0$ and $U_0U_0^T+V_0V_0^T=I_k$. So above term can be written as
	
	\begin{align*}
	=&P\left(\lambda_{max}\left(U_0U_0^T\left(\frac{1}{t}\zeta\left(U_0U_0^T+V_0V_0^T\right)+U_0\Sigma_{(t)}U_0^T\right)^{-1}\right)\left\|\frac{1}{t}\sum_{i=1}^t f(S_i,A_i)\epsilon_i-\frac{\zeta}{t}\mu^*\right\|>\epsilon\right)\\
	= & P\left(\lambda_{max}\left(U_0U_0^T\left(\frac{1}{t}
	\zeta V_0V_0^T+U_0\left(\Sigma_{(t)}+\frac{1}{t}\zeta I\right)U_0^T\right)^{-1}\right)\left\|\frac{1}{t}\sum_{i=1}^t f(S_i,A_i)\epsilon_i-\frac{\zeta}{t}\mu^*\right\|>\epsilon\right)\\
	= & P\left(\lambda_{max}\left(U_0U_0^T\left(\frac{1}{\zeta}
	t V_0V_0^T+U_0\left(\Sigma_{(t)}+\frac{1}{t}\zeta I\right)^{-1}U_0^T\right)\right)\left\|\frac{1}{t}\sum_{i=1}^t f(S_i,A_i)\epsilon_i-\frac{\zeta}{t}\mu^*\right\|>\epsilon\right)\\
	=& P\left(\lambda_{max}\left(U_0\left(\Sigma_{(t)}+\frac{1}{t}\zeta I\right)^{-1}U_0^T\right)\left\|\frac{1}{t}\sum_{i=1}^t f(S_i,A_i)\epsilon_i-\frac{\zeta}{t}\mu^*\right\|>\epsilon\right)\\
	\leq& P\left(\lambda_{min}^{-1}\left(\Sigma_{(t)}\right)\left\|\frac{1}{t}\sum_{i=1}^t f(S_i,A_i)\epsilon_i-\frac{\zeta}{t}\mu^*\right\|>\epsilon\right)\\
	=& P\left(\lambda_{min}^{-1}\left(\Sigma_{(t)}\right) \left\|\frac{1}{t}\sum_{i=1}^t f(S_i,A_i)\epsilon_i-\frac{\zeta}{t}\mu^*\right\|>\epsilon \cap \text{rank}(\Sigma_{(t)})=r_0\right)\\
	+& P\left(\lambda_{min}^{-1}\left(\Sigma_{(t)}\right) \left\|\frac{1}{t}\sum_{i=1}^t f(S_i,A_i)\epsilon_i-\frac{\zeta}{t}\mu^*\right\|>\epsilon \cap \text{rank}(\Sigma_{(t)})\neq r_0\right)\\
	\leq & P\left(\lambda_{(r_0)}^{-1}\left(\Sigma_{(t)}\right)\left\|\frac{1}{t}\sum_{i=1}^t f(S_i,A_i)\epsilon_i-\frac{\zeta}{t}\mu^*\right\|>\epsilon\right)+P\left(\text{rank}(\Sigma_{(t)})\neq r_0\right)
	\end{align*}
	
	In order to upper bound $P\left(\left|f(S_\tau,a_0)^T(\mu^*-\hat{\mu}_t)\right|>\epsilon\right)$, we only need to upper bound the two probabilities in above term.
	
	\textbf{Part 1: lower bound $\lambda_{(r_0)}\left(\frac{1}{t}\sum_{i=1}^t f(S_i,A_i)f(S_i,A_i)^T \mathbb{1}_{\{A_i=a_0\}}\right)$}
	
	\begin{align*}
	&\lambda_{(r_0)}\left(\frac{1}{t}\sum_{i=1}^t f(S_i,A_i)f(S_i,A_i)^T \mathbb{1}_{\{A_i=a_0\}}\right)\\
	&\lambda_{(r_0)}\left(\frac{1}{t}\sum_{i=1}^t f(S_i,a_0)f(S_i,a_0)^T \mathbb{1}_{\{A_i=a_0\}}\right)\\
	\geq & \lambda_{(r_0)}\left(\frac{1}{t}\sum_{i=t/2}^t f(S_i,a_0)f(S_i,a_0)^T \mathbb{1}_{\{A_i=a_0\}}\right) (\text{By Weyl's \ inequality})\\
	\geq & \lambda_{(r_0)}\left(\frac{1}{t}\sum_{i=t/2}^t f(S_i,a_0)f(S_i,a_0)^T \pi_{\hat{\theta}_{i-1}}\left(S_i,a_0\right)\right) - \lambda_{(r_0)}\left(\frac{1}{t}\sum_{i=t/2}^t f(S_i,a_0)f(S_i,a_0)^T \tilde{p}_0\right)\\
	+&\lambda_{(r_0)}\left(\frac{1}{t}\sum_{i=t/2}^t f(S_i,a_0)f(S_i,a_0)^T \tilde{p}_0\right) - \left\|\frac{1}{t}\sum_{i=t/2}^{t}f\left(S_i,a_0\right)f\left(S_i,a_0\right)^T\left(\mathbb{1}_{\{A_i=a_0\}} - \pi_{\hat{\theta}_{i-1}}\left(S_i,a_0\right)\right)\right\|_{op}
	\end{align*}
	
	Thus we have
	\begin{align*}
	&P\left(\lambda_{(r_0)}\left(\frac{1}{t}\sum_{i=1}^t f(S_i,A_i)f(S_i,A_i)^T \mathbb{1}_{\{A_i=a_0\}}\right)\geq \frac{\lambda_{r0}\tilde{p}_0}{8} \right)\\
	\geq& 1-P\left(\lambda_{(r_0)}\left(\frac{1}{t}\sum_{i=t/2}^t f(S_i,a_0)f(S_i,a_0)^T \pi_{\hat{\theta}_{i-1}}\left(S_i,a_0\right)\right) - \lambda_{(r_0)}\left(\frac{1}{t}\sum_{i=t/2}^t f(S_i,a_0)f(S_i,a_0)^T \tilde{p}_0\right)<0\right)\\
	-& P\left(\lambda_{(r_0)}\left(\frac{1}{t}\sum_{i=t/2}^t f(S_i,a_0)f(S_i,a_0)^T \tilde{p}_0\right)<\frac{\lambda_{r0}\tilde{p}_0}{4}\right)\\
	-& P\left(\left\|\frac{1}{t}\sum_{i=t/2}^{t}f\left(S_i,a_0\right)f\left(S_i,a_0\right)^T\left(\mathbb{1}_{\{A_i=a_0\}} - \pi_{\hat{\theta}_{i-1}}\left(S_i,a_0\right)\right)\right\|_{op}>\frac{\lambda_{r0}\tilde{p}_0}{8}\right)
	\end{align*}
	In the last inequality above, we have
	\begin{align*}
	&P\left(\lambda_{(r_0)}\left(\frac{1}{t}\sum_{i=t/2}^t f(S_i,a_0)f(S_i,a_0)^T \pi_{\hat{\theta}_{i-1}}\left(S_i,a_0\right)\right) - \lambda_{(r_0)}\left(\frac{1}{t}\sum_{i=t/2}^t f(S_i,a_0)f(S_i,a_0)^T \tilde{p}_0\right)\geq 0 \right)\\
	\geq &P\left(\pi_{\hat{\theta}_{i-1}}\left(S_i,a_0\right) \geq \tilde{p}_0, \forall i = t/2,\ldots,t\right) \geq 1-\frac{tp}{2}\left(2e\right)^{-\frac{t\lambda_{p}}{4}}.
	\end{align*}
	By Chernoff's inequality, we have
	\begin{align*}
	&P\left(\lambda_{(r_0)}\left(\frac{1}{t}\sum_{i=t/2}^t f(S_i,a_0)f(S_i,a_0)^T \tilde{p}_0\right) \geq \frac{\lambda_{r0}\tilde{p}_0}{4}\right) \\
	=& P\left(\lambda_{(r_0)}\left(\frac{2}{t}\sum_{i=t/2}^t f(S_i,a_0)f(S_i,a_0)^T \right) \geq \frac{\lambda_{r0}}{2}\right) \geq 1 - k(2e)^{-\frac{t\lambda_{r0}}{4}}.
	\end{align*}
	By matrix Azuma inequality on martingale we have, 
	\begin{align*}
	P\left(\left\|\frac{1}{t}\sum_{i=t/2}^{t}f\left(S_i,a_0\right)f\left(S_i,a_0\right)^T\left(\mathbb{1}_{\{A_i=a_0\}} - \pi_{\hat{\theta}_{i-1}}\left(S_i,a_0\right)\right)\right\|_{op}\geq \epsilon\right) \leq k\exp\left(-\frac{t\epsilon^2}{8}\right)
	\end{align*}
	Combine everything together, we finish \textbf{part 1} of the proof for equation~\ref{equ:J_mu_fix} by
	\begin{align*}
	&P\left(\lambda_{(r_0)}\left(\frac{1}{t}\sum_{i=t/2}^t f(S_i,A_i)f(S_i,A_i)^T \mathbb{1}_{\{A_i=a_0\}}\right) \geq \frac{\lambda_{r0}\tilde{p}_0}{8}\right)\\
	\geq& 1-\frac{tp}{2}\left(2e\right)^{-\frac{t\lambda_{p}}{4}}-k(2e)^{-\frac{t\lambda_{r0}}{4}}-k\exp\left(-\frac{t\lambda_{r0}^2\tilde{p}_0^2}{512}\right)
	\end{align*}
	\noindent
	\textbf{Part 2:} upper bound $P\left(\text{rank}(\Sigma_{(t)})\neq r_0\right)$:
	
	First we consider $\sum_{i=t/2}^{t}\mathbb{1}_{\{A_i=a_0\}}$:
	\begin{align*}
	&P\left(\sum_{i=t/2}^{t}\mathbb{1}_{\{A_i=a_0\}}>t\tilde{p}_0/4\right)\\
	=& P\left(\sum_{i=t/2}^{t}\left(\mathbb{1}_{\{A_i=a_0\}}-\pi_{\hat{\theta}_{i-1}}\left(S_i,a_0\right)\right)+\sum_{i=t/2}^{t}\left(\pi_{\hat{\theta}_{i-1}}\left(S_i,a_0\right)-\tilde{p}_0\right)+t\tilde{p}_0/2>t\tilde{p}_0/4\right)\\
	\geq& 1-P\left(\sum_{i=t/2}^{t}\left(\mathbb{1}_{\{A_i=a_0\}}-\pi_{\hat{\theta}_{i-1}}\left(S_i,a_0\right)\right)\leq -t\tilde{p}_0/4\right)-P\left(\sum_{i=t/2}^{t}\left(\pi_{\hat{\theta}_{i-1}}\left(S_i,a_0\right)-\tilde{p}_0\right)<0\right)\\
	\geq& 1-\exp\left(-\frac{t\tilde{p}_0^2}{16}\right)-\frac{t}{2}p(2e)^{-\frac{t\lambda_{p}}{4}}
	\end{align*}
	Last inequality follows from the Azuma's inequality and Lemma~\ref{lemma:bdd}. Thus we have
	\begin{align*}
	& P\left(\text{rank}(\Sigma_{(t)})=r_0\right)\\
	&= P\left(\text{rank}(\Sigma_{(t)})\geq r_0\right)\\
	&= P\left(\lambda_{(r_0)}\left(\Sigma_{(t)}\right)>0\right)\\
	&\geq P\left(\left\|\frac{1}{\sum_{i=1}^{t}\mathbb{1}_{\{A_i=a_0\}}}\sum_{i:A_i=a_0}^{t}f(S_i,a_0)f(S_i,a_0)^T-\mathbb{E}\left[f(S,a_0)f(S,a_0)^T\right]\right\|_{op}\leq\frac{\lambda_{(r_0)}}{2}\right)\\
	&\geq  P\left(\text{event in previous line}\mid \sum_{i=t/2}^{t}\mathbb{1}_{\{A_i=a_0\}}> t\tilde{p}_0/4\right)P\left(\sum_{i=t/2}^{t}\mathbb{1}_{\{A_i=a_0\}}> t\tilde{p}_0/4\right)\\
	&\geq \left(1-k\exp\left(-\frac{t\tilde{p}_0\lambda_{r_0}^2}{256}\right)\right)P\left(\sum_{i=t/2}^{t}\mathbb{1}_{\{A_i=a_0\}}>t\tilde{p}_0/4\right) \text{(Matrix Hoeffding)}\\
	&\geq 1-k\exp\left(-\frac{t\tilde{p}_0\lambda_{r_0}^2}{256}\right)-\exp\left(-\frac{t\tilde{p}_0^2}{16}\right)-\frac{t}{2}p(2e)^{-\frac{t\lambda_{p}}{4}}.
	\end{align*}
	\noindent
	Thus, combining \textbf{part 1} and \textbf{part 2}, we have
	\begin{align*}
	&P\left(\left\|U_0U_0^T\left(\mu^*-\hat{\mu_t}\right)\right\| >\epsilon\right)\\
	\leq& P\left(\lambda_{(r_0)}^{-1}\left(\frac{1}{t}\sum_{i=1}^t f(S_i,A_i)f(S_i,A_i)^T\mathbb{1}_{\{A_i=a_0\}}\right)>\frac{8}{\tilde{p}_0\lambda_{r_0}}\right)\\
	+& P\left(\left\|\frac{1}{t}\sum_{i=1}^t f(S_i,A_i)\epsilon_i-\frac{\zeta}{t}\mu^*\right\|>\frac{\epsilon \tilde{p}_0\lambda_{r_0}}{8}\right)+P\left(rank(\Sigma_{(t)})\neq r_0\right)\\
	=& P\left(\lambda_{(r_0)}\left(\frac{1}{t}\sum_{i=1}^t f(S_i,A_i)f(S_i,A_i)^T\mathbb{1}_{\{A_i=a_0\}}\right) \leq \frac{\tilde{p}_0\lambda_{r_0}}{8}\right)\\
	+& P\left(\left\|\frac{1}{t}\sum_{i=1}^t f(S_i,A_i)\epsilon_i\right\|>\frac{\epsilon \tilde{p}_0\lambda_{r_0}}{8}-\left\|\frac{\zeta}{t}\mu^*\right\|\right)+P\left(rank(\Sigma_{(t)})\neq r_0\right)\\
	\leq& \frac{tp}{2}\left(2e\right)^{-\frac{t\lambda_{p}}{4}}+k(2e)^{-\frac{t\lambda_{r0}}{4}}+k\exp\left(-\frac{t\lambda_{r0}^2\tilde{p}_0^2}{512}\right)+
	\left(2+\frac{C_1\sigma^2}{t\epsilon^2\tilde{p}_0^2\lambda_{r_0}^2}\right)\exp\left(-\frac{\sqrt{t}\epsilon \tilde{p}_0 \lambda_{r_0}}{C_2\sigma}\right)
	\\
	+&k\exp\left(-\frac{t\tilde{p}_0\lambda_{r_0}^2}{256}\right)+\exp\left(-\frac{t\tilde{p}_0^2}{16}\right)+\frac{t}{2}p(2e)^{-\frac{t\lambda_{p}}{4}}
	\end{align*}
	while $t \geq \frac{16\zeta}{\epsilon \tilde{p}_0\lambda_{r_0}}$. The last inequality is obtained from previous results and martingale valued inequality applied on $\{f(S_i,A_i)\epsilon_i\}_{i=1}^t$ and $C_1,C_2$ are constants.
	
	Similarly, by eigenvalue decomposition, we have $\mathbb{E}\left[f(S,a_1)f(S,a_1)^T\right] = U_1\Sigma_1 U_1^T$, $rank(\Sigma_1)=r_1$ and $\lambda_{r_1}:=\lambda_{(r_1)}(\Sigma_1)$, we have
	\begin{align*}
	&P\left(\left\|U_1U_1^T\left(\mu^*-\hat{\mu_t}\right)\right\| >\epsilon\right)\\
	\leq& \frac{tp}{2}\left(2e\right)^{-\frac{t\lambda_{p}}{4}}+k(2e)^{-\frac{t\lambda_{r1}}{4}}+k\exp\left(-\frac{t\lambda_{r1}^2\tilde{p}_0^2}{512}\right)+\left(2+\frac{C_1\sigma^2}{t\epsilon^2\tilde{p}_0^2\lambda_{r_1}^2}\right)\exp\left(-\frac{\sqrt{t}\epsilon \tilde{p}_0 \lambda_{r_1}}{C_2\sigma}\right)\\
	+&k\exp\left(-\frac{t\tilde{p}_0\lambda_{r_1}^2}{256}\right)+\exp\left(-\frac{t\tilde{p}_0^2}{16}\right)+\frac{t}{2}p(2e)^{-\frac{t\lambda_{p}}{4}}
	\end{align*}
	\noindent
	Thus, for Equation~\ref{equ:J_mu_fix} can be bounded by
	\begin{align*}
	&P\left(\left|\hat{J}_t(\theta^*,\mu^*)-\hat{J}_t(\theta^*,\hat{\mu}_t)\right|>\epsilon\right)\\
	\leq& P\left(\sum_{a\in \mathcal{A}}\frac{1}{t}\sum_{\tau=1}^{t}\left|f(S_\tau,a)^T(\mu^*-\hat{\mu}_t)\right|>\epsilon\right)\\
	\leq& P\left(\left\|U_0U_0^T(\mu^*-\hat{\mu}_t)\right\|_2>\epsilon/2\right)+P\left(\left\|U_1U_1^T(\mu^*-\hat{\mu}_t)\right\|_2>\epsilon/2\right)\\
	\leq& tp\left(2e\right)^{-\frac{t\lambda_{p}}{4}}+2k\left(2e\right)^{-\frac{t\min\{\lambda_{r0},\lambda_{r1}\}}{4}}+2k\exp\left(-\frac{t\min\{\lambda_{r0}^2,\lambda_{r1}^2\}\tilde{p}_0^2}{512}\right)\\
	+&2\left(2+\frac{C_1\sigma^2}{t\epsilon^2\tilde{p}_0^2\min\{\lambda_{r_0}^2,\lambda_{r_1}^2\}}\right)\exp\left(-\frac{\sqrt{t}\epsilon \tilde{p}_0 \min\{\lambda_{r_0},\lambda_{r_1}\}}{C_2\sigma}\right)\\
	+&2k\exp\left(-\frac{t\tilde{p}_0\min\{\lambda_{r_0}^2,\lambda_{r_1}^2\}}{256}\right)+2\exp\left(-\frac{t\tilde{p}_0^2}{16}\right)+tp(2e)^{-\frac{t\lambda_{p}}{4}},
	\end{align*}
	while $t\geq \frac{16\zeta}{\epsilon \tilde{p}_0 \min\{\lambda_{r_0},\lambda_{r_1}\}}$.
	\\
	\\
	\noindent
	\textbf{Equation~\ref{equ:J_mu_random}:} This term can be bounded same as Equation~\ref{equ:J_mu_fix}.
	\\
	\\
	\noindent
	\textbf{Combine Equations~\ref{equ:hoeff},\ref{equ:J_approx_fix},\ref{equ:J_mu_fix},\ref{equ:easy},\ref{equ:J_mu_random},\ref{equ:J_approx_random},\ref{equ:hoeff_random}:}

	\begin{align*}
	&P\left(J(\theta^*,\mu^*) - J(\hat{\theta}_t,\mu^*) > 5\epsilon\right)\\
	\leq & P\left(\left|J(\theta^*,\mu^*) - \tilde{J}_t(\theta^*,\mu^*)\right| > \epsilon\right)\\
	+& P\left(\left|\hat{J}_t(\theta^*,\mu^*)-\hat{J}_t(\theta^*,\hat{\mu}_t)\right|+\left|\hat{J}_t(\hat{\theta}_t,\mu^*) - \hat{J}(\hat{\theta}_t,\mu^*)\right|> 2\epsilon\right)\\
	+& P\left(\hat{J}_t(\theta^*,\hat{\mu}_t) - \hat{J}_t(\hat{\theta}_t,\hat{\mu}_t)> \epsilon\right)\\
	+& P\left(\left|\tilde{J}_t(\hat{\theta}_t,\mu^*) - J(\hat{\theta}_t,\mu^*)\right|> \epsilon\right)\\
	\leq& 2\exp\left\{-\frac{t\epsilon^2}{2\left(2+\frac{1}{\lambda_p}\right)^2}\right\} + 4k\left(2e\right)^{-\frac{t\min\{\lambda_{r0},\lambda_{r1}\}}{4}}+4k\exp\left(-\frac{t\min\{\lambda_{r0}^2,\lambda_{r1}^2\}\tilde{p}_0^2}{512}\right)\\
	+& 2tp\left(2e\right)^{-\frac{t\lambda_{p}}{4}}+4\left(2+\frac{C_1\sigma^2}{t\epsilon^2\tilde{p}_0^2\min\{\lambda_{r_0}^2,\lambda_{r_1}^2\}}\right)\exp\left(-\frac{\sqrt{t}\epsilon \tilde{p}_0 \min\{\lambda_{r_0},\lambda_{r_1}\}}{C_2\sigma}\right)\\
	&+4k\exp\left(-\frac{t\tilde{p}_0\min\{\lambda_{r_0}^2,\lambda_{r_1}^2\}}{256}\right)+4\exp\left(-\frac{t\tilde{p}_0^2}{16}\right)+2tp(2e)^{-\frac{t\lambda_{p}}{4}}\\
	+&2\left(\frac{\left(24+\frac{96}{\lambda_p}\right)\sqrt{\frac{4}{\lambda\lambda_p}}}{\epsilon}\right)^p \exp\left\{-\frac{t\epsilon^2}{8\left(2+\frac{1}{\lambda_p}\right)^2}\right\}+p(2e)^{-\frac{t\lambda_p}{2}}
	\end{align*}
	\noindent
	\textbf{Finishing the proof:} Plug in $\delta$ to the right side in above and one can directly obtain:
	\begin{align*}
	t &=O(\max\left\{\frac{p}{\epsilon^2\lambda_p^2},\frac{\sigma^2}{\epsilon^2\tilde{p}_0^2\gamma^2},\frac{1}{\gamma^2\tilde{p}_0^2},\frac{\zeta}{\epsilon\tilde{p}_0\gamma}\right\}
	\\
	&\times \left(\log\left(\frac{\max\{p,k\}}{\delta}\right)+\log^2\left(\frac{\sigma^2}{\delta\epsilon^2\tilde{p}_0^2\gamma^2}\right)
	+ \log\left(\frac{1}{\epsilon\delta\sqrt{\lambda\lambda_p}\lambda_p}\right)\right))
	\end{align*}
\end{proof}

\section{Proof of Corollary~\ref{cor:J_convergence_regret}}
\begin{proof}
	For high probability parameter $\delta/2 > 0$, we define $T_1 = CT^{2/3}$ and $\epsilon = C'T^{-1/3}$, for some $C>0$ and $C'>0$ such that the condition in Theorem~\ref{thm:J_convergence} holds. 
	That is being said, $P\left(\left|J\left(\theta^*,\mu^*\right)-J\left(\hat{\theta}_t,\mu^*\right)\right|\leq 5\epsilon\right)\geq 1-\delta/2$ holds for $t \geq T_1$.
	Define the good event $E$ as:
	\begin{align*}
	E\triangleq\left\{\left|J\left(\theta^*,\mu^*\right)-J\left(\hat{\theta}_t,\mu^*\right)\right|\leq 5\epsilon \text{ and } \left\|\hat{\theta}_t\right\|_2^2\leq \frac{4}{\lambda\lambda_p},\forall T_1\leq t \leq T\right\}.
	\end{align*}
	By Lemma~\ref{lemma:bdd}, we know that under above parameters setting, $\left\|\hat{\theta}_t\right\|_2^2\leq \frac{4}{\lambda\lambda_p}$ also holds for $t\geq T_1$ for large enough $T$.
	Then $P(E^c)\leq 2P\left(\left|J\left(\theta^*,\mu^*\right)-J\left(\hat{\theta}_t,\mu^*\right)\right| > 5\epsilon\right) = \delta$ and thus $P(E)\geq 1-\delta$.
	Under event $E$ we have,
	\begin{align*}
	\mathrm{Reg}_J(T) &= \sum_{t=1}^{T_1} \left(J\left(\theta^*,\mu^*\right)-J\left(\hat{\theta}_t,\mu^*\right)\right) + \sum_{t=T_1+1}^{T}\left(J\left(\theta^*,\mu^*\right)-J\left(\hat{\theta}_t,\mu^*\right)\right)\\
	& \leq \left(T_1+\frac{1}{\lambda_p}T_1\right) +  \left(T_1+\frac{4}{\lambda_p}T_1\right) + (T - T_1)\times 5\epsilon = \widetilde{O}(T^{2/3}).
	\end{align*}
\end{proof}

\section{Proof of Theorem~\ref{theory_critic}}
\subsection{Proof of the consistency of the critic}

\begin{proof}
	
	The $\mathcal{L}_2$ distance between $\hat{\mu}_t$ and $\mu^*$ can be written as
	\begin{align*}
	|\hat{\mu}_t - \mu^*|^2  &  = \frac{C(t)}{t} (\frac{B(t)}{t})^{-1} (\frac{B(t)}{t})^{-1} \frac{C(t)}{t} +o_p(1)
	\end{align*}
	where 
	\begin{align*}
	C(t) &=\sum_{\tau=1}^t f(S_{\tau},A_{\tau}) \epsilon_{\tau} \\
	B(t) &=\zeta I_{k\times k}+\sum_{\tau=1}^t f(S_{\tau},A_{\tau})f(S_{\tau},A_{\tau})^T
	\end{align*}
	The two steps in proving $|\hat{\mu}_t - \mu^*|_2^2\rightarrow 0$ in probability are
	\begin{enumerate}
		\item We show that the matrix $\frac{B(t)}{t}$ has minimal eigenvalue bounded away from 0 with probability going to 1, and
		\item We also show that $\frac{C(t)}{t}$ converges to a zero vector in probability. 
	\end{enumerate}
	To prove the first step, we construct a matrix-valued martingale difference sequence $\{X_i\}_{i=1}^{\infty}$:  
	\begin{align*}
	X_i & =  {f(S_i, A_i) f(S_i, A_i)^T} - \mathbb{E}({f(S_i, A_i) f(S_i, A_i)^T} |\mathcal{F}_i) \\
	& =  {f(S_i, A_i) f(S_i, A_i)^T} - \int_s d(s)\sum_{a}f(s,a)f(s,a)^T \pi_{\hat{\theta}_{i-1}}(s,a) ds\\
	& =  {f(S_i, A_i) f(S_i, A_i)^T} - K(\hat{\theta}_{i-1})
	\end{align*}
	where $K(\theta)=\mathbb{E}_{\theta}[f(S,A)f(S,A)^T]=\sum_s d(s)\sum_{a}f(s,a)f(s,a)^T \pi_{\theta}(A=a|S=s)$. In our matrix-valued martingale definition, the filtration $\mathcal{F}_i=\sigma \{ \hat{\theta}_j,j\leq i-1 \}$ is the sigma algebra expand by the estimated optimal policy before decision point $i$.
	\noindent By Assumption \ref{assumption_boundedness}, the sequence of random matrices $\{X_i\}$ are uniformly bounded. Applying the matrix Azuma inequality in \cite{tropp2012user}, it follows that  
	
	\begin{align*}
	\lambda_{max}(\frac{B(t)}{t} - \frac{\sum_{i=1}^t K(\hat{\theta}_{i-1}) }{t} )  & = \lambda_{max}(\frac{1}{t}\sum_{i=1}^t X_i + \frac{\zeta I_{k\times k}}{t}) \\
	& \leq  \lambda_{max}(\frac{1}{t}\sum_{i=1}^t X_i) + \lambda_{max}(\frac{\zeta I_{k\times k}}{t}) \\
	& \rightarrow 0 \mbox{ in probability} 
	\end{align*}
	We use operators $\lambda_{min}$ and $\lambda_{max}$ to denote the smallest and the largest eigenvalue of a matrix. Using a similar argument, $\lambda_{min}(\frac{B(t)}{t} - \frac{\sum_{i=1}^t K(\hat{\theta}_{i-1}) }{t} ) \rightarrow 0 \mbox{ in probability}$. Now we can show that the minimal eigenvalue of $\frac{B(t)}{t}$ is bounded away from 0 with probability going to 1:
	\begin{align*}
	\lambda_{min}(\frac{B(t)}{t} ) & = \lambda_{min}(\frac{B(t)}{t} -  \frac{\sum_{i=1}^t K(\hat{\theta}_{i-1} )}{t} +  \frac{\sum_{i=1}^t K(\hat{\theta}_{i-1}) }{t} ) \\
	& \geq \lambda_{min}(\frac{B(t)}{t} -  \frac{\sum_{i=1}^t K(\hat{\theta}_{i-1}) }{t}) + \lambda_{min}( \frac{\sum_{i=1}^t K(\hat{\theta}_{i-1} )}{t} ) 
	\end{align*}
	We've showed that the first term converges to 0 in probability. By the consistency of $\hat{\theta}_t$, $K(\hat{\theta}_t)$ converges to $K(\theta^*)$ in probability element-wise. Since eigenvalues are continuous functions of its element values (to see this, notice that eigenvalues of a matrix are roots of its characteristic polynomial (\cite{zedek1965continuity}) we have $\lambda_{min}( \frac{\sum_{i=1}^t K(\hat{\theta}_{i-1} )}{t} ) \rightarrow \lambda_{min}(K(\theta^*))$ in probability. Hence we have shown that the minimal eigenvalue of $\frac{B(t)}{t}$ is bounded with probability going to 1. 
	
	To prove that $\frac{C(t)}{t}$ converges to $0_k$ vector in probability, we construct vector-valued martingale difference sequence $Y_i={f(S_i, A_i)} \epsilon_i$ using the same filtration $\mathcal{F}_i=\sigma \{ \hat{\theta}_j,j\leq i-1 \}$ . This sequence has bounded variance under Assumption \ref{assumption_boundedness}. Applying the vector-valued Azuma inequality, 
	\begin{align*}
	\frac{C(t)}{t} = \frac{1}{t}\sum_{\tau=1}^t Y_{\tau} \rightarrow 0_k
	\end{align*}
	
\end{proof}

\subsection{Proof of the asymptotic normality of the critic}
\begin{proof}
	Based on the formula of $\hat{\mu}_t$,  
	\begin{align*}
	\hat{\mu}_t-\mu^* & =  (\zeta I_d + \sum_{i=1}^t f(S_i,A_i) f(S_i,A_i)^T)^{-1} (\sum_{i=1}^t f(S_i,A_i)\epsilon_i - \mu^*) \\
	& =   (\frac{\zeta I_d + \sum_{i=1}^t f(S_i,A_i) f(S_i,A_i)^T}{t})^{-1} \sqrt{t} \frac{\sum_{i=1}^t f(S_i,A_i)\epsilon_i }{t} +o_p(1) \\
	\end{align*} 
	Based on the consistency of $\theta_t$, we have that $\frac{\zeta I_d + \sum_{i=1}^t f(S_i,A_i) f(S_i,A_i)^T}{t}$ converges in probability to $\mathbb{E}_{\theta^*}(f(S,A)f(S,A)^T$. Now it is the key to analyze the asymptotic distribution of the martingale difference sequence $\{f(S_i,A_i)\epsilon_i \}_{i=1}^t$. With respect to filtration $\mathcal{F}_{t,j}= \sigma( \{S_i,A_i,\epsilon_i\}_{i=1}^j)$. Define $M^* = [\mathbb{E}_{\theta^*}(f(S,A)f(S,A)^T)]^{-1/2}$ and a martingale difference sequence $\{\xi_{t,i}=\frac{M^* f(s_i,a_i)\epsilon_i}{\sqrt{t}} \}_{i=1}^t$ which is adapted to the filtration $\mathcal{F}_{t,j}$ and satisfies $\mathbb{E}(\xi_{t,i} |\mathcal{F}_{t, i-1}) = 0$, To apply vector Lindberg-Levy central limit theorem for martingale difference sequences (\cite{billingsley1961lindeberg}), we check the two conditions in this theorem:
	
	\begin{enumerate}
		\item The conditional variance assumption. 
		\begin{align*}
		V_t & = \sum_{i=1}^t \mathbb{E}(\xi_{t,i}^2 | \mathcal{F}_{t,i-1} ) \\
		&= \frac{1}{t} \sum_{i=1}^t M^* \mathbb{E}_{\theta_{i-1}}(f(s,a)f(s,a)^T) M^* 
		\end{align*}
		converges in probability to $I_d \sigma^2$ by consistency of $\theta_t$.
		
		\item The Lindeberg condition. For any given $\delta>0$, 
		\begin{align*}
		& \sum_{i=1}^t \mathbb{E}(\xi_{t,i}^2 \mathbb{I}( \|\xi_{t,i}\|_2 >\delta) |\mathcal{F}_{t, i-1}) \\
		= & \frac{1}{t} \sum_{i=1}^t \mathbb{E}( M^* f(S_i,A_i)f(S_i,A_i)^T \epsilon_i^2 M^* \mathbb{I}( \|M^*f(S_i,A_i) \epsilon_i\|_1 >\sqrt{t} \delta) |\mathcal{F}_{t, i-1}) \\
		\leq & \frac{1}{t} \sum_{i=1}^t \mathbb{E}( M^* f(S_i,A_i)f(S_i,A_i)^T \epsilon_i^2 M^* \mathbb{I}( \|M^*f(S_i,A_i)\|_2 \epsilon_i^2 >\sqrt{t} \delta) |\mathcal{F}_{t, i-1}) 
		\end{align*}
		By Assumption \ref{assumption_boundedness}, $f(S,A)$ are bounded almost surely, therefore the above expression goes to $0$ as $t\rightarrow 0$. 
		
	\end{enumerate}
	The Lindberg-Levy martingale central limit theorem concludes that 
	\begin{align*}
	\sum_{i=1}^{t} \xi_{t,i} \rightarrow N(0_d,I_d\sigma^2) \mbox{   in distribution}
	\end{align*}
	Therefore 
	\begin{align}
	\sqrt{t}(\hat{\mu}_t-\mu^*) \rightarrow N(0_d, [\mathbb{E}_{\theta^*}(f(S,A)f(S,A)^T)]^{-1}\sigma^2)
	\end{align}
	%
	%
\end{proof}

\section{Proof of Theorem~\ref{theory_actor}}
\subsection{Proof of the consistency of the actor}

Based on Theorem~\ref{thm:J_convergence} and Assumption~\ref{assumption:uniq}, for sufficiently small $\delta >0$, there exist $\epsilon>0$ such that 
\begin{align*}
\{\theta: |J(\theta, \mu^*)-J(\theta^*, \mu^*)| \leq \epsilon\} \cap \{\theta: |\theta-\theta^*| > \delta\} = \emptyset
\end{align*}
\noindent Therefore 
\begin{align*}
& P(|J(\hat{\theta}_t, \mu^*)-J(\theta^*, \mu^*)| \leq \epsilon) \\
= & P(\{ |J(\hat{\theta}_t, \mu^*)-J(\theta^*, \mu^*)| \leq \epsilon \} \cap \{ |\hat{\theta}_t - \theta^*| > \delta \}) \\
+ & P(\{ |J(\hat{\theta}_t, \mu^*)-J(\theta^*, \mu^*)| \leq \epsilon \} \cap \{ |\hat{\theta}_t - \theta^*| \leq \delta \}) \\
= & P(\{ |J(\hat{\theta}_t, \mu^*)-J(\theta^*, \mu^*)| \leq \epsilon \} \cap \{ |\hat{\theta}_t - \theta^*| \leq \delta \}) \\
\leq & P( |\hat{\theta}_t - \theta^*| \leq \delta) \leq 1
\end{align*}
Since $P(|J(\hat{\theta}_t, \mu^*)-J(\theta^*, \mu^*)| \leq \epsilon) \rightarrow 1$ in as $t \rightarrow \infty$, $P( |\hat{\theta}_t - \theta^*| \leq \delta) \rightarrow 1$. In other words, $P( |\hat{\theta}_t - \theta^*| > \delta) \rightarrow 0$, convergence in probability follows. 

\subsection{Proof of the asymptotic normality of the actor}
Again, our strategy is to derive the asymptotic normality of $\tilde{\theta}_t$ and then use the fact that $\hat{\theta}_t$ must have  the same asymptotic distribution. 
\begin{proof}
	We first prove that 
	\begin{align}
	\mathbb{G}_t j_{\theta}(\hat{\mu}_t, \hat{\theta}_t, S) - \mathbb{G}_t j_{\theta}(\mu^*, \theta^*, S) = o_p(1)
	\label{t41}
	\end{align}
	, where $\mathbb{G}_t=\sqrt{t} (\mathbb{P}_t-P)$, the empirical process induced by the ``marginal" stochastic process $\{S_i\}_{i=1}^t$ formed by the history of contexts. The ``full" stochastic process involves the sequence of triples $\{S_i, A_i, \epsilon_i\}_{i=1}^t$, the complete history of contexts, actions and reward errors. We consider the class of functions $\mathcal{F}=\{ j_{\theta}(\mu, \theta, s): \|\theta-\theta^*\|_2\leq \delta, \|\mu-\mu^*\|_2 \leq \delta \}$, where $j_{\theta}(\mu,\theta,s)$ is the partial derivative with respect to $\theta$ of function:
	
	\begin{align*}
	j(\mu,\theta,s)= \sum_a f(s,a)^T\mu \pi_{\theta}(s,a)-\lambda\theta^Tg(s)g(s)^T\theta
	\end{align*}
	The boundedness assumption on reward feature, policy feature and reward ensures that the parametrized class of functions $j_{\theta}(\mu, \theta, s)$ is P-Donsker in a neighborhood of $(\mu^*,\theta^*)$. In other words $\mathcal{F}$ is P-Donsker, where P is the distribution of the marginal stochastic process formed by contexts. We complete the first part of the proof by modiftying Lemma 19.24 in \cite{van2000asymptotic}. It may seem that the dependence of $\hat{\mu}_t$ and $\tilde{\theta}_t$ on the full stochastic process could introduce complexity but a closer inspection shows that the proof goes through.  The random function $j_{\theta}(\hat{\mu}_t, \tilde{\theta}_t, S)$ belongs to the P-Donsker class defined above and satisfies that 
	
	\begin{align*}
	\sum_s d(s) (j_{\theta}(\hat{\mu}_t, \tilde{\theta}_t, s)-j_{\theta}(\mu^*,\theta^*,s))^2 \rightarrow 0
	\end{align*}
	
	\noindent in probability. This is a result of the consistency of both $\hat{\mu}_t$ and $\tilde{\theta}_t$, as well as applying the continuous mapping theorem. By Theorem 18.10(v) in \cite{van2000asymptotic}, $(\mathbb{G}_t, j_{\theta}(\hat{\mu}_t, \tilde{\theta}_t, s)) \rightarrow (\mathbb{G}_p, j_{\theta}(\mu^*, \theta^*, s))$ in distribution, where $\mathbb{G}_p$ is the P-Brownian bridge. The key here is that Theorem 18.10 only relies on the convergence of two stochastic processes, regardlessly of whether the stochastic processes consist of i.i.d. observations and whether or not the two processes are dependent. By Lemma 18.15 in \cite{van2000asymptotic}, almost all sample paths of $\mathbb{G}_p$ are continuous on $\mathcal{F}$. Define a mapping $h: l(\mathcal{F})^{\infty}\times \mathcal{F} \rightarrow \mathbb{R}$ by $h(z,f)=z(f)-z(j_{\theta}(\mu^*,\theta^*,s))$, which is continuous at almost every point of $(\mathbb{G}_p, j_{\theta}(\mu^*,\theta^*,s))$. By the continuous mapping theorem, we have 
	
	\begin{align*}
	\mathbb{G}_t(j_{\theta}(\hat{\mu}_t, \tilde{\theta}_t, s)-j_{\theta}(\mu^*,\theta^*,s)) = h(\mathbb{G}_t, j_{\theta}(\hat{\mu}_t, \tilde{\theta}_t, s)) \rightarrow h(\mathbb{G}_p, j_{\theta}(\mu^*,\theta^*,s))=0
	\end{align*}
	\noindent in distribution and thus in probability, therefore \eqref{t41} holds. 
	The second part of the proof begins by noticing that $\tilde{\theta}_t$ satisfies the estimating equation $\mathbb{P}_t j_{\theta}(\hat{\mu}_t, \tilde{\theta}_t, s)=0$, so we have
	\begin{align*}
	\mathbb{G}_t j_{\theta}(\hat{\mu}_t,\tilde{\theta}_t,s) & = \sqrt{t} (Pj_{\theta}(\mu^*,\theta^*,s)-Pj_{\theta}(\hat{\mu}_t,\tilde{\theta}_t,s)) \\
	& = \sqrt{t} (J_{\theta}(\mu^*,\theta^*)-J_{\theta}(\hat{\mu}_t,\tilde{\theta}_t) )\\
	& = \sqrt{t} J_{\theta\theta}^*(\theta^*-\tilde{\theta}_t) + \sqrt{t} J_{\theta\mu}^*(\mu^*-\hat{\mu}_t) + \sqrt{t} o_p(\|\tilde{\theta}_t-\theta^*\|) +o_p(1)
	\end{align*}
	Together with \eqref{t41} the above implies
	\begin{align*}
	\sqrt{t}(\theta^*-\tilde{\theta}_t) & = (J_{\theta\theta}^*)^{-1}J_{\theta\mu}^*\sqrt{t} (\hat{\mu}_t-\mu^*) + \sqrt{t} o_p(\|\tilde{\theta}_t-\theta^*\|) + (J_{\theta\theta} ^*)^{-1}\mathbb{G}_t j_{\theta}(\mu^*,\theta^*,s)+o_p(1) \\ \label{t42}
	& = O_p(1) + \sqrt{t} o_p(\|\tilde{\theta}_t-\theta^*\|)
	\end{align*}
	where $J_{\theta\theta}^*$ and $J_{\theta\mu}^*$ are $J_{\theta\theta}$ and $J_{\theta\mu}$ evaluated at $(\theta^*,\mu^*)$. The $\sqrt{t}$ consistency of $\tilde{\theta}_t$ follows through. Now \eqref{t42} has become
	\begin{align}
	\sqrt{t}(\theta^*-\tilde{\theta}_t) & = (J_{\theta\theta}^*)^{-1}J_{\theta\mu}^*\sqrt{t} (\hat{\mu}_t-\mu^*) + (J_{\theta\theta}^*)^{-1}\mathbb{G}_t j_{\theta}(\mu^*,\theta^*,S)+o_p(1) 
	\end{align}
	Since both the two non-vanishing terms on the righthand side are asymptotically normal with zero mean, $\sqrt{t}(\theta^*-\tilde{\theta}_t)$ is asymptotically normal. The only task left is to derive the asymptotic variance. Plugging in the formula for $\hat{\mu}_t$, we have 
	\begin{align*}
	\sqrt{t}(\theta^*-\tilde{\theta}_t) & = (J_{\theta\theta}^*)^{-1} \frac{ \sum_{i=1}^t J_{\theta\mu}^* B^* f(S_i,A_i)\epsilon_i + j_{\theta}(\mu^*,\theta^*,S_i)}{t}+o_p(1)  \\
	& = (J_{\theta\theta}^*)^{-1} \sum_{i=1}^t \zeta_{t,i} +o_p(1)
	\end{align*}
	where $B^*=(M^*)^2=[\mathbb{E}_{\theta^*}f(S,A)f(S,A)^T]^{-1}$. $\{ \zeta_i= \frac{J_{\theta\mu}^* B^* f(S_i,A_i)\epsilon_i + j_{\theta}(\mu^*,\theta^*,S_i)}{t}\}_{i=1}^t$ is a martingale difference sequence with asymptotic variance 
	\begin{align*}
	&\sum_{i=1}^t \mathbb{E}(\zeta_{t,i}^2 | \mathcal{F}_{t,i}) \\
	& = \frac{1}{t}\sum_{i=1}^t \mathbb{E} ( \epsilon_i^2 g_{\theta\mu}^* B^* f(S_i,A_i) f(S_i,A_i)^T B^* g_{\mu\theta}^* \\ 
	&+ j_{\theta}(\mu^*,\theta^*,S_i)j_{\theta}(\mu^*,\theta^*,S_i)^T- 2J_{\theta\mu}B^* f(S_i,A_i)j_{\theta}(\mu^*,\theta^*,S_i)^T \epsilon_i |  \mathcal{F}_{t,i}) \\
	&=  \frac{1}{t}\sum_{i=1}^t \sigma^2 J_{\theta\mu}^* B^* \mathbb{E}_{\theta_{i-1}}(f(S,A)f(S,A)^T) B^* J_{\mu\theta}^*+ \sum_s d(s) j_{\theta}(\mu^*,\theta^*,s)j_{\theta}(\mu^*,\theta^*,s)^T \\
	\end{align*}
	which converges in probability to $V^* = \sigma^2 J_{\theta\mu}^* B^* J_{\mu\theta}^*+ \sum_s d(s) j_{\theta}(\mu^*,\theta^*,s)j_{\theta}(\mu^*,\theta^*,s)^T$. Therefore the asymptotic variance of $\sqrt{t}(\theta^*-\tilde\theta_t)$ is $(J_{\theta\theta}^*)^{-1}V^* (J_{\theta\theta}^*)^{-1}$. 	
\end{proof}

\section{Proof of Corollary~\ref{cor:V_convergence_regret}}

\begin{proof}
	\begin{align*}
	\text{Reg}_V(T)& = 2\tilde{T} + \sum_{t=\tilde{T}+1}^T \int_{s\in \mathcal{S}} d(s) \sum_a r(s,a) [\pi_{\theta^*}(s,a)-\pi_{\hat{\theta}_{t-1}}(s,a)]ds \\
	&= 2\tilde{T} +\sum_{t=\tilde{T}+1}^T \int_{s\in \mathcal{S}} d(s) \sum_a r(s,a) \triangledown\pi_{\hat{\theta}_{t,s,a}}(s,a)^T(\theta^*-\hat{\theta}_{t-1}) ds \\
	& \leq 2\tilde{T} +2 \sum_{t=\tilde{T}+1}^T \int_{s\in \mathcal{S}} \frac{d(s)}{\sqrt{t}} \left\|\sqrt{t} (\theta^*-\hat{\theta}_{t-1})\right\|_2 ds \\
	& = 2\tilde{T} +2 \sum_{t=\tilde{T}+1}^T \frac{1 }{\sqrt{t}} \left\| \sqrt{t} (\theta^*-\hat{\theta}_{t-1})\right\|_2.
	\end{align*}
	where $\hat{\theta}_{t,s,a}$ is a random variable that lies on the line segment joining $\theta^*$ and $\hat{\theta}_{t-1}$. 
	
	Let $\sigma^2_{max}$ be the largest diagonal element of the asymptotic covariance matrix defined in Theorem \ref{theory_actor}. 
	For a random variable $X$ following the multivariate Gaussian distribution in Theorem \ref{theory_actor}., with high probability, we have $\left\|X\right\|_2\leq C(p \sigma_{max} +  \log T)$ for large $T$ and some constant $C$, where $p$ is the dimension of $\theta^*$. 
	Therefore we have $ \left\| \sqrt{t} (\theta^*-\hat{\theta}_{t-1})\right\|_2 \rightarrow \left\|X\right\|_2 \leq C(p \sigma_{max} + \log T)$. 
	We can find a $\tilde{T}$ such that when $t> \tilde{T}$, $ \left\| \sqrt{t} (\theta^*-\hat{\theta}_{t-1})\right\|_2 \leq 2C(p \sigma_{max} + \log T)$ and as a result $\max_{t > \tilde{T}} \left\| \sqrt{t} (\theta^*-\hat{\theta}_{t-1})\right\|_2$ is bounded by $C'(p \sigma_{max} + \log T)$ for some constant $C'$. 
	Plugging this bound into the bound for $\text{Reg}_V(T)$, we arrive at the result in the corollary for $T\gg \tilde{T}$.
	
\end{proof}

\section{Small Sample Variance estimation and Bootstrap Confidence intervals}\label{small}

In this section, we discuss issues, challenges and solutions in creating confidence intervals for the optimal policy parameter $\theta^*$ when the sample size, the total number of decision points, is small. We use a simple example to illustrate that the traditional plug-in variance estimator is plagued with underestimation issue, the direct consequence of which is the deflated confidence levels of the Wald-type confidence intervals for $\theta^*$. We propose to use bootstrap confidence intervals when the sample size is finite. We use simulation to evaluate the bootstrap confidence intervals.

\subsection{Plug-in Variance Estimation and Wald Confidence intervals}

One of the most straightforward ways to estimate the asymptotic variance of ${\theta}_t$ is through the plug-in variance estimation, the formula of which is provided in Theorem \ref{theory_actor}. Once an estimated variance $\hat{V}_i$ is obtained for $\sqrt{t}(\hat{\theta}_i-\theta^*_i)$, a $(1-2\alpha)\%$ Wald type confidence interval for $\theta^*_i$ has the form: $[ \hat{\theta}_i - z_{\alpha}\frac{\hat{V}_i}{\sqrt{t}}, \hat{\theta}_i + z_{\alpha}\frac{\hat{V}_i}{\sqrt{t}} ]$. Here $\theta_i$ is the i-th component in $\theta$ and $z_{\alpha}$ is the upper $100\alpha$ percentile of a standard normal distribution. The plug-in variance estimator and the associated Wald confidence intervals work well in many statistics problems. We shall see that, however, the plug-in variance estimator of the estimated optimal policy parameters suffers from underestimation issue in small to moderate sample sizes. In particular this estimator is very sensitive to the plugged-in value of the estimated reward parameter and policy parameter: a small deviation from the true parameters can result in an inflated or deflated variance estimation. Deflated variance estimation produces anti-conservative confidence intervals, a grossly undesirable property for confidence intervals. The following simple example illustrates the problem. 

\begin{example}
The context is binary with probability distribution $\mathbb{P}(S=1)=\mathbb{P}(S=-1)=0.5$. The reward is generated according to the following linear model: given context $S \in\{-1,1\}$ and action $A \in \{0,1\}$,
\begin{align*}
R = \mu_0^* +\mu_1^*S + \mu_2^*A +\mu_3^*SA +\epsilon
\end{align*}
\noindent where $\epsilon $ follows a normal distribution with mean zero and standard deviation 9. The true reward parameter is $\mu^*=[1,1,1,1]$. Both $\mu^*$ and the standard deviation of $\epsilon$ are chosen to approximate the realistic signal noise ratio in mobile health applications. We consider the policy class $ \pi_{\theta}(A=1|S=s)= \frac{e^{\theta_0+\theta_1s}}{1+e^{\theta_0+\theta_1s}}$. 

The differences between the plug-in estimated variance and its population counterpart are that (1) the former uses the empirical distribution of context to replace the unknown population distribution and (2) the unknown reward parameter and optimal policy parameter are replaced by their estimates. We emphasize that it is the second difference that leads to the underestimated variance in small sample size. To see this, we ignore the difference between the empirical distribution and the population distribution of contexts, which is very small for sample size $T\geq 50$ under a Bernoulli context distribution with equal probability. Now the plug-in variance estimator is a function of the estimated reward parameter $\hat{\mu}_t$ and the estimated policy parameter $\hat{\theta}_t$. Notice that $\hat{\theta}_t=[\hat{\theta}_{t,0},\hat{\theta}_{t,1}]$ is a function of $\hat{\mu}_t=[\hat{\mu}_{t,0},\hat{\mu}_{t,1},\hat{\mu}_{t,2},\hat{\mu}_{t,3}]$ and the empirical distribution of context. If we replace the empirical distribution in calculating $\hat{\theta}_t$ by its population counterpart, $\hat{\theta}_t$ is simply a function of $\hat{\mu}_t$. In the rest part of the example, we drop the subscript $t$ in the estimated reward parameter and denote the estimate of $\mu_2$ and $\mu_3$ by $\hat{\mu}_2$ and $\hat{\mu}_3$, respectively. Likewise, $\hat{\theta}_{t,i}$ is replaced by $\hat{\theta}_i$ for $i=0,1$. 


Figure \ref{fig_var} is the surface plot showing how the plug-in variance estimation changes as function of the estimated reward parameter. The surface plot of the plug-in variance estimation has a mountain-like pattern with two ridges along the two diagonals $\hat{\mu}_{2}+\hat{\mu}_{3}=0$ and $\hat{\mu}_{2}-\hat{\mu}_{3}=0$. The height of the ridge increases as both $\hat{\mu}_2$ and $\hat{\mu}_3$ approaches the origin. The peak of mountain is at the origin where $\hat{\mu}_{2}=\hat{\mu}_{3}=0$. The true reward parameter $(\mu^*_2,\mu^*_3)=(1,1)$ is close to the origin and lies right on the one of the ridges. There are four ``valleys'' where the combinations of $\hat{\mu}_2$ and $\hat{\mu}_3$ gives a small plug-in variance. 
%
%
%
%

\begin{figure}[H] 
    \centering
    \includegraphics[width=4in]{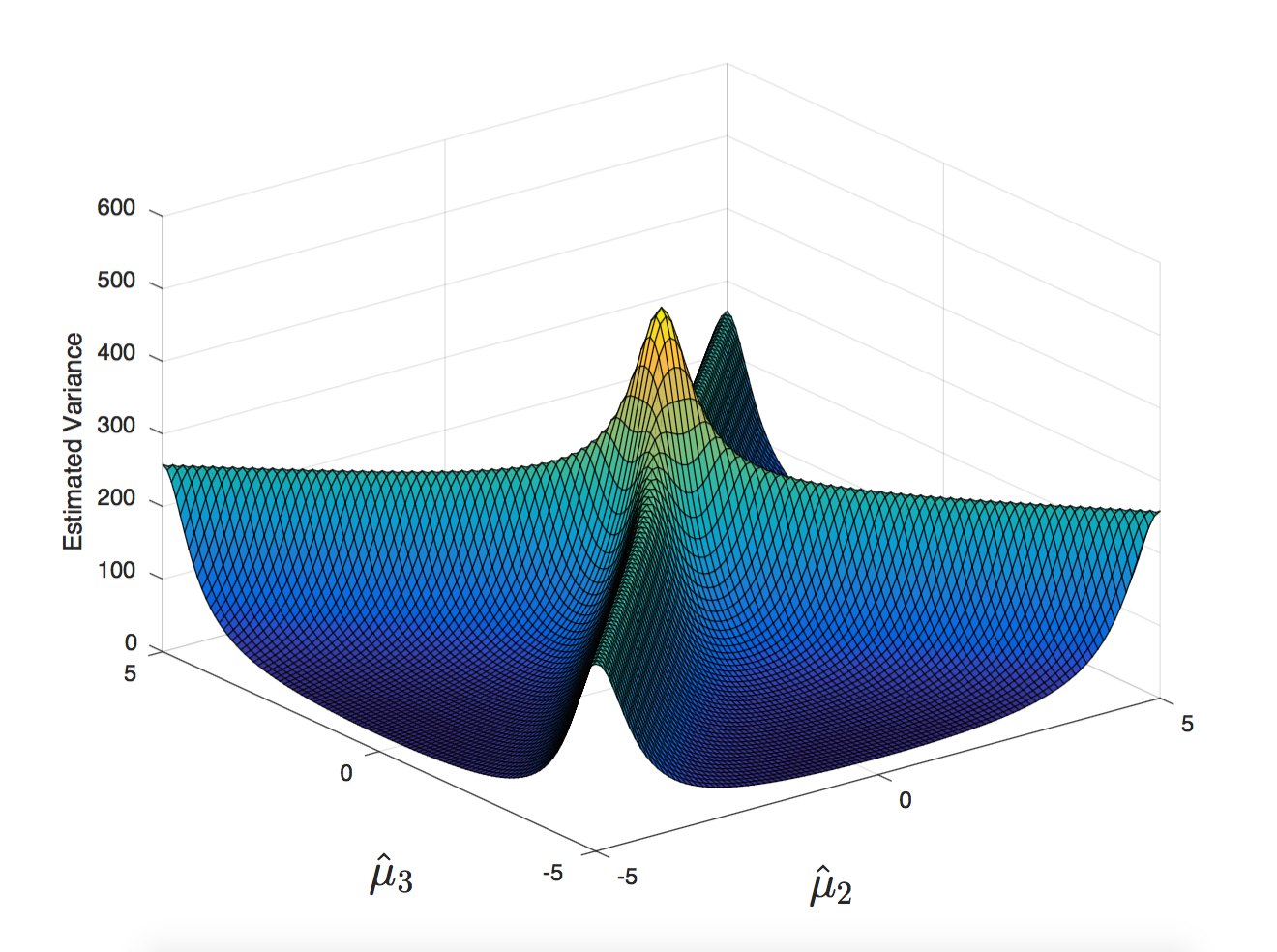} 
    \caption{Plug in variance estimation as a function of\\ $\hat{\mu}_{2}$ and $\hat{\mu}_{3}$, x axis represents $\hat{\mu}_{t,2}$, y axis represents $\hat{\mu}_{t,3}$ and z axis represents the plug-in asymptotic variance of $\hat{\theta}_0$ with $\lambda=0.1$} 
    \label{fig_var}
\end{figure} 

\noindent 



Due to large areas of valley the plug-in variance estimation is biased down, a direct consequence of which is the anti-conservatism of the Wald confidence intervals. We perform a simulation study using the toy generative model described above. The simulation consists of 1000 repetitions of running the online actor critic algorithm and recording the end-of-study statistics, including the plugin variance estimate, the Wald confidence intervals and the theoretical Wald confidence intervals based on the true asymptotic variance. The first two columns in table \ref{t_var} show the bias of plug-in variance at different sample sizes. At all three different sample sizes, the plug-in variance estimator underestimates the true asymptotic variance, which is $293.03$ for both policy parameters. Column 3 and column 4 show the coverage rate of the Wald-type confidence interval (CI) using the plug-in estimated variance. It is not surprising that the confidence intervals suffer from severe anti-conservatism, a consequence of the heavily biased variance estimation. Column 5 and 6 show the coverage rate of the Wald-type confidence interval based on the true asymptotic variance. Comparing the coverage rates, it is clear that the anti-conservatism is due to the underestimated variance. 

\begin{landscape}
\begin{table}[H]
  \begin{tabular}{l |cc|cc|cc}
    \toprule
    \multirow{2}{*}{sample size} &
      \multicolumn{2}{c}{bias in variance estimation} &
      \multicolumn{2}{c}{coverage of Wald CI (\%)} &
      \multicolumn{2}{c}{coverage of theoretical Wald CI (\%)} \\ \cline{2-7}
      & {$\theta_0$} & {$\theta_1$} & {$\theta_0$} & {$\theta_1$} & {$\theta_0$} & {$\theta_1$} \\
      \midrule
    100 & -181.56 & -181.56 & 75.5 & 74.9 & 100.0 & 100.0 \\ \hline
    250 & -131.71 & -131.71 & 77.9 & 77.3 & 98.5 & 98.1 \\\hline
    500 & -108.64 & -108.64 & 78.8 & 79.2 & 98.9 & 98.7 \\\hline
    \bottomrule
  \end{tabular}
  \caption{Underestimation of the plug-in variance estimator and the Wald confidence intervals. \\Theoretical Wald CI is created based on the true asymptotic variance.}
  \label{t_var}
\end{table}
\end{landscape}

To detail how the confidence interval coverage is connected with the estimated reward parameter $(\hat{\mu}_2,\hat{\mu}_3)$, figure \ref{scatter_100} and figure \ref{scatter_500} present two scatter plots of $\hat{\mu}_2, \hat{\mu}_3$ for the 1000 simulated datasets at sample size 100 and 500. Different colors are used to mark the datasets where the confidence intervals of both $\theta_0$ and $\theta_1$ cover the true parameter (blue), only one of them cover the truth (green), neither of them covers the truth (fading yellow). The true parameter are marked with a red asterisk. Indeed the yellow points and green points are in the ``valleys''. Some of the blue points are away from truth, but nevertheless they remain on the ridge, which produces a high variance estimate. Comparing the two scatter plots, as the sample size increases, the estimated reward parameter is less spread out. Nevertheless there are still significantly many pair of $\hat{\mu}_2, \hat{\mu}_3$ that fall in the ``valleys", leading to a underestimated variance and anti-conservative confidence intervals. 


\begin{landscape}
\begin{figure}[H] 
  \begin{minipage}[b]{0.5\linewidth}
    \centering
    \includegraphics[scale=.5]{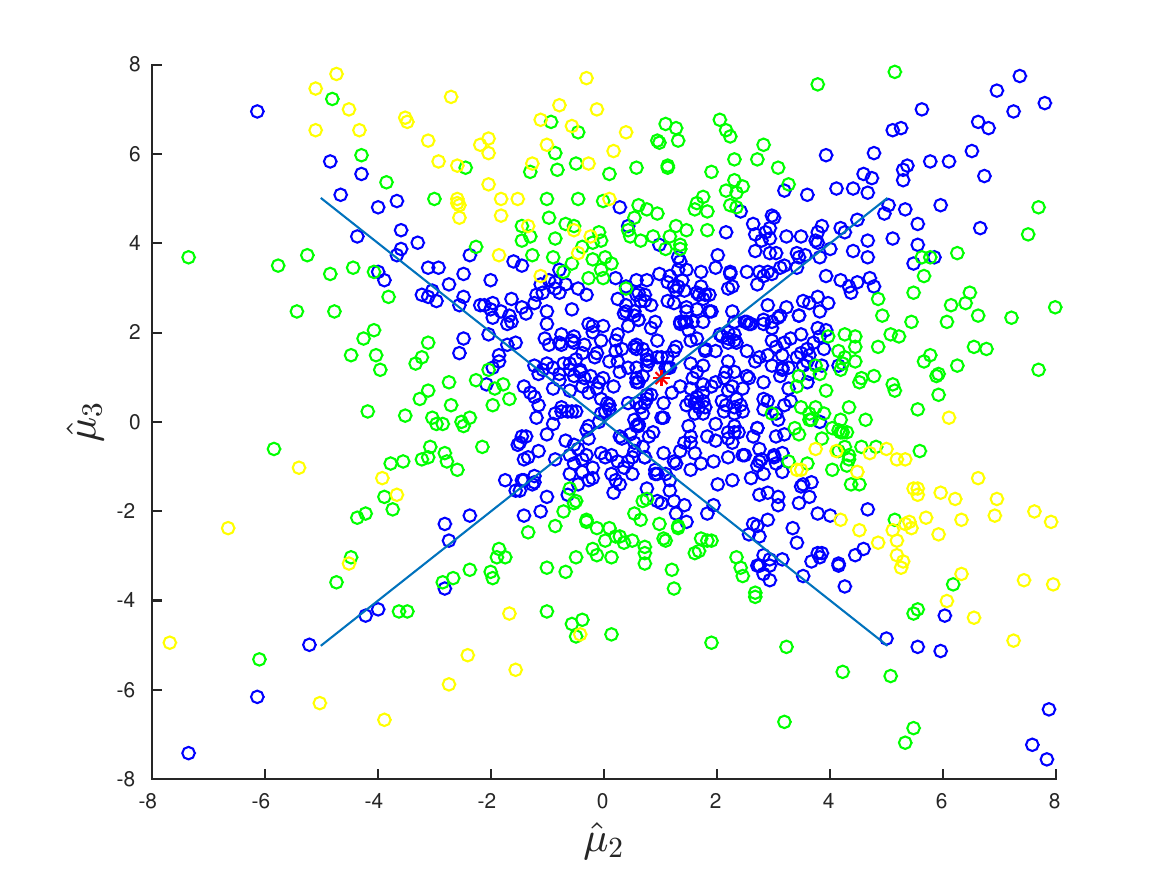} 
    \caption{Wald confidence interval coverage for 1000 simulated datasets as a function of $\hat{\mu}_3$ and $\hat{\mu}_2$ at sample size 100.} 
    \label{scatter_100}
    \vspace{4ex}
  \end{minipage}
  \begin{minipage}[b]{0.5\linewidth}
    \centering
    \includegraphics[scale=.5]{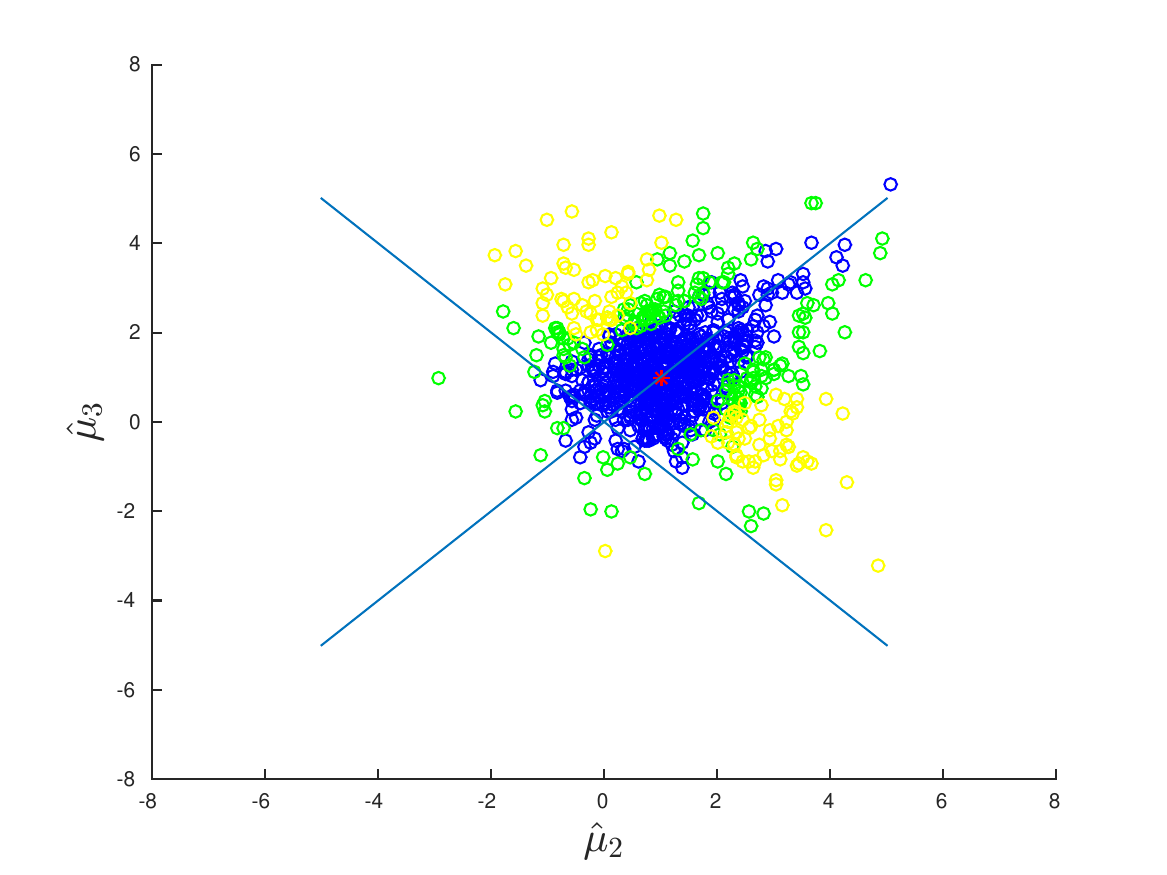} 
    \caption{Wald confidence interval coverage in 1000 simulated datasets as a function of $\hat{\mu}_3$ and $\hat{\mu}_2$ at sample size 500.} 
    \label{scatter_500}
    \vspace{4ex}
  \end{minipage} 
\end{figure} 
\end{landscape}

Figure \ref{tail_100} shows the histogram for the normalized distance $\frac{\hat{\sqrt{T}(\theta}_i - \theta_i^*)}{\hat{V}_i}$ for $i=0,1$ where $T=100$. This is the distance between the estimated and the true optimal policy parameter normalized by the estimated asymptotic variance. For the Wald confidence intervals to have descent coverage, histogram of the normalized distances need to approximate a standard normal distribution. However, as figure \ref{tail_100} suggests, the histograms have heavier tails compared to a standard normal due to the underestimated variance. The figure also suggests that the percentile-t bootstrap confidence intervals can be a good remedy. 

\begin{figure}[H]
\centering
    \includegraphics[scale=.5]{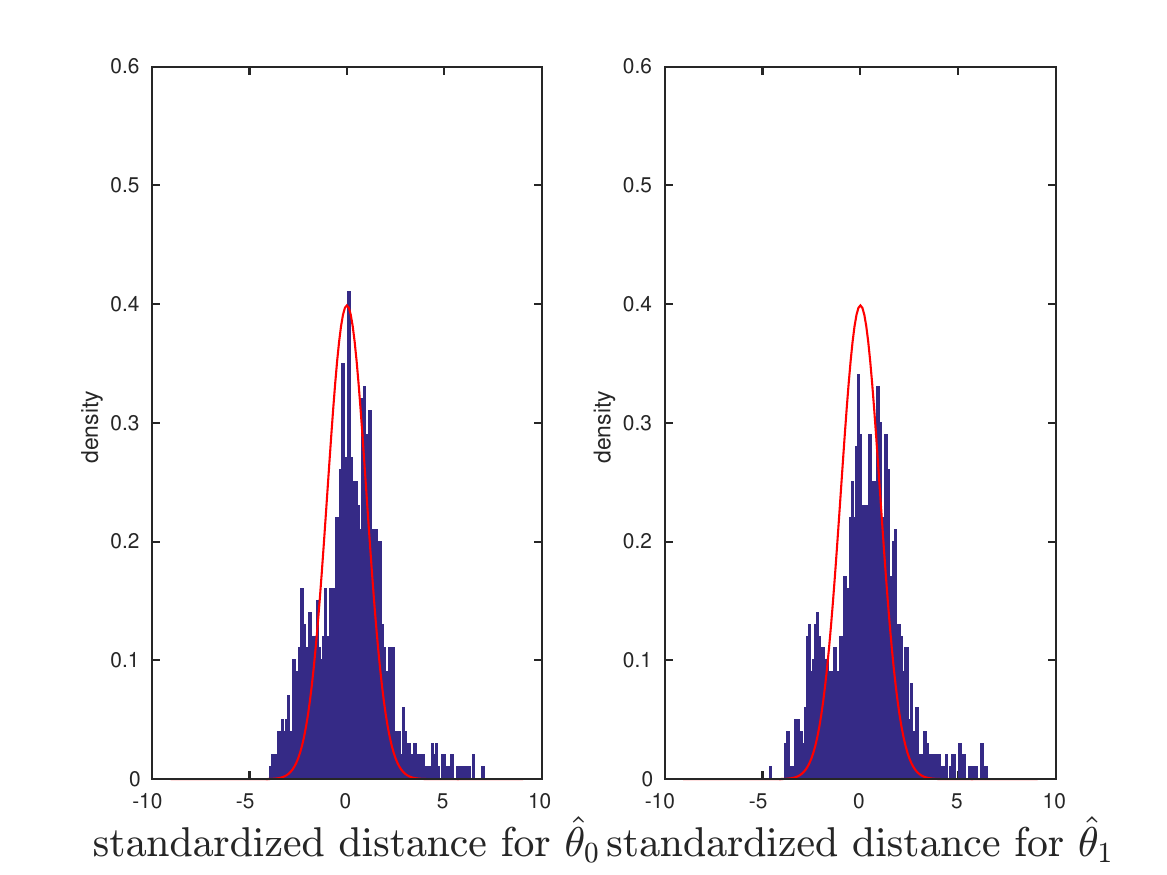} 
    \caption{Histograms of the normalized distance $\frac{\hat{\sqrt{T}(\theta}_i - \theta_i^*)}{\hat{V}_i}$ for $i=0,1$ at sample size 100} 
    \label{tail_100}
\end{figure}

\end{example}
%
%
%
%
%

\section{Burden Effect: Actor Critic Algorithm Uses $\lambda^*$}\label{sup_burden}

\begin{table}[H]
\centering
\begin{tabular} {c| ccccc}
\toprule
$\nu$ & $\lambda^*$ & $\theta_0^*$ & $\theta_1^*$ & $\theta_2^*$ & $\theta_3^*$ \\ \hline
0.0 & 0.06 & 0.3410 &   0.3269  &  0.3264   &    0 \\ \hline
0.2&0.05&0.0844&0.3844&0.4&-0.1609 \\ \hline
0.4&0.06&-0.1922&0.3547&0.3312&-0.2313 \\ \hline
0.6&0.08&-0.3312&0.2488&0.2234&-0.2687 \\ \hline
0.8&0.1&-0.3883&0.2078&0.2&-0.2687 \\ \hline
\bottomrule
\end{tabular}
\caption{Burden effect: the optimal policy and the oracle lambda.}
\label{burden_opt}
\end{table}

\begin{table}[H]
\centering
\begin{tabular} {c| *{4}{S[round-precision=3]}}
\toprule
{$\tau$} & {$\theta_0^*$} & {$\theta_1^*$} & {$\theta_2^*$} & {$\theta_3^*$} \\ \hline
0&-0.027352&-0.035565&-0.030344&0.003449\\\hline
0.2&0.22947&-0.092877&-0.10406&0.16421\\\hline
0.4&0.50586&-0.063199&-0.035223&0.23473\\\hline
0.6&0.64507&0.042695&0.072542&0.27198\\\hline
0.8&0.70229&0.083867&0.09608&0.2718\\\hline
\bottomrule
\end{tabular}
\caption{Burden effect: bias in estimating the optimal policy parameter at sample size 200. The algorithm uses $\lambda^*$ instead of learning $\lambda$ online. Bias=$\mathbb{E}(\hat{\theta}_t)-\theta^*$.}
\label{burden_b200_fix}
\end{table}

\begin{table}[H]
\centering
\begin{tabular} {c| *{4}{S[round-precision=3]}}
\toprule
{$\tau$} & {$\theta_0^*$} & {$\theta_1^*$} & {$\theta_2^*$} & {$\theta_3^*$} \\ \hline
0&0.057811&0.03716&0.036343&0.035898\\\hline
0.2&0.10961&0.044463&0.046192&0.062836\\\hline
0.4&0.31295&0.039819&0.03665&0.090984\\\hline
0.6&0.47309&0.037714&0.040625&0.10984\\\hline
0.8&0.55024&0.042799&0.04454&0.1097\\\hline
\bottomrule
\end{tabular}
\caption{Burden effect: MSE in estimating the optimal policy parameter at sample size 200. The algorithm uses $\lambda^*$ instead of learning $\lambda$ online.}
\label{burden_m200_fix}
\end{table}

\begin{table}[H]
\centering
\begin{tabular} {c| cccc}
\toprule
$\nu$ & $\theta_0$ & $\theta_1$ & $\theta_2$ & $\theta_3$ \\ \hline
0&0.963&0.963&0.955&0.942\\\hline
0.2&0.853*&0.946&0.937&0.862*\\\hline
0.4&0.565*&0.96&0.954&0.776*\\\hline
0.6&0.39*&0.937&0.916*&0.739*\\\hline
0.8&0.329*&0.908*&0.899*&0.739*\\\hline
\bottomrule
\end{tabular}
\caption{Burden effect: coverage rates of percentile-t bootstrap confidence intervals for the optimal policy parameter at sample size 200. The algorithm uses $\lambda^*$ instead of learning $\lambda$ online. Coverage rates significantly lower than $0.95$ are marked with asterisks (*).}
\label{burden_c200_fix}
\end{table}

\begin{table}[H]
\centering
\begin{tabular} {c| *{4}{S[round-precision=3]}}
\toprule
{$\tau$} & {$\theta_0^*$} & {$\theta_1^*$} & {$\theta_2^*$} & {$\theta_3^*$} \\ \hline
0.0 & -0.017692 & -0.013808 & -0.006068 & -0.008696 \\\hline
0.2&0.28829&-0.03135&-0.039795&0.14892\\\hline
0.4&0.51553&-0.041593&-0.010872&0.22259\\\hline
0.6&0.59124&0.005305&0.037288&0.26154\\\hline
0.8&0.60607&0.006205&0.020356&0.26263\\\hline
\bottomrule
\end{tabular}
\caption{Burden effect: bias in estimating the optimal policy parameter at sample size 500. The algorithm uses $\lambda^*$ instead of learning $\lambda$ online. Bias=$\mathbb{E}(\hat{\theta}_t)-\theta^*$. }
\label{burden_b500_fix}
\end{table}

\begin{table}[H]
\centering
\begin{tabular} {c| *{4}{S[round-precision=3]}}
\toprule
{$\tau$} & {$\theta_0^*$} & {$\theta_1^*$} & {$\theta_2^*$} & {$\theta_3^*$} \\ \hline
0.0 & 0.029022 & 0.016576 & 0.015445 & 0.016196\\ \hline
0.2&0.12073&0.022334&0.021348&0.042485\\\hline
0.4&0.29446&0.018117&0.015525&0.065667\\\hline
0.6&0.36697&0.011343&0.011526&0.078681\\\hline
0.8&0.37872&0.008136&0.007766&0.076209\\\hline
\bottomrule
\end{tabular}
\caption{Burden effect: MSE in estimating the optimal policy parameter at sample size 500. The algorithm uses $\lambda^*$ instead of learning $\lambda$ online.}
\label{burden_m500_fix}
\end{table}

\begin{table}[H]
\centering
\begin{tabular} {c| cccc}
\toprule
$\tau$ & $\theta_0$ & $\theta_1$ & $\theta_2$ & $\theta_3$ \\ \hline
0.0& 0.944 & 0.950  &0.952 & 0.933*\\\hline
0.2&0.689*&0.943&0.959&0.815*\\\hline
0.4&0.159*&0.944&0.954&0.6*\\\hline
0.6&0.006*&0.941&0.928*&0.295*\\\hline
0.8&0*&0.94&0.944&0.144*\\\hline
\bottomrule
\end{tabular}
\caption{Burden effect: coverage rates of percentile-t bootstrap confidence intervals for the optimal policy parameter at sample size 500. The algorithm uses $\lambda^*$ instead of learning $\lambda$ online. Coverage rates significantly lower than $0.95$ are marked with asterisks (*).}
\label{burden_c500_fix}
\end{table}


\begin{table}[H]
\centering
\begin{tabular} {c| *{4}{S[round-precision=3]}}
\toprule
{$\tau$} & {$\theta_0^*$} & {$\theta_1^*$} & {$\theta_2^*$} & {$\theta_3^*$} \\ \hline
0&0.392&0.3723&0.3713&-0.0006\\\hline
0.2&0.3921&0.3722&0.3713&-0.0006\\\hline
0.4&0.392&0.3723&0.3713&-0.0006\\\hline
0.6&0.392&0.3723&0.3713&-0.0006\\\hline
0.8&0.392&0.3723&0.3713&-0.0006\\\hline
\bottomrule
\end{tabular}
\caption{Burden effect: the myopic equilibrium policy.}
\label{burden_bandit_opt}
\end{table}

\section{Nonlinear Reward: The Optimal Policy}\label{sup_nonlin}

\begin{table}[H]
\centering
\begin{tabular} {c| *{4}{S[round-precision=3]}}
\toprule
{$\alpha$} & {$\theta_0^*$} & {$\theta_1^*$} & {$\theta_2^*$} & {$\theta_3^*$} \\ \hline
0&0.418035 & 0.395067 &  0.397071 & -0.001615\\\hline
0.2&0.496240 & 0.296973 & 0.385421 & 0.000480\\\hline
0.4&0.564503 & 0.202857 & 0.365239 & 0.001684\\\hline
0.6&0.811000 & 0.542000 & 0.888000 & 0.925000\\\hline
\bottomrule
\end{tabular}
\caption{Nonlinear reward: the optimal policy.}
\label{nonlin_opt}
\end{table}

\end{document}